\documentclass[english]{article}
\usepackage[T1]{fontenc}
\usepackage[latin9]{inputenc}
\usepackage{geometry}
\usepackage{array}
\usepackage{verbatim}
\usepackage{float}
\usepackage{amsmath}
\usepackage{amssymb}
\usepackage{graphicx}

\makeatletter

\providecommand{\tabularnewline}{\\}

\floatstyle{ruled}
\newfloat{algorithm}{tbp}{loa}
\providecommand{\algorithmname}{Algorithm}
\floatname{algorithm}{\protect\algorithmname}

\usepackage{babel}
\usepackage{babel}
\usepackage{cite}\usepackage{amsthm}\usepackage{dsfont}\usepackage{array}\usepackage{mathrsfs}\usepackage{comment}\onecolumn

\usepackage{color}\usepackage{babel}

\newcommand{\R}{\mathbb{R}}
\let\bs=\boldsymbol

\def \Diag {\mathrm{Diag}}
\def \diag {\mathrm{diag}}
\def \objective {\mathit{obj}}
\def \area {\mathit{area}}
\def \domain {\set{D}}
\def \opt {\set{opt}}

\def \saliency {\textup{\saliency}}

\def \path {\mathit{path}}

\def \label {\mathit{label}}

\def \minimize {\textup{minimize} }
\def \maximize {\textup{maximize}}
\def \subjectto {\textup{subject to}}

\usepackage{algorithm}
\usepackage{wrapfig}
\usepackage{algorithmic}
\usepackage{babel}
\usepackage{arydshln}
\usepackage{hyperref}

\makeatother

\usepackage{babel}
\begin{document}
\theoremstyle{plain} \newtheorem{lem}{\textbf{Lemma}} \newtheorem{prop}{\textbf{Proposition}}\newtheorem{theorem}{\textbf{Theorem}}
\newtheorem{corollary}{\textbf{Corollary}} \newtheorem{assumption}{\textbf{Assumption}}
\newtheorem{example}{\textbf{Example}} \newtheorem{definition}{\textbf{Definition}}
\newtheorem{fact}{\textbf{Fact}} \theoremstyle{definition}

\theoremstyle{remark}\newtheorem{remark}{\textbf{Remark}}

\let\vec=\mathbf \let\mat=\mathbf \let\set=\mathcal \global\long\def\para#1{\noindent{\bf #1}}

\global\long\def\Diag{\mathrm{Diag}}
 \global\long\def\diag{\mathrm{diag}}
 \global\long\def\objective{\mathit{obj}}
 \global\long\def\area{\mathit{area}}
 \global\long\def\domain{\set{D}}
 \global\long\def\opt{\set{opt}}
 \global\long\def\minimize{\textup{minimize}}
 \global\long\def\subjectto{\textup{subject to}}

\global\long\def\minimize{\textup{minimize} }
 \global\long\def\maximize{\textup{maximize}}
 \global\long\def\subjectto{\textup{subject to}}
 \global\long\def\R{\mathbb{R}}

\title{Near-Optimal Joint Object Matching via Convex Relaxation}

\author{Yuxin Chen%
\thanks{Y. Chen is with the Department of Electrical Engineering, Stanford
University (email: yxchen@stanford.edu). %
}, $\text{ }$Leonidas J. Guibas%
\thanks{L. J. Guibas is with the Department of Computer Science, Stanford
University (email: guibas@stanford.edu). %
}, $\text{ }$and Qi-Xing Huang%
\thanks{Q-X. Huang is with the Department of Computer Science, Stanford University
(email: huangqx@stanford.edu). %
}}
\maketitle
\begin{abstract}
Joint matching over a collection of objects aims at aggregating information
from a large collection of similar instances (e.g. images, graphs,
shapes) to improve maps between pairs of them. Given multiple objects
and matches computed between a few object pairs in isolation, the
goal is to recover an entire collection of maps that are (1) globally
consistent, and (2) close to the provided maps --- and under certain
conditions provably the ground-truth maps. Despite recent advances
on this problem, the best-known recovery guarantees are limited to
a small constant barrier --- none of the existing methods find theoretical
support when more than 50\% of input correspondences are corrupted.
Moreover, prior approaches focus mostly on fully similar objects,
while it is practically more demanding to match instances that are
only partially similar to each other (e.g., different views of a single
physical object).

In this paper, we propose an algorithm to jointly match multiple objects
that exhibit only partial similarities, given a few (possibly highly
incomplete) pairwise matches that are densely corrupted. By encoding
a consistent partial map collection into a 0-1 semidefinite matrix,
we propose to recover the ground-truth maps via a parameter-free convex
program called MatchLift, following a spectral method that pre-estimates
the total number of distinct elements to be matched.  Numerically,
this program can be efficiently solved via alternating direction methods
of multipliers (ADMM) along with a greedy rounding strategy. Theoretically,
MatchLift exhibits near-optimal error-correction ability, i.e. in
the asymptotic regime it is guaranteed to work even when a dominant
fraction $1-\Theta\left(\frac{\log^{2}n}{\sqrt{n}}\right)$ of the
input maps behave like random outliers. Furthermore, MatchLift succeeds
with minimal input complexity, namely, perfect matching can be achieved
as soon as the provided maps form a connected map graph. We evaluate
the proposed algorithm on various benchmark data sets including synthetic
examples and real-world examples, all of which confirm the practical
applicability and usefulness of MatchLift.
\end{abstract}

\section{Introduction}

Finding consistent relations across multiple objects is a fundamental
scientific problem spanning many fields. A partial list includes jigsaw
puzzle solving \cite{Cho:2010:JPS,Goldberg:2004:GAA}, structure from
motion \cite{Zach:2010:DVR,crandall2011discrete}, re-assembly of
fragmented objects and documents \cite{Huang:2006:RFO,zhu2008globally},
and DNA/RNA shotgun assembly sequencing \cite{Marande:2007:DNA}.
Compared with the rich literature in pairwise matching (e.g. of graphs,
images or shapes), joint matching of multiple objects has not been
well explored. A naive approach for joint object matching is to pick
a base object and perform \emph{pairwise matching} with each of the
remaining objects. However, as pairwise matching algorithms typically
generate noisy results, the performance of such approaches is often
far from satisfactory in practice. This gives rise to the question
as to how to aggregate and exploit information from all pairwise maps
that one computes, in order to improve joint object matching in a
consistent and efficient manner.

In this paper, we represent each object as a discrete set of points
or elements, and investigate the problem of joint matching over $n$
different sets, for which the input / observation is a collection
of pairwise maps computed in isolation. A natural and popular criterion
to preserve the global relational compatibility is called cycle-consistency,
i.e., that composition of maps between two objects should be independent
of the connecting path chosen. Such criterion has recently been invoked
in many algorithms \cite{Roberts:2011:SFM,Zach:2010:DVR,Nguyen:2011:CSM,huang2012optimization,Kim:2012:FC}
to detect outliers among the pairwise input maps. These works have
shown experimentally that one can use inconsistent cycles to prune
outliers, provided that the corruption rate is sufficiently small.

Despite the empirical advances of these works, little is known on
the theoretical side, namely, under what conditions can the underlying
ground-truth maps be reliably recovered. Recent work by \cite{huang2013consistent}
provided the first theoretical guarantee for robust and consistent
joint matching. However, there are several fundamental issues left
unaddressed that must be faced in order to accommodate practical challenges.
\begin{enumerate}
\item \textbf{Dense Input Errors: } The state-of-the-art results (e.g. \cite{huang2013consistent})
did not provide theoretical support when more than 50\% of the input
matches are corrupted. This gives rise to the question regarding their
applicability in the presence of highly noisy sources, in which case
the majority of the input maps can be corrupted. Observe that as the
number $n$ of objects to be matched increases, the amount of pairwise
maps one can obtain significantly exceeds $n$. As a result, dense
error correction is information theoretically possible as long as
the global consistency across pairwise maps can be appropriately exploited.
While one would expect an ideal algorithm to work even when most input
maps are random outliers, the challenge remains as to whether there
exist \emph{computationally feasible} methods that can provably detect
and separate dense outliers.
\item \textbf{Partial Similarity: } To the best of our knowledge, all prior
approaches dealt only with a restricted scenario where the ground-truth
maps are given by full isomorphisms (i.e. \emph{one-to-one} correspondences
between any two sets). In reality, a collection of objects usually
exhibit only partial similarity, as in the case of images of the same
scene but from different camera positions. These practical scenarios
require consistent matching of multiple objects that are only partially
similar to each other.
\item \textbf{Incomplete Input Maps: } Computing pairwise maps across all
object pairs are often expensive, sometimes inadmissible, and in fact
unnecessary. Depending on the characteristics of input sources, one
might be able to infer unobserved maps from a small sample of noisy
pairwise matches. While \cite{huang2013consistent} considered incomplete
inputs, the tradeoff between the undersampling factor and the error-correction
ability remains unknown. 
\end{enumerate}
All in all, practical applications require matching partially similar
objects from a small fraction of densely corrupted pairwise maps ---
a goal this paper aims to achieve.

\subsection{Contributions}

This paper is concerned with joint object matching under dense input
errors. Our main contributions in this regard are three-fold. 
\begin{enumerate}
\item \textbf{Algorithms:} Inspired by the recent evidence on the power
of convex relaxation, we propose to solve the joint matching problem
via a semidefinite program called MatchLift. The algorithm relaxes
the binary-value constraints, and attempts to maximize the compatibility
between the input and the recovered maps. The program is established
upon a semidefinite conic constraint that relies on the total number
$m$ of distinct elements to be matched. To this end, we propose to
pre-estimate $m$ via a spectral method. Our methodology is essentially
\emph{ parameter free}, and can be solved by scalable optimization
algorithms. 
\item \textbf{Theory:} We derive performance guarantees for exact matching.
Somewhat surprisingly, MatchLift admits perfect map recovery even
in the presence of dense input corruptions. Our findings reveal the
near-optimal error-correction ability of MatchLift, i.e. as $n$ grows,
the algorithm is guaranteed to work even when a dominant fraction
-- more precisely, a fraction $1-\Omega\left(\frac{\log^{2}n}{\sqrt{n}}\right)$
-- of the inputs behave as random outliers. Besides, while the presence
of partial similarity unavoidably incurs more severe types of input
errors, MatchLift exhibits a strong recovery ability nearly order-wise
equivalent to that in the full-similarity scenario, as long as the
fraction of each object being disclosed is bounded away from zero.
Finally, in many situations, MatchLift succeeds even with minimal
input complexity, in the sense that it can reliably fill in all unobserved
maps based on very few noisy partial inputs, as soon as the provided
maps form a connected graph. This is information theoretically optimal.
\item \textbf{Practice:} We have evaluated the performance of MatchLift
on several benchmark datasets. These datasets include several synthetic
examples as well as real examples from several popular benchmarks.
Experimental results on synthetic examples corroborate our theoretical
findings. On real datasets, the quality of the maps generated by MatchLift
outperforms the state-of-the-art object matching and graph clustering
algorithms.
\end{enumerate}

\subsection{Prior Art}

There has been numerous work studying the problem of object matching,
either in terms of shape mapping, graph matching, or image mapping,
which is impossible to enumerate. We list below a small sample of
development on joint object matching, as well as its relation and
distinction to the well-renowned graph clustering problem.
\begin{itemize}
\item \textbf{Object Matching.} Early work on object matching focused primarily
on matching pairs of objects in isolation (e.g. \cite{schellewald2005probabilistic,cour2007balanced,caetano2009learning}).
Due to the limited and biased information present in an isolated object
pair, pairwise matching techniques can easily, sometimes unavoidably,
generate false correspondences. Last few years have witnessed a flurry
of activity in joint object matching, e.g. ~\cite{Nguyen:2011:CSM,Kim:2012:FC,huang2012optimization,huang2013consistent},
which exploited the global cycle-consistency criterion to prune noisy
maps. The fundamental understanding has recently been advanced by
\cite{huang2013consistent}. Nevertheless, none of the prior work
have demonstrated provable recovery ability when the majority of input
maps/correspondences are outliers, nor were they able to accommodate
practical scenarios where different objects only exhibit partial similarity.
Recent work \cite{PachauriKS13} employed spectral methods for denoising
in the full-similarity case. However, the errors considered therein
are modeled as Gaussian-Wigner additive noise, which is not applicable
in our setting. Another line of work \cite{WangSingerRotation,chaudhury2013global}
proposed to recover global rigid transform between points via convex
relaxation, where the point coordinates might only be partially observed.
While this line of work is relevant, the problem considered therein
is more specialized than the point-based joint matching studied in
this paper; also, none of these paradigms are able to enable dense
error correction.
\item \textbf{Matrix Completion and Robust PCA.} In a broader sense, our
approach is inspired by the pioneering work in low-rank matrix completion
\cite{ExactMC09,keshavan2010few} and robust principal component analysis
\cite{CanLiMaWri09,chandrasekaran2011rank,xu2010robust,ganesh2010dense,YudongRPCA2013},
which reveal the power of convex relaxation in recovering low-dimensional
structures among high-dimensional objects. In fact, the ground truth
herein is equivalent to a block-constant low-rank matrix \cite{jiaming2014block},
as occurred in various graph-related problems. Nevertheless, their
theoretical analyses fail to provide tight bounds in our setting,
as the low-rank matrix relevant in our cases is highly sparse as well.
That said, additional structural assumptions need to be incorporated
in order to achieve optimal performance. 
\item \textbf{Graph Clustering.} The joint matching problem can be treated
as a structured graph clustering (GC) problem, where graph nodes represent
points on objects and the edge set encodes all correspondences. In
this regard, any GC algorithm \cite{bansal2004correlation,mathieu2010correlation,jalali2011clustering,chen2012clustering,jalali2012max_norm,ailon2013breaking}
provides a heuristic to estimate graph matching. Nevertheless, there
are several intrinsic structural properties herein that are not explored
by any generic GC approaches. First, our input takes a block-matrix
form, where each block is highly \emph{structured} (i.e. doubly-substochastic),
\emph{sparse}, and inter-dependent. Second, the points belonging to
the same object are \emph{mutually exclusive} to each other. Third,
the corruption rate for different entries can be highly non-symmetric
-- when translated into GC languages, this means that in-cluster edges
might suffer from an order-of-magnitude larger error rate than inter-cluster
edges. As a result, the findings for generic GC methods do not deliver
encouraging guarantees when applied to our setting. Detailed theoretical
and empirical comparisons are provided in Sections~\ref{sec:Theoretic-Guarantees}
and \ref{sec:results}, respectively.
\end{itemize}

\subsection{Organization}

The rest of the paper is organized as follows. Section \ref{sec:Problem-Formulation}
formally presents the problem setup, including the input model and
the expected output. Our two-step recovery procedure -- a spectral
method followed by a convex program called MatchLift -- is described
in Section \ref{sec:Methodology}. A scalable alternating direction
method of multipliers (ADMM) together with a greedy rounding strategy
is also introduced in Section \ref{sec:Methodology}. Section \ref{sec:Theoretic-Guarantees}
presents the main theoretical performance guarantees for our method
under a natural randomized model. All proofs of the main theorems
are deferred to the appendices. We introduce numerical experiments
demonstrating the practicability of our method in Section \ref{sec:results},
as well as empirical comparison with other best-known algorithms.
Finally, Section \ref{sec:Conclusions} concludes the paper with a
summary of our findings.

\section{Problem Formulation and Preliminaries\label{sec:Problem-Formulation}}

This section presents the problem setup for matching multiple partially
similar objects, and introduces an algebraic form for representing
a collection of pairwise maps.

\subsection{Terminology}

Below we formally define several important notions that will be used
throughout this paper.
\begin{itemize}
\item \textbf{Set.} We represent objects to be matched as discrete sets.
For example, these sets can represent the vertex sets in the graph
matching problem, or encode feature points when matching images. 
\item \textbf{Partial Map.} Given two discrete sets $\set{S}$ and $\set{S}'$,
a subset $\phi\subset\set{S}\times\set{S}'$ is termed a partial map
if each element of $\set{S}$ (resp. $\set{S}'$) is paired with \emph{at
most one} element of $\set{S}'$ (resp. $\set{S}$) --- in particular,
not all elements need to be paired. 
\item \textbf{Map Graph.} A graph $\set{G}=(\set{V},\set{E})$ is called
a map graph w.r.t. $n$ sets $\set{S}_{1},\cdots,\set{S}_{n}$ if
(i) $\set{V}:=\left\{ \set{S}_{1},\cdots,\set{S}_{n}\right\} $, and
(ii) $(\mathcal{S}_{i},\mathcal{S}_{j})\in\mathcal{E}$ implies that
pairwise estimates on the partial maps $\phi_{ij}$ and $\phi_{ji}$
between $\mathcal{S}_{i}$ and $\mathcal{S}_{j}$ are available. 
\end{itemize}

\subsection{Input and Output}

The input and expected output for the joint object matching problem
are described as follows.
\begin{itemize}
\item \textbf{Input (Noisy Pairwise Maps)}. Given $n$ sets $\set{S}_{1},\cdots,\set{S}_{n}$
with respective cardinality $m_{1},\cdots,m_{n}$ and a (possibly
sparse) map graph $\set{G}$, the input to the recovery algorithm
consists of partial maps $\phi_{ij}^{\textup{in}}\text{ }\left((i,j)\in\mathcal{G}\right)$
between $\mathcal{S}_{i}$ and $\mathcal{S}_{j}$ estimated in isolation,
using any off-the-shelf \emph{pairwise} matching method. Note that
the input maps $\phi_{ij}^{\textup{in}}$ one obtain might not agree,
partially or totally, with the ground truth.
\item \textbf{Output (Consistent Global Matching)}. The main objective of
this paper is to detect and prune incorrect pairwise input maps in
an efficient and reliable manner. Specifically, we aim at proposing
a tractable algorithm that returns a full collection of partial maps
$\left\{ \phi_{ij}\mid1\leq i,j\leq n\right\} $ that are (i) globally
consistent, and (ii) close to the provided pairwise maps -- and under
some conditions provably the ground-truth maps.
\end{itemize}
As will be detailed later, the key idea of our approach is to explore
global consistency across all pairwise maps. In fact, points across
different objects must form several clusters, and the ground-truth
maps only exhibit in-cluster edges. We will introduce a novel convex
relaxation tailored to the structure of the input maps (Section~\ref{sec:Methodology})
and investigate its theoretical performance (Section~\ref{sec:Theoretic-Guarantees}).

\subsection{Joint Matching in Matrix Form}

In the same spirit as most convex relaxation techniques (e.g.,~\cite{chen2012clustering,huang2013consistent}),
we use matrices to encode maps between objects. Specifically, we encode
a partial map $\phi_{ij}:\mathcal{S}_{i}\mapsto\mathcal{S}_{j}$ as
a binary matrix $\bs{X}_{ij}\in\{0,1\}^{|\mathcal{S}_{i}|\times|\mathcal{S}_{j}|}$
such that $\bs{X}_{ij}(s,s')=1$ iff $(s,s')\in\phi_{ij}$. Valid
partial map matrices $\bs{X}_{ij}$ shall satisfy the following doubly
sub-stochastic constraints: 
\begin{equation}
{\bf 0}\leq\bs{X}_{ij}{\bf 1}\leq{\bf 1},\quad{\bf 0}\leq\bs{X}_{ij}^{\top}{\bf 1}\leq{\bf 1}.\label{eq:substochastic}
\end{equation}
We then use an $n\times n$ block matrix $\bs{X}\in\{0,1\}^{N\times N}$
to encode the entire collection of partial maps $\left\{ \phi_{ij}\mid1\leq i,j\leq n\right\} $
over $\left\{ \mathcal{S}_{1},\cdots,\mathcal{S}_{n}\right\} $: 
\begin{align}
\boldsymbol{X}=\left(\begin{array}{cccc}
\boldsymbol{I}_{m_{1}} & \bs{X}_{12} & \cdots & \bs{X}_{1n}\\
\bs{X}_{21} & \boldsymbol{I}_{m_{2}} & \cdots & \bs{X}_{2n}\\
\vdots & \vdots & \ddots & \vdots\\
\bs{X}_{n1} & \cdots & \cdots & \boldsymbol{I}_{m_{n}}
\end{array}\right),\label{eq:GraphMapMatrix}
\end{align}
where $m_{i}:=\left|\mathcal{S}_{i}\right|$ and $N:=\sum_{i=1}^{n}m_{i}$.
Note that all diagonal blocks are identity matrices, as each object
is isomorphic to itself. 

For notational simplicity, we will use $\boldsymbol{X}^{\mathrm{in}}$
throughout to denote the collection of pairwise input maps, i.e. each
obtained pairwise estimate $\phi_{ij}^{\textup{in}}$ is encoded as
a binary map matrix $\boldsymbol{X}_{ij}^{\mathrm{in}}\in\left\{ 0,1\right\} ^{m_{i}\times m_{j}}$
obeying the constraint (\ref{eq:substochastic}). Some other useful
notation is summarized in Table \ref{tab:Summary-of-Notation}.

\begin{table}
\centering%
\begin{tabular}{c|>{\centering}p{0.7\textwidth}}
\hline 
\textbf{Symbol} & \textbf{Description} \tabularnewline
\hline 
${\bf 1}$  & \raggedright{}ones vector: a vector with all entries one\tabularnewline
\hline 
$\boldsymbol{X}_{ij}$  & \raggedright{}$(i,j)$-th block of a block matrix $\boldsymbol{X}$. \tabularnewline
\hline 
$\langle\boldsymbol{A},\boldsymbol{B}\rangle$  & \raggedright{}matrix inner product, i.e. $\langle\boldsymbol{A},\boldsymbol{B}\rangle=\text{tr}\left(\boldsymbol{A}^{\top}\boldsymbol{B}\right)$. \tabularnewline
\hline 
$\diag(\boldsymbol{X})$  & \raggedright{}a column vector formed from the diagonal of a square
matrix $\boldsymbol{X}$ \tabularnewline
\hline 
$\Diag(\boldsymbol{x})$  & \raggedright{}a diagonal matrix that puts $\boldsymbol{x}$ on the
main diagonal\tabularnewline
\hline 
$\boldsymbol{e}_{i}$  & \raggedright{}$i$th unit vector, whose $i$th component is 1 and
all others 0 \tabularnewline
\hline 
$\otimes$  & \raggedright{}tensor product, i.e. $\boldsymbol{A}\otimes\boldsymbol{B}=\left[\begin{array}{cccc}
a_{1,1}\boldsymbol{B} & a_{1,2}\boldsymbol{B} & \cdots & a_{1,n_{2}}\boldsymbol{B}\\
a_{2,1}\boldsymbol{B} & a_{2,2}\boldsymbol{B} & \cdots & a_{2,n_{2}}\boldsymbol{B}\\
\vdots & \vdots & \vdots & \vdots\\
a_{n_{1},1}\boldsymbol{B} & a_{n_{1},2}\boldsymbol{B} & \cdots & a_{n_{1},n_{2}}\boldsymbol{B}
\end{array}\right]$\tabularnewline
\hline 
$\Omega_{\mathrm{gt}},\Omega_{\mathrm{gt}}^{\perp}$ & \raggedright{}support of $\boldsymbol{X}^{\mathrm{gt}}$, its complement
support\tabularnewline
\hline 
$T_{\mathrm{gt}},T_{\mathrm{gt}}^{\perp}$ & \raggedright{}tangent space at $\boldsymbol{X}^{\mathrm{gt}}$, its
orthogonal complement\tabularnewline
\hline 
$\mathcal{P}_{\Omega_{\mathrm{gt}}}$, $\mathcal{P}_{\Omega_{\mathrm{gt}}^{\perp}}$ & \raggedright{}projection onto the space of matrices supported on $\Omega_{\mathrm{gt}}$
and $\Omega_{\mathrm{gt}}^{\perp}$, respectively\tabularnewline
\hline 
$\mathcal{P}_{T_{\mathrm{gt}}}$, $\mathcal{P}_{T_{\mathrm{gt}}^{\perp}}$ & \raggedright{}projection onto $T_{\mathrm{gt}}$ and $T_{\mathrm{gt}}^{\perp}$,
respectively\tabularnewline
\hline 
\end{tabular}\caption{\label{tab:Summary-of-Notation}Summary of Notation and Parameters}
\end{table}

\section{Methodology\label{sec:Methodology}}

This section presents a novel methodology, based on a theoretically
rigorous and numerically efficient framework.

\subsection{MatchLift: A Novel Two-Step Algorithm\label{sub:Convex}}

We start by discussing the consistency constraint on the underlying
ground-truth maps. Assume that there exists a universe $\mathcal{S}=\{1,\cdots,m\}$
of $m$ elements such that i) each object $\mathcal{S}_{i}$ is a
(partial) image of $\mathcal{S}$; ii) each element in $\mathcal{S}$
is contained in at least one object $\mathcal{S}_{i}$. Then the ground-truth
correspondences shall connect points across objects that are associated
with the same element.

Formally speaking, let the binary matrix $\boldsymbol{Y}_{i}\in\{0,1\}^{m_{i}\times m}$
encode the underlying correspondences between each point and the universe,
i.e. for any $s_{i}\in\mathcal{S}_{i}$ and $s\in\mathcal{S}$,
\[
\boldsymbol{Y}_{i}(s_{i},s)=1,\quad\text{iff}\quad s_{i}\text{ corresponds to }s.
\]
This way one can express
\[
\boldsymbol{X}=\boldsymbol{Y}\boldsymbol{Y}^{\top}
\]
with $\boldsymbol{Y}=(\boldsymbol{Y}_{1}^{\top},\cdots,\boldsymbol{Y}{}_{n}^{\top})^{\top}$,
which makes clear that
\[
\mathrm{rank}(\boldsymbol{X})=m.
\]
This is equivalent to the graph partitioning setting with $m$ cliques.
Consequently, a natural candidate is to seek a low-rank and positive
semidefinite (PSD) matrix to approximate the input. However, this
strategy does not effectively explore the sparsity structure underlying
the map collection.

To obtain a more powerful formulation, the proposed algorithm is based
on the observation that even under dense input corruption, we are
often able to obtain reliable estimates on $m$ -- the universe size,
using spectral techniques. This motivates us to incorporate the information
of $m$ into the formulation so as to develop tighter relaxation.
Specifically, we lift $\boldsymbol{X}$ with one more dimension and
consider 
\begin{equation}
\left[\begin{array}{cc}
m & {\bf 1}^{\top}\\
{\bf 1} & \boldsymbol{X}
\end{array}\right]=\left[\begin{array}{c}
\bs{1}^{\top}\\
\bs{Y}
\end{array}\right]\left[\begin{array}{cc}
\bs{1} & \bs{Y}^{\top}\end{array}\right]\succeq{\bf 0},\label{eq:X_rank_PSD}
\end{equation}
which is strictly tighter than merely imposing $\boldsymbol{X}\succeq{\bf 0}$.
Intuitively, the formulation (\ref{eq:X_rank_PSD}) entitles us one
extra degree of freedom to assist in outlier pruning, which turns
out to be crucial in ``debiasing'' the errors. Encouragingly, this
tightened constraint leads to remarkably improved theoretical guarantees,
as will be shown in Section~\ref{sec:Theoretic-Guarantees}. In the
following, we formally present our two-step matching procedure.
\begin{itemize}
\item \textbf{Step I: Estimating $m$.} We estimate $m$ by tracking the
spectrum of the input $\boldsymbol{X}^{\mathrm{in}}$. According to
common wisdom (e.g. \cite{keshavan2010matrix}), a block-sparse matrix
$\boldsymbol{X}^{\mathrm{in}}$ must first be trimmed in order to
remove the undesired bias effect caused by over-represented rows /
columns. One candidate trimming procedure is provided as follows.

\begin{itemize}
\item \textbf{Trimming Procedure}. Set $d_{\text{min}}$ to be the smallest
vertex degree of $\mathcal{G}$, and we say the a vertex is {\em
over-represented} if its vertex degree in $\mathcal{G}$ exceeds
$2d_{\text{min}}$. Then for each overrepresented vertex $i$, randomly
sample $2d_{\text{min}}$ edges incident to it and set to zero all
blocks $\boldsymbol{X}_{ij}^{\mathrm{in}}$ associated with the remaining
edges. 
\end{itemize}

With this trimming procedure, we propose to pre-estimate $m$ via
Algorithm \ref{alg:EsimateM}. 
\begin{algorithm}[h]
\caption{Estimating the size $m$ of the universe $\mathcal{S}$}

\label{alg:EsimateM} \begin{algorithmic} \STATE 1) trim $\boldsymbol{X}^{\mathrm{in}}$,
and let $\tilde{\boldsymbol{X}}^{\mathrm{in}}$ be the output. \STATE
2) perform eigenvalue decomposition on $\tilde{\boldsymbol{X}}^{\mathrm{in}}$;
denote by $\lambda_{i}$ the $i$th largest eigenvalue. \STATE 3)
\textbf{output}: $\hat{m}:=\arg\max\nolimits _{M\leq i<N}|\lambda_{i}-\lambda_{i+1}|$,
where $M=\max\{2,\max_{1\leq i\leq n}m_{i}\}$. 
\end{algorithmic} 
\end{algorithm}

In short, Algorithm \ref{alg:EsimateM} returns an estimate of $m$
via spectral methods, which outputs the number of dominant principal
components of $\boldsymbol{X}^{\mathrm{in}}$.

\item \textbf{Step II: Map Recovery.} Now that we have obtained an estimate
on $m$, we are in position to present our optimization heuristic
that exploits the structural property (\ref{eq:X_rank_PSD}). In order
to guarantee that the recovery is close to the provided maps $\phi_{ij}^{\textup{in}}$,
one alternative is to maximize correspondence agreement (i.e. the
number of compatible non-zero entries) between the input and output.
This results in an objective function: 
\[
\sum_{(i,j)\in\mathcal{G}}\langle\boldsymbol{X}_{ij}^{\textup{in}},\boldsymbol{X}_{ij}\rangle.
\]
Additionally, since a {\em non-negative} map matrix $\boldsymbol{X}$
is inherently sparse, it is natural to add an $\ell_{1}$ regularization
term to encourage sparsity, which in our case reduces to 
\[
\langle{\bf {1}\cdot{\bf {1}^{\top},\boldsymbol{X}\rangle}}.
\]
Since searching over all 0-1 map matrices is intractable, we propose
to relax the binary constraints. Putting these together leads to the
following semidefinite program referred to as \emph{MatchLift}:
\begin{align}
(\text{MatchLift})\text{ }\text{ }\underset{\boldsymbol{X}\in\mathbb{R}^{N\times N}}{\text{maximize}}\text{ } & \sum\limits _{(i,j)\in\mathcal{G}}\langle\boldsymbol{X}_{ij}^{\textup{in}},\boldsymbol{X}_{ij}\rangle-\lambda\langle{\bf 1\cdot{\bf 1^{\top},\boldsymbol{X}\rangle}}\nonumber \\
\text{subject to}\quad & \boldsymbol{X}_{ii}=\boldsymbol{I}_{m_{i}},\quad1\leq i\leq n,\nonumber \\
 & \boldsymbol{X}\geq{\bf 0},\nonumber \\
 & \left[\begin{array}{cc}
m & {\bf 1}^{\top}\\
{\bf 1} & \boldsymbol{X}
\end{array}\right]\succeq{\bf 0}.\label{eq:SDP_conic}
\end{align}
\begin{remark}Here, $\lambda$ represents the regularization parameter
that balances the compatibility to the input and the sparsity structure.
As we will show, the recovery ability of MatchLift is not sensitive
to the choice of $\lambda$. By default, one can set 
\begin{equation}
\lambda=\frac{\sqrt{|\mathcal{E}|}}{2n},\label{eq:LambdaDefault}
\end{equation}
which results in a \emph{parameter-free} formulation.\end{remark}\begin{remark}Careful
readers will note that the set of doubly stochastic constraints (\ref{eq:substochastic})
can be further added into the program. Nevertheless, while enforcement
of these constraints (\ref{eq:substochastic}) results in a strictly
tighter relaxation, it only leads to marginal improvement when (\ref{eq:SDP_conic})
is present. As a result, we remove them for the sake of computational
efficiency. We note, however, that in the scenario where $m$ is difficulty
to estimate, imposing (\ref{eq:substochastic}) will ``become crucial
in allowing a constant fraction (e.g. 50\%) of error rate, although
dense error correction might not be guaranteed.\end{remark}
\end{itemize}
This algorithm, all at once, attempts to disentangle the ground truth
and outliers as well as predict unobserved maps via convex relaxation,
inspired by recent success in sparse and low-rank matrix decomposition
\cite{CanLiMaWri09,chandrasekaran2011rank}. Since the ground truth
matrix is \emph{simultaneously} low-rank and sparse; existing methodologies,
which focus on dense low-rank matrices, typically yield loose, uninformative
bounds in our setting. 

Finally, we note that our matching algorithm and main results are
well suited for a broad class of scenarios where each pairwise input
can be modeled as a (partial) permutation matrix. For instance, our
setting subsumes phase correlation \cite{horner1984phase}, angular
synchronization \cite{singer2011angular}, and multi-signal alignment
\cite{bandeira2014multireference} as special cases.


\subsection{Alternating Direction Methods of Multipliers (ADMM)\label{sub:ADMM}}

Most advanced off-the-shelf SDP solvers like SeDuMi or MOSEK are typically
based on interior point methods, and such second-order methods are
unable to handle problems with large dimensionality. For practical
applicability, we propose a first-order optimization algorithm for
approximately solving MatchLift, which is a variant of the ADMM method
for semidefinite programs presented in \cite{wen2010alternating}.
Theoretically it is guaranteed to converge. Empirically, it is often
the case that ADMM converges to modest accuracy within a reasonable
amount of time, and produces desired results with the assistance of
appropriate rounding procedures. This feature makes ADMM practically
appealing in our case since the ground-truth matrix is known to be
a 0-1 matrix, for which moderate entry-wise precision is sufficient
to ensure good rounding accuracy. The details of the ADMM algorithm
are deferred to Appendix \ref{sec:ADMM-Appendix}.

\subsection{Rounding Strategy\label{sub:Rounding}}

As MatchLift solves a relaxed program of the original convex problem,
it may return fractional solutions. In this case, we propose a greedy
rounding method to generate valid partial maps. Given the solution
$\hat{\boldsymbol{X}}$ to MatchLift, the proposed strategy proceeds
as in Algorithm \ref{alg:Rounding}. One can verify that this simple
deterministic rounding strategy returns a matrix that encodes a consistent
collection of partial maps. Note that $\boldsymbol{v}_{i}^{T}$ denotes
the $i$th row of a matrix $\boldsymbol{V}$.

\begin{algorithm}
\caption{Rounding Strategy}

\label{alg:Rounding} \begin{algorithmic} \STATE \textbf{initialize}
$\text{ }$compute the top $r$ eigenvalues $\boldsymbol{\Sigma}=\diag(\sigma_{1},\cdots,\sigma_{r})$
and eigenvectors $\boldsymbol{U}=(\boldsymbol{u}_{1},\cdots,\boldsymbol{u}_{r})$
of $\hat{\boldsymbol{X}}$, where $r$ is an estimate of the total
number distinctive points to be recovered. Form $\boldsymbol{V}=\boldsymbol{U}\boldsymbol{\Sigma}^{\frac{1}{2}}$.
\REPEAT \STATE 1) Let $\boldsymbol{O}$ be a unitary matrix that
obeys $\boldsymbol{O}\boldsymbol{v}_{1}=\boldsymbol{e}_{1}$, and
set $\boldsymbol{V}\leftarrow\boldsymbol{V}\boldsymbol{O}^{\top}$.
\STATE 2) For each of the remaining rows $\boldsymbol{v}_{i}$ belonging
to each set $\mathcal{S}_{j}$ ($i\in\mathcal{S}_{j}$), perform 
\[
\boldsymbol{v}_{i}\leftarrow\boldsymbol{e}_{1},\quad\text{if }\left\langle \boldsymbol{v}_{i},\boldsymbol{v}_{1}\right\rangle >0.5\text{ and }i=\arg\max_{l\in\mathcal{S}_{j}}\left\langle \boldsymbol{v}_{l},\boldsymbol{v}_{1}\right\rangle .
\]
\STATE 3) All indices $i$ obeying $\boldsymbol{v}_{i}=\boldsymbol{e}_{1}$
are declared to be matched with each other, and are then removed.
Repeat 1) for the next row that has not been fixed. 
 \UNTIL{all the rows of $\bs{V}$ have been fixed.} 
\end{algorithmic} 
\end{algorithm}

\section{Theoretical Guarantees: Exact Recovery\label{sec:Theoretic-Guarantees}}

Our heuristic algorithm MatchLift recovers, under a natural randomized
setting, the ground-truth maps even when only a vanishing portion
of the input correspondences are correct. Furthermore, MatchLift succeeds
with minimal input complexity, namely, the algorithm is guaranteed
to work as soon as those input maps that coincide with the ground
truth maps form a connected map graph.

\subsection{Randomized Model\label{sub:Randomized-Model}}

In the following, we present a natural randomized model, under which
the feature of MatchLift is easiest to interpret. Specifically, consider
a universe $[m]:=\left\{ 1,2,\cdots,m\right\} $. The randomized setting
consider herein is generated through the following procedure.
\begin{itemize}
\item For each set $\mathcal{S}_{i}$ ($1\leq i\leq n$), each point $s\in[m]$
is included in $\mathcal{S}_{i}$ independently with probability $p_{\mathrm{set}}$. 
\item Each $\boldsymbol{X}_{ij}^{\mathrm{in}}$ is observed / computed independently
with probability $p_{\mathrm{obs}}$. 
\item Each observed $\boldsymbol{X}_{ij}^{\mathrm{in}}$ coincides with
the ground truth independently with probability $p_{\mathrm{true}}=1-p_{\mathrm{false}}$. 
\item Each {\em observed but incorrect} $\boldsymbol{X}_{ij}^{\mathrm{in}}$
is independently drawn from a set of partial map matrices satisfying
\begin{equation}
\mathbb{E}\boldsymbol{X}_{ij}^{\mathrm{in}}=\frac{1}{m}{\bf {1}\cdot{\bf {1}}}^{\top},\text{ }\text{if }\boldsymbol{X}_{ij}^{\mathrm{in}}\text{ is observed and corrupted.}\label{eq:MeanOutlier}
\end{equation}

\end{itemize}
\begin{remark}The above mean condition (\ref{eq:MeanOutlier}) holds,
for example, when the augmented block (i.e. that obtained by enhancing
$\mathcal{S}_{i}$ and $\mathcal{S}_{j}$ to have all $m$ elements)
is drawn from the entire set of permutation matrices or other symmetric
groups uniformly at random. While we impose (\ref{eq:MeanOutlier})
primarily to simplify our presentation of the analysis, we remark
that this assumption can be significantly relaxed without degrading
the matching performance. \end{remark}

\begin{remark}We also note that the outliers do not need to be generated
in an i.i.d. fashion. Our main results hold as long as they are jointly
independent and satisfy the mean condition (\ref{eq:MeanOutlier}).\end{remark}

\subsection{Main Theorem: Near-Optimal Matching\label{sub:Main-Theorem}}

We are now in position to state our main results, which provide theoretical
performance guarantees for our algorithms. 

\begin{theorem}[{\bf Accurate Estimation of $m$}]\label{thm:SpectralMethod}Consider
the above randomized model. There exists an absolute constant $c_{1}>0$
such that with probability exceeding $1-\frac{1}{m^{5}n^{5}}$, the
estimate on $m$ returned by Algorithm \ref{alg:EsimateM} is exact
as long as 
\begin{align}
p_{\mathrm{true}}\geq\frac{c_{1}\log^{2}\left(mn\right)}{\sqrt{np_{\mathrm{obs}}}p_{\mathrm{set}}}.\label{eq:SpectralCond}
\end{align}
\end{theorem}

\begin{proof}See Appendix \ref{sec:Proof_thm:SpectralMethod-m}.\end{proof}

Theorem \ref{thm:SpectralMethod} ensures that one can obtain perfect
estimate on the universe size or, equivalently, the rank of the ground
truth map matrix via spectral methods. With accurate information on
$m$, MatchLift allows perfect matching from densely corrupted inputs,
as revealed below.

\begin{theorem}[\textbf{Exact and Robust Matching}]\label{thm:RandomGraph-1}Consider
the randomized model described above. There exist universal constants
$c_{0},c_{1},c_{2}>0$ such that for any
\begin{align}
c_{1}\left(\frac{p_{\mathrm{obs}}}{m}+\sqrt{\frac{p_{\mathrm{obs}}\log(mn)}{np_{\mathrm{set}}^{3}}}\right)\leq\lambda\leq\frac{\sqrt{p_{\mathrm{obs}}\log\left(mn\right)}}{p_{\mathrm{set}}},\label{eq:Lambda_Range}
\end{align}
if the non-corruption rate obeys 
\begin{align}
p_{\mathrm{true}}>\frac{c_{0}\log^{2}\left(mn\right)}{\sqrt{np_{\mathrm{obs}}}p_{\mathrm{set}}^{2}},\label{eq:recovery_condition}
\end{align}
then the solution to MatchLift is exact and unique with probability
exceeding $1-\left(mn\right)^{-3}$.\end{theorem}

\begin{proof}See Appendix \ref{sec:Proof-of-Theorem-RandomGraph}.\end{proof}

Note that the performance is not sensitive to $\lambda$ as it can
be arbitrarily chosen between $\Theta\left(\sqrt{\frac{p_{\mathrm{obs}}}{n}}\right)$
and $\Theta(\sqrt{p_{\mathrm{obs}}})$. 
The implications of Theorem \ref{thm:RandomGraph-1} are summarized
as follows.
\begin{enumerate}
\item \textbf{Near-Optimal Recovery under Dense Errors}. Under the randomized
model, MatchLift succeeds in pruning all outliers and recovering the
ground truth with high probability. Somewhat surprisingly, this is
guaranteed to work even when the non-corrupted pairwise maps account
for only a {\em vanishing} fraction of the inputs. As a result,
MatchLift achieves near-optimal recovery performance in the sense
that as the number $n$ of objects grows, its outlier-tolerance rate
can be arbitrarily close to $1$. Equivalently speaking, in the asymptotic
regime, almost all input maps -- more precisely, a fraction
\begin{equation}
1-\Omega\left(\frac{\log^{2}n}{\sqrt{n}}\right)
\end{equation}
of inputs -- can be badly corrupted by random errors without degrading
the matching accuracy. This in turn highlights the significance of
joint object matching: no matter how noisy the input sources are,
perfect matching can be obtained as long as sufficiently many instances
are available.

To the best of our knowledge, none of the prior results can support
perfect recovery with more than 50\% corruptions, regardless of how
large $n$ can be. The only comparative performance is reported for
the robust PCA setting, where semidefinite relaxation enables dense
error correction \cite{ganesh2010dense,YudongRPCA2013}. However,
their condition cannot be satisfied in our case. Experimentally, applying
RPCA on joint matching is unable to tolerate dense errors (see Section~\ref{sec:results}).

\item \textbf{Exact Matching of Partially Similar Objects}. The challenge
for matching partially similar objects arises in that the overlapping
ratio between each pair of objects is in the order of $p_{\mathrm{set}}^{2}$
while the size of each object is in the order of $p_{\mathrm{set}}$.
As correct correspondences only come from overlapping regions, it
is expected that with a fixed $p_{\mathrm{false}}$, the matching
ability degrades when $p_{\mathrm{set}}$ decreases, which coincides
with the bound in (\ref{eq:recovery_condition}). However, the order
of fault-tolerance rate with $n$ is independent of $p_{\mathrm{set}}$
as long as $p_{\mathrm{set}}$ is bounded away from 0. 




\item \textbf{Minimal Input Complexity}. Suppose that $p_{\mathrm{set}}$
and $p_{\mathrm{false}}$ are both constants bounded away from 0 and
1, and that $m=n^{O(\mathrm{poly}\log(n))}$. Condition (\ref{eq:recovery_condition})
asserts that: the algorithm is able to separate outliers and fill
in all missing maps reliably with no errors, as soon as the input
complexity (i.e. the number of pairwise maps provided) is about the
order of $n\mathrm{poly}\log(n)$. Recall that the connectivity threshold
for an Erd\H{o}s\textendash{}Renyi graph $\mathcal{G}(n,p_{\mathrm{obs}})$
is $p_{\mathrm{obs}}>\frac{\log n}{n}$ (see \cite{durrett2007random}).
This implies that MatchLift allows exact recovery nearly as soon as
the input complexity exceeds the information theoretic limits.
\end{enumerate}

\subsection{Comparison with Prior Approaches\label{sub:Comparison-Prior-Approaches}}

Our exact recovery condition significantly outperforms the best-known
performance guarantees, including various SDP heuristics for matching
problems, as well as general graph clustering approaches when applied
to object matching, detailed below.
\begin{itemize}
\item \textbf{Semidefinite Programming}: The SDP formulation proposed by
Wang and Singer \cite{WangSingerRotation} admits exact recovery in
the {\em full-similarity} setting when $p_{\mathrm{true}}>c_{1}$
for some absolute constant $c_{1}\approx50\%$ in the asymptotic regime.
One might also attempt recovery by minimizing a weighted sum of nuclear
norm and $\ell_{1}$ norm as suggested in matrix completion \cite{ExactMC09}
and robust PCA \cite{CanLiMaWri09,chandrasekaran2011rank}. In order
to enable dense error correction, robust PCA requires the sparse components
(which is $\boldsymbol{X}^{\mathrm{in}}-\boldsymbol{X}^{\mathrm{gt}}$
here with $\bs{X}^{\mathrm{gt}}$ denoting the ground truth) to exhibit
random signs \cite{ganesh2010dense,YudongRPCA2013}. This cannot be
satisfied in our setting since the sign pattern of $\boldsymbol{X}^{\mathrm{in}}-\boldsymbol{X}^{\mathrm{gt}}$
is highly biased (i.e. all non-negative entries of $\boldsymbol{X}^{\mathrm{in}}-\boldsymbol{X}^{\mathrm{gt}}$
lying in the support of $\boldsymbol{X}^{\mathrm{gt}}$ have negative
signs, while all non-negative entries of $\boldsymbol{X}^{\mathrm{in}}-\boldsymbol{X}^{\mathrm{gt}}$
outside the support of $\boldsymbol{X}^{\mathrm{gt}}$ have positive
signs).
\item \textbf{Graph Clustering}: Various approaches for general graph clustering
have been proposed with theoretical guarantees under different randomized
settings \cite{jalali2011clustering,jalali2012max_norm,mathieu2010correlation}.
These results typically operate under the assumption that in-cluster
and inter-cluster correspondences are independently corrupted, which
does not apply in our model. Due to the block structure input model,
these two types of corruptions are highly correlated and usually experience
order-of-magnitude difference in corruption rate (i.e. $\left(1-p_{\mathrm{true}}\right)\frac{m-1}{m}$
for in-cluster edges and $\left(1-p_{\mathrm{true}}\right)\frac{1}{m}$
for inter-cluster edges). To facilitate comparison, we evaluate the
most recent deterministic guarantees obtained by \cite{jalali2012max_norm}.
The key metric $D_{\text{max}}$ therein can be easily bounded by
$D_{\text{max}}\geq1-p_{\text{true}}$ due to a significant degree
of in-cluster edge errors. The recovery condition therein requires
\[
D_{\text{max}}<\frac{1}{m+1},\quad\Rightarrow\quad p_{\mathrm{true}}>\frac{m}{m+1},
\]
which does not deliver encouraging guarantees compared with $p_{\mathrm{true}}>\Theta\left(\frac{\log^{2}n}{\sqrt{n}}\right)$
achieved by MatchLift. 
\end{itemize}

\section{Experimental Evaluation}

\label{sec:results}

In this section, we evaluate the performance of MatchLift and compare
it against ~\cite{jalali2011clustering} and other graph matching
methods. We consider both synthetic examples, which are used to verify
the exact recovery conditions described above, as well as popular
benchmark datasets for evaluating the practicability on real-world
images.

\subsection{Synthetic Examples}

We follow the randomized model described in Section \ref{sec:Theoretic-Guarantees}
to generate synthetic examples. For simplicity, we only consider the
full observation mode, which establishes input maps between all pairs
of objects. In all examples, we fix the universe size such that it
consists of $m=16$ points. We then vary the remaining parameters,
i.e., $n$, $p_{\mathrm{set}}$ and $p_{\mathrm{false}}$, to assess
the performance of an algorithm. We evaluate $31\times36$ sets of
parameters for each scenario, where each parameter configuration is
simulated by 10 Monte Carlo trials. The empirical success probability
is reflected by the color of each cell. Blue denotes perfect recovery
in all experiments, and red denotes failure for all trials.

\begin{figure*}[htp]
\centering %
\begin{tabular}{cc}
\includegraphics[width=0.32\columnwidth]{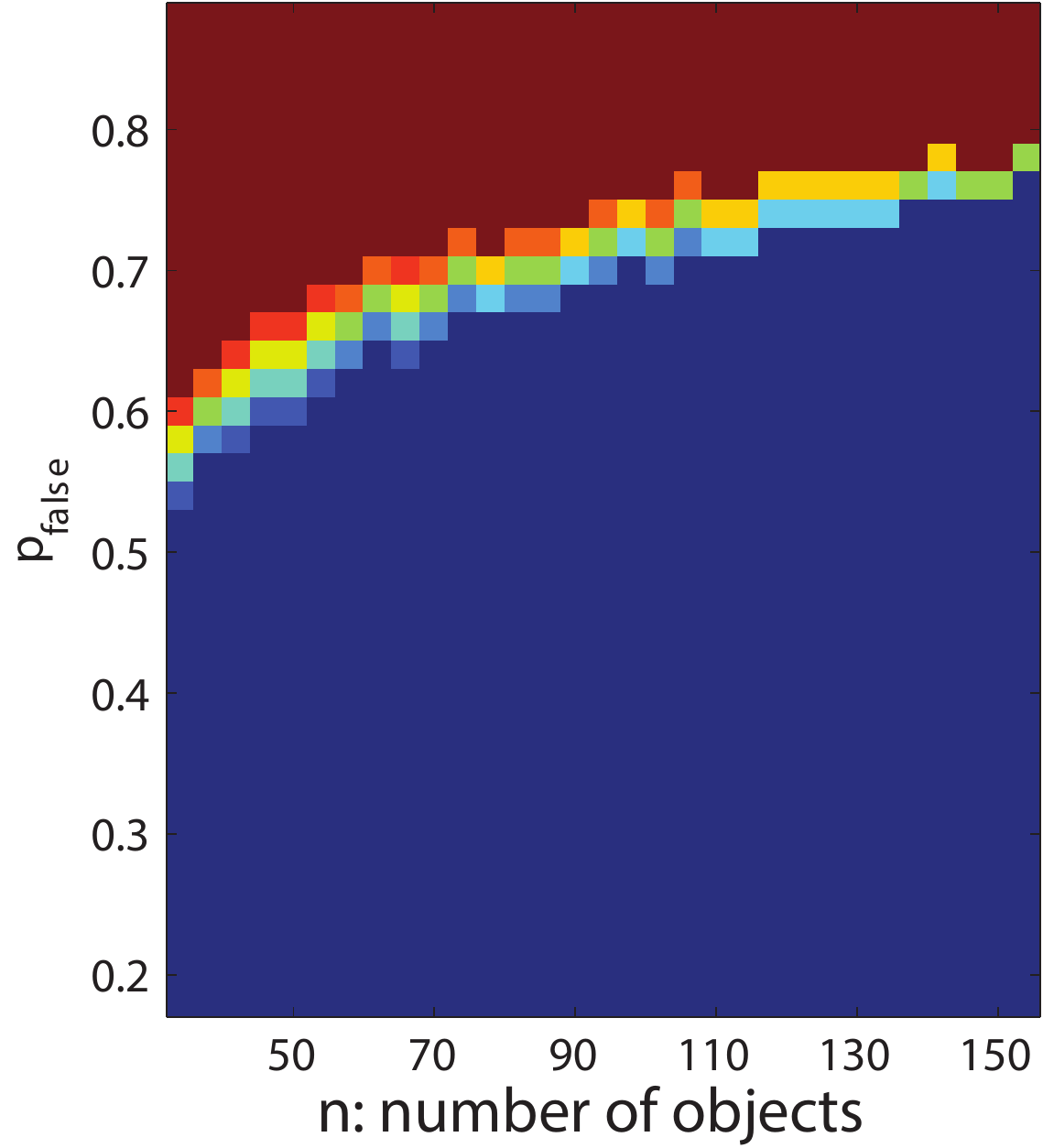} & \includegraphics[width=0.32\columnwidth]{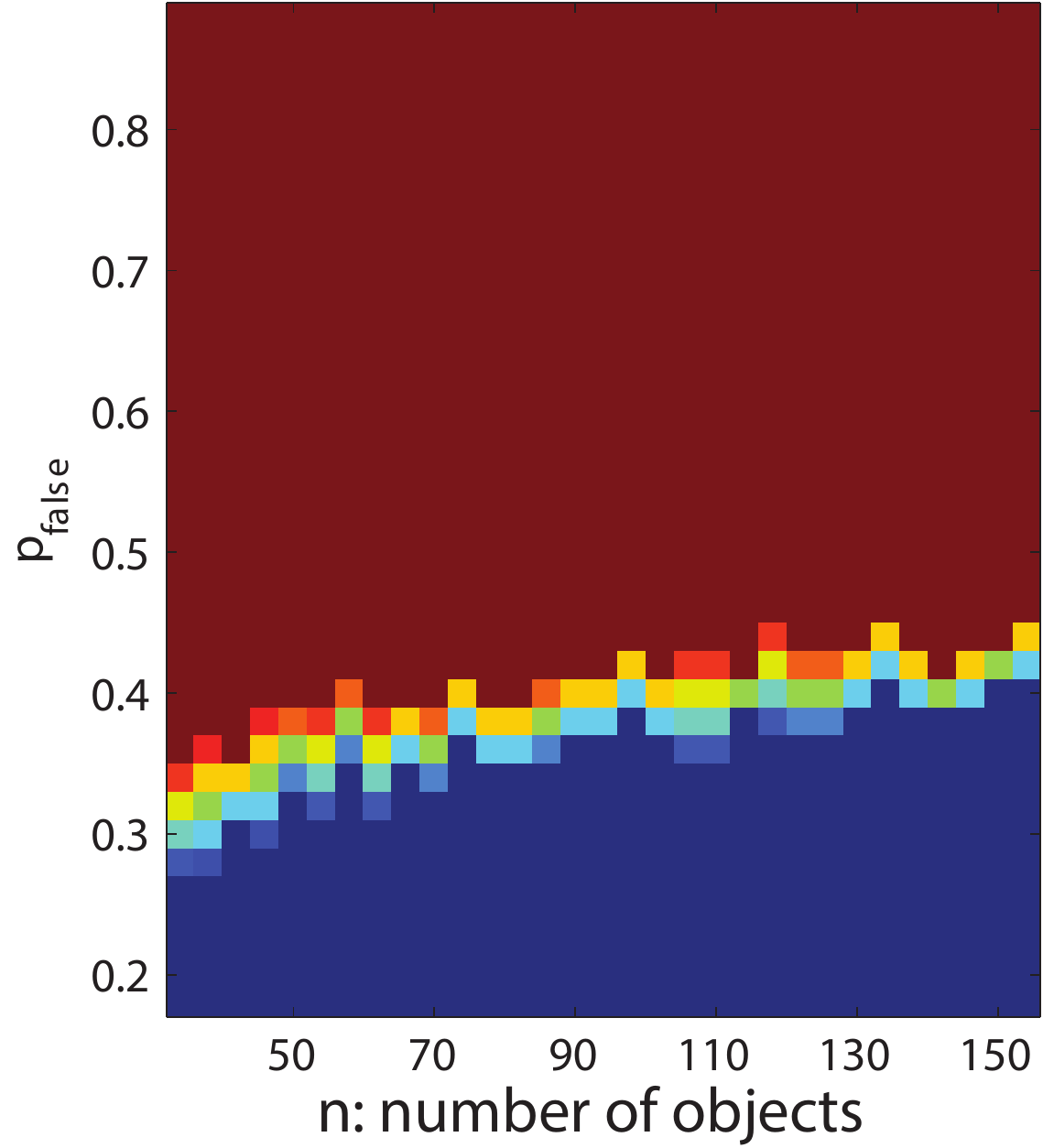}\tabularnewline
(a)  & (b) \tabularnewline
\includegraphics[width=0.32\columnwidth]{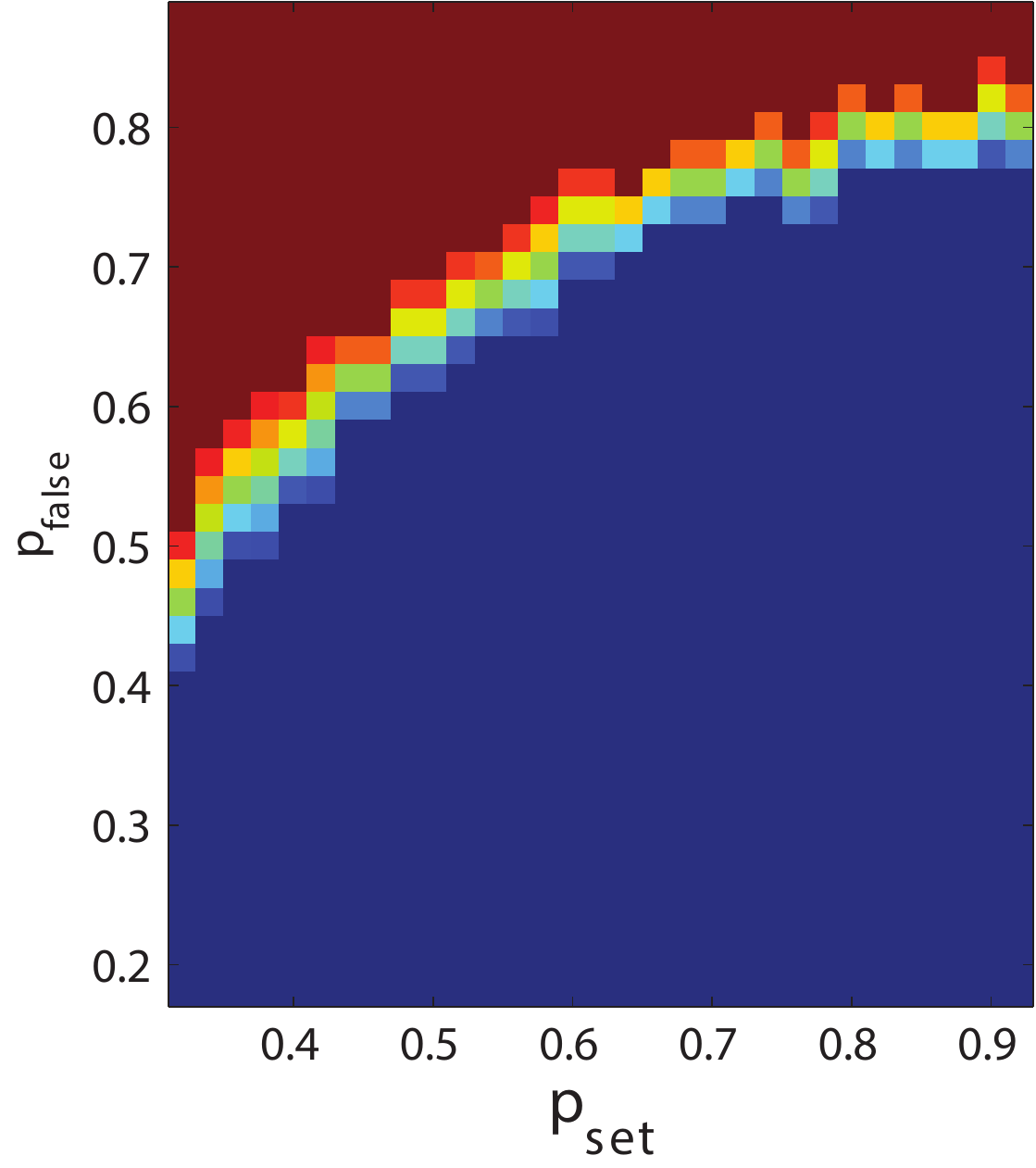} & \includegraphics[width=0.32\columnwidth]{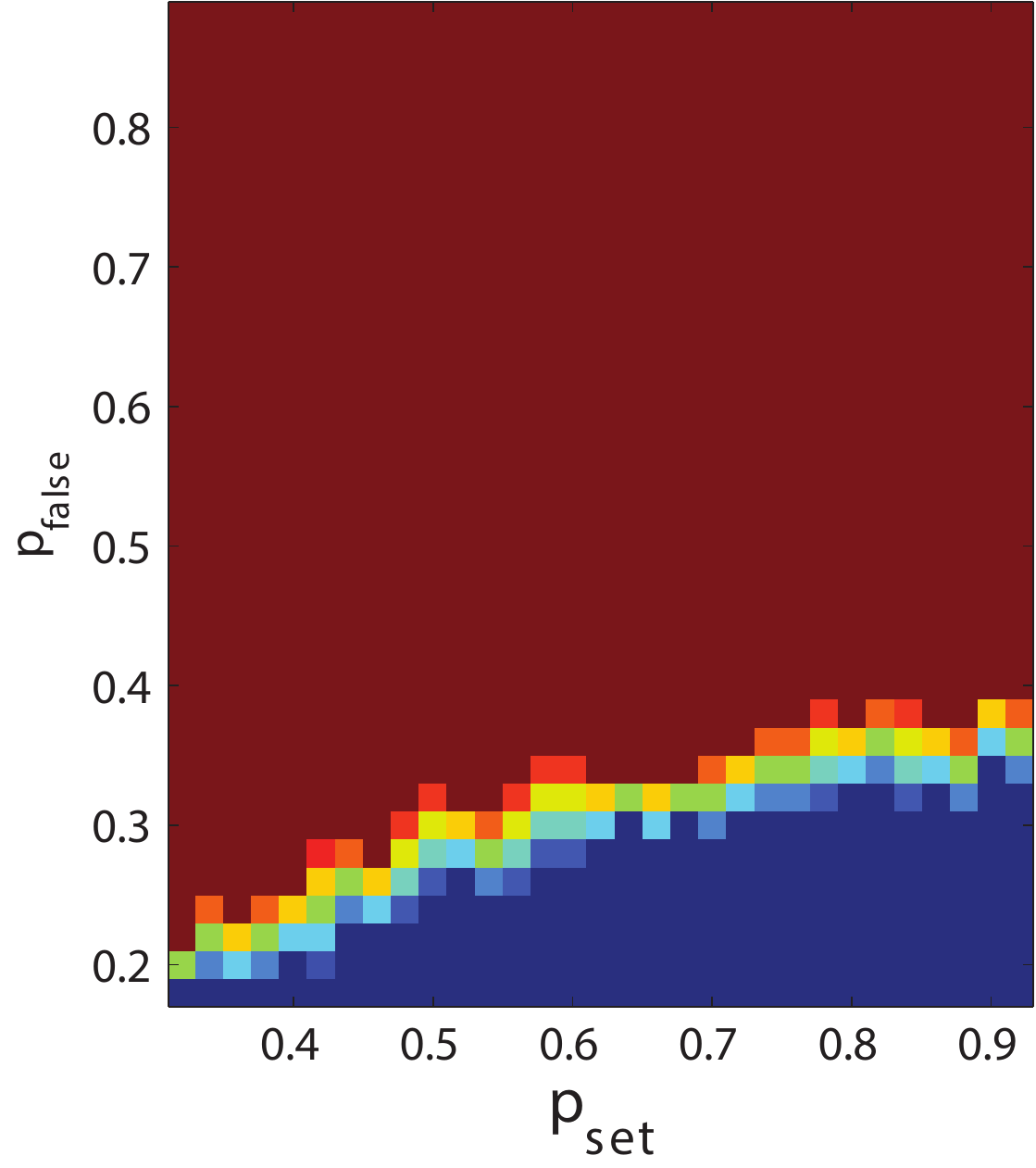}\tabularnewline
(c) & (d)\tabularnewline
\end{tabular}\caption{Phase Transition Diagrams of the proposed approach (MatchLift) and
\cite{jalali2011clustering}. We can see that MatchLift can recover
the ground-truth maps even the majority of the input correspondences
are wrong, while the exact recovery of \cite{jalali2011clustering}
requires that the percentage of incorrect correspondences is less
than $50\%$. (a-b) $p_{\mathrm{set}}=0.6$. (c-d) $n=100$.}

\label{Fig:PT} 
\end{figure*}

Figure~\ref{Fig:PT}(a) illustrates the phase transition for $p_{\mathrm{set}}=0.6$,
when the number of objects $n$ and $p_{\textup{false}}$ vary. We
can see that MatchLift is exact even when the majority of the input
correspondences are incorrect (e.g., $75\%$ when $n=150$). This
is consistent with the theoretical result that the lower bound on
$p_{\mathrm{true}}$ for exact recovery is $O(\log^{2}n/\sqrt{n})$.

Figure~\ref{Fig:PT}(c) shows the phase transition for $n=100$,
when $p_{\mathrm{set}}$ and $p_{\mathrm{false}}$ vary. We can see
that MatchLift tolerates more noise when $p_{\mathrm{set}}$ is large.
This is also consistent with the result that the error-correction
ability improves with $p_{\mathrm{set}}$.

In comparison, Figure~\ref{Fig:PT}(b) and Figure~\ref{Fig:PT}(d)
illustrate the phase transition diagrams achieved by the algorithm
proposed in \cite{jalali2011clustering}. One can see that MatchLift
is empirically superior, as \cite{jalali2011clustering} is unable
to allow dense error correction in our case.

\subsection{Real-World Examples}

\begin{figure}
\centering\includegraphics[width=0.8\linewidth]{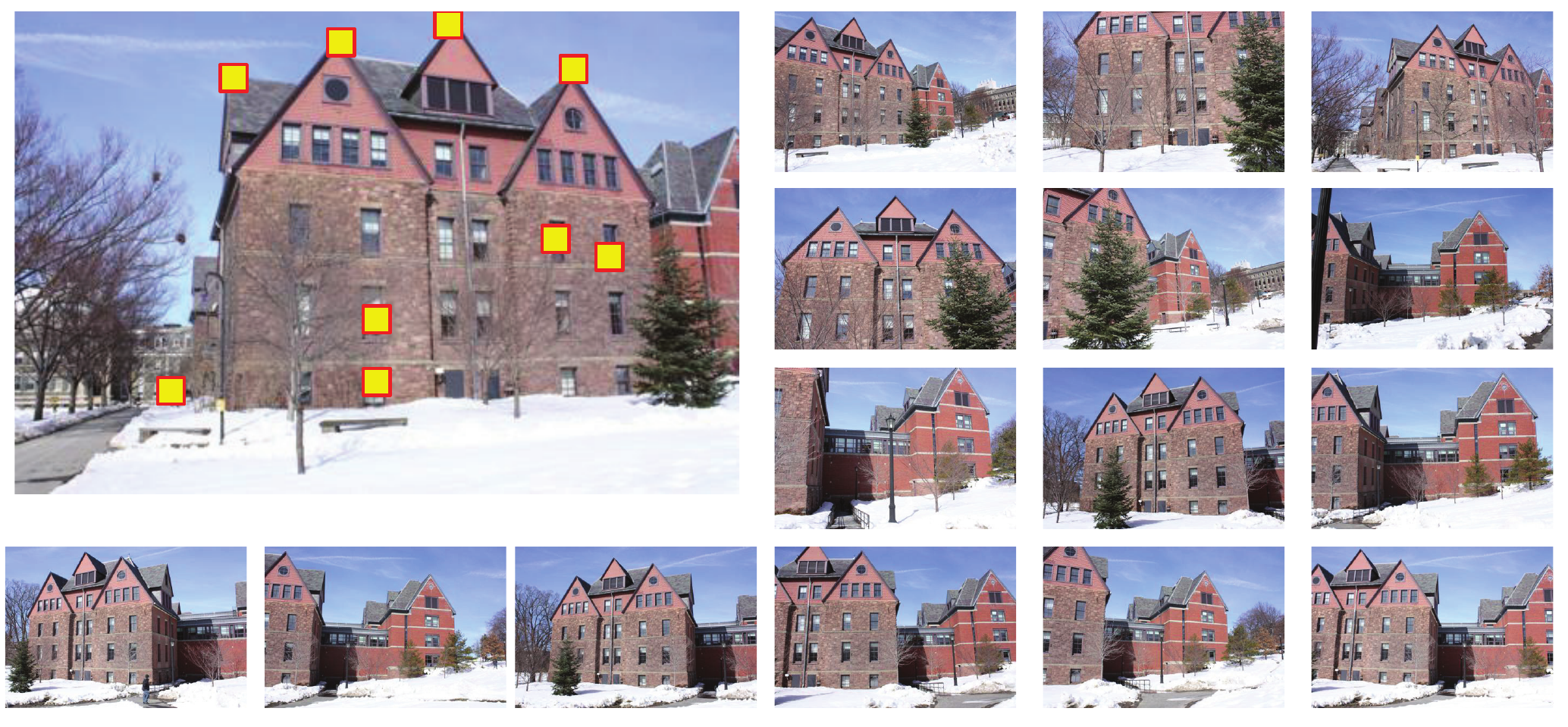}
\caption{A small benchmark called Building created for matching multiple images
with partial similarity. Manually labeled feature points are highlighted. }

\label{Fig:Benchmark-building} 
\end{figure}

\begin{figure}
\centering\includegraphics[width=0.8\linewidth]{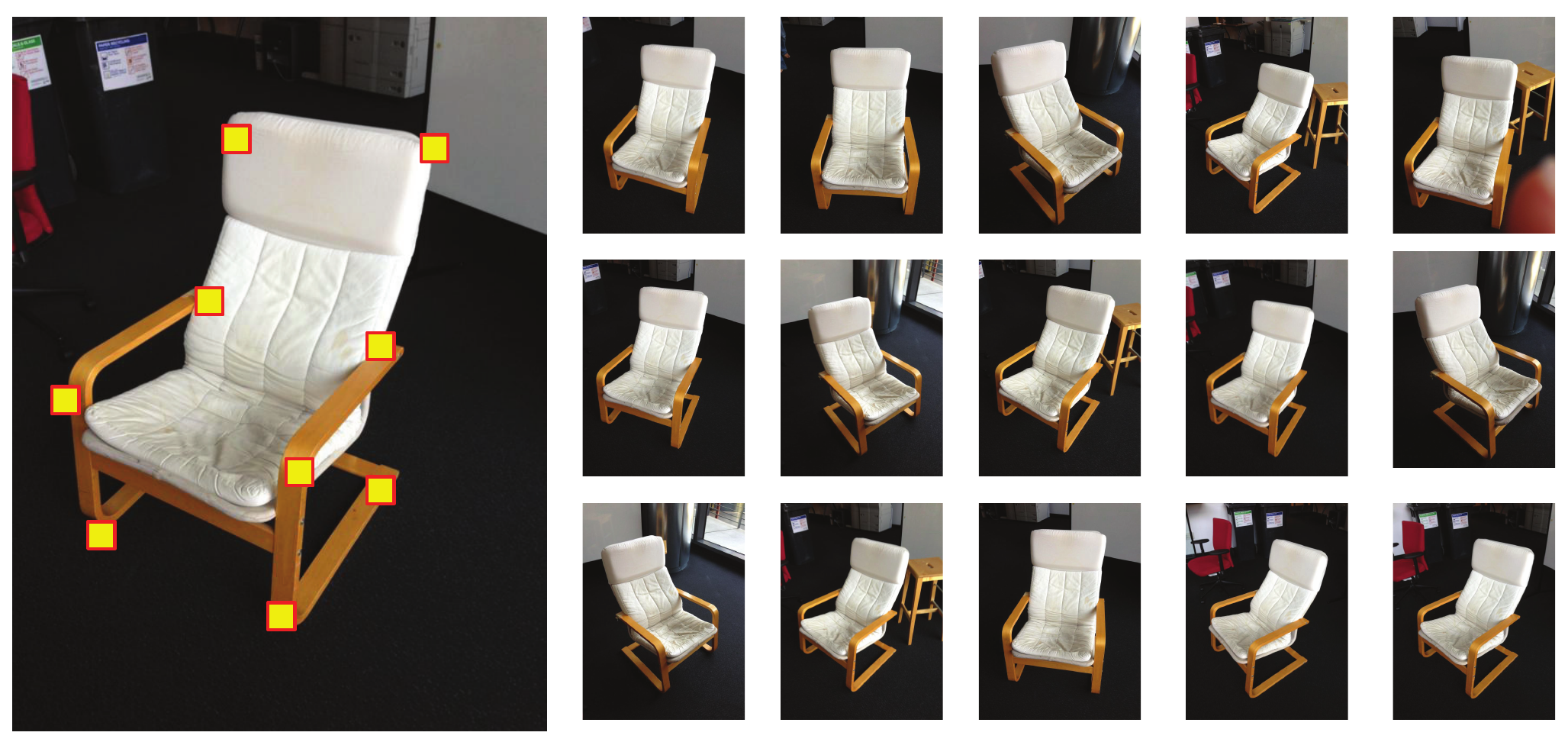}
\caption{A small benchmark called Chair created for matching multiple images
with partial similarity. Manually labeled feature points are highlighted. }

\label{Fig:Benchmark-chair} 
\end{figure}

We have applied our algorithm on six benchmark datasets, i.e., CMU-House,
CMU-Hotel, two datasets (Graf and Bikes) from~\cite{Mikolajczyk:2005:PEL}%
\footnote{available online: robots.ox.ac.uk/~vgg/research/affine%
} and two new datasets (referred as Chair and Building, respectively)
designed for evaluating joint partial object matching. As shown in
Figures~\ref{Fig:Benchmark-building} and \ref{Fig:Benchmark-chair},
the Building data set contains $16$ images taken around a building~\cite{crandall2011discrete},
while the Chair data set contains $16$ images of a chair model from
different viewpoints. In the following, we first discuss the procedure
for generating the input to our algorithm, i.e., the input sets and
the initial maps. We then present the evaluation setup and analyze
the results.
\begin{itemize}
\item \textbf{Feature points and initial maps.} To make fair comparisons
with previous techniques on CMU-House and CMU-Hotel, we use the features
points provided in~\cite{caetano2009learning} and apply the spectral
matching algorithm described in ~\cite{Leordeanu:2005:SM} to establish
initial maps between features points. To assess the performance of
the proposed algorithm with sparse input maps, we only match each
image with $10$ random neighboring images.
\begin{figure*}[htp]
\centering %
\begin{tabular}{cc}
\includegraphics[width=0.3\columnwidth]{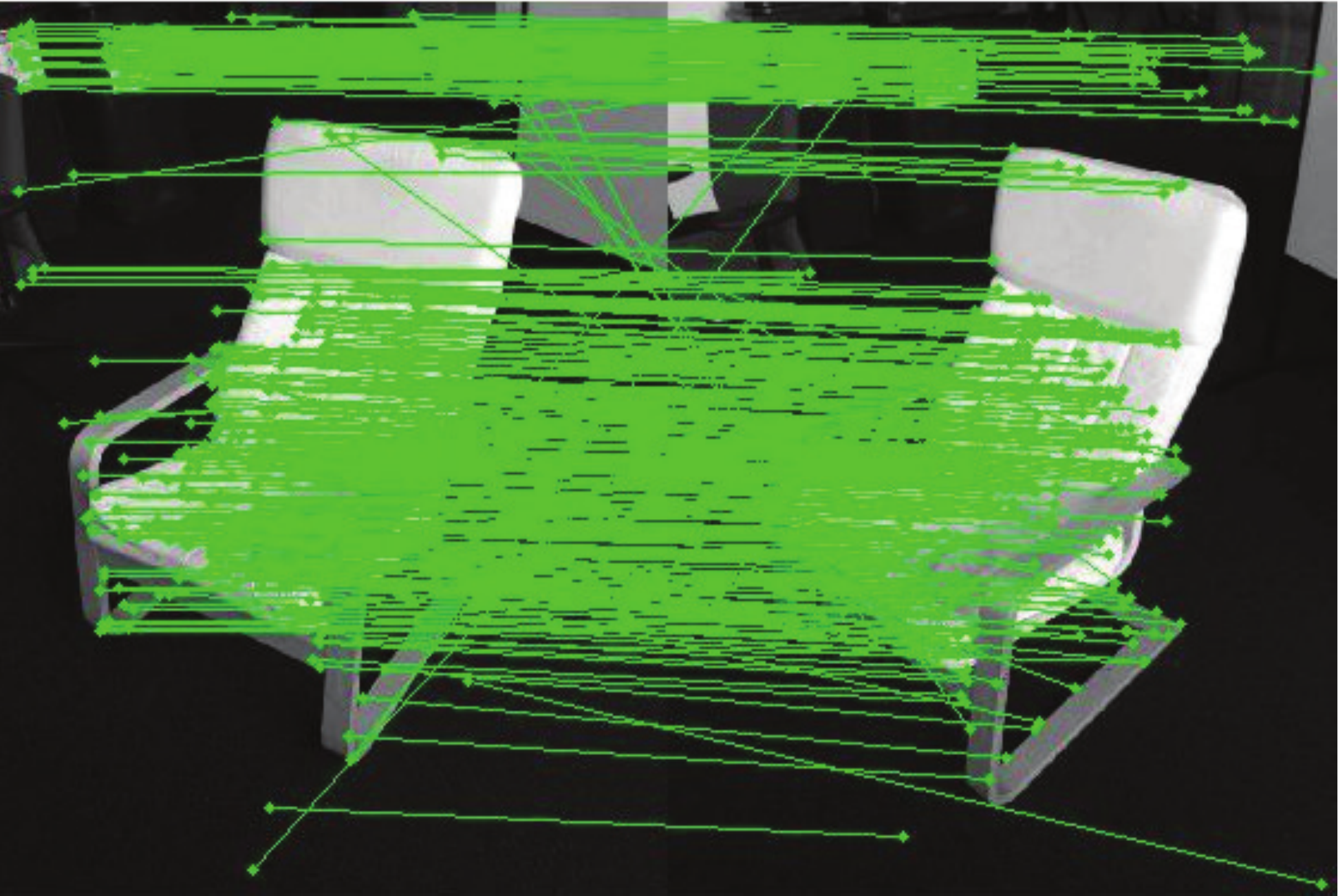} & \includegraphics[width=0.3\columnwidth]{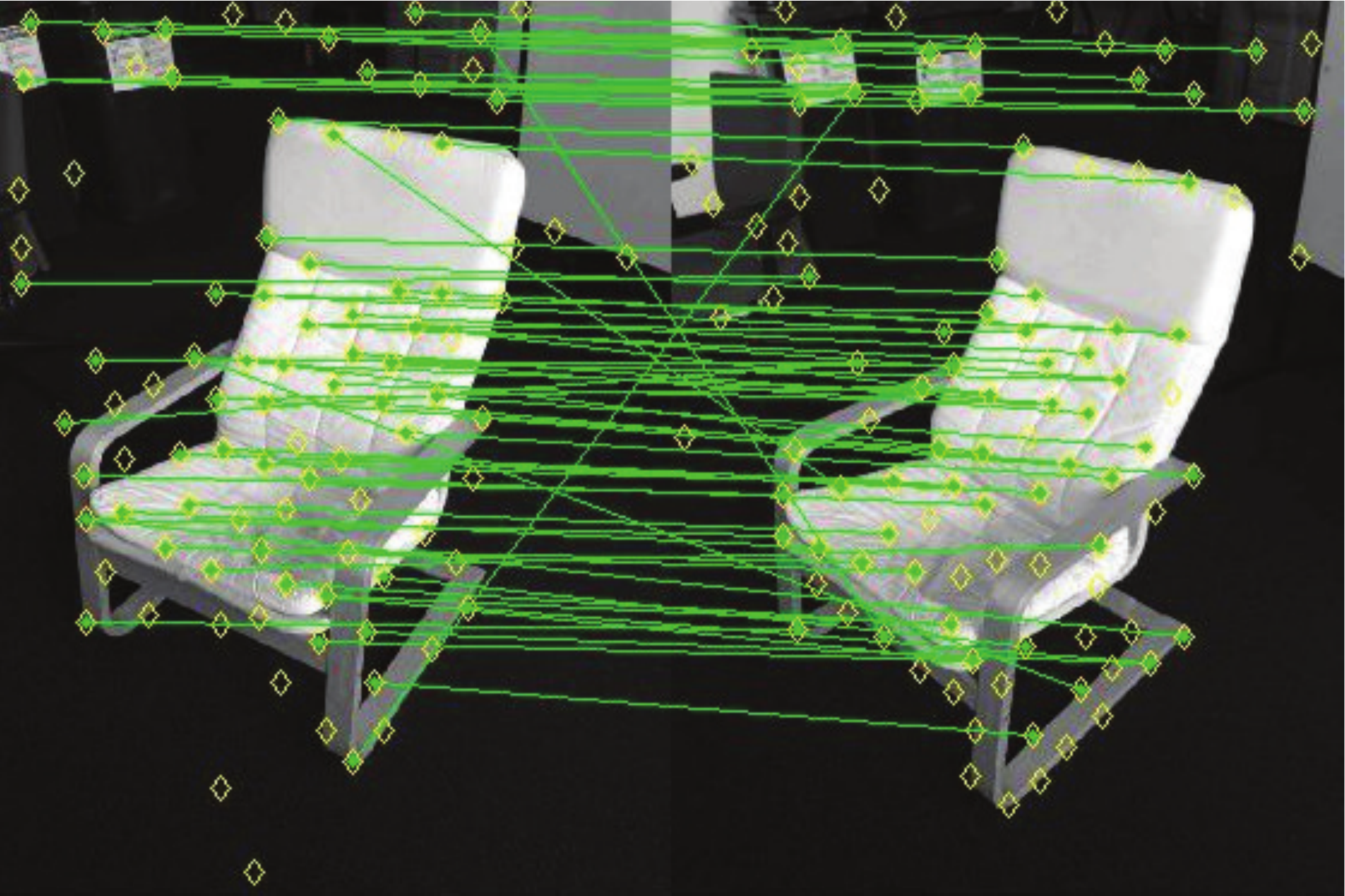}\tabularnewline
(a)  & (b) \tabularnewline
\end{tabular}\caption{A map between dense SIFT feature points (a) is converted into a map
between sampled feature points (b. }

\label{Fig:Downsampling} 
\end{figure*}

\item 
\begin{figure*}[htp]
\centering\includegraphics[width=0.6\textwidth]{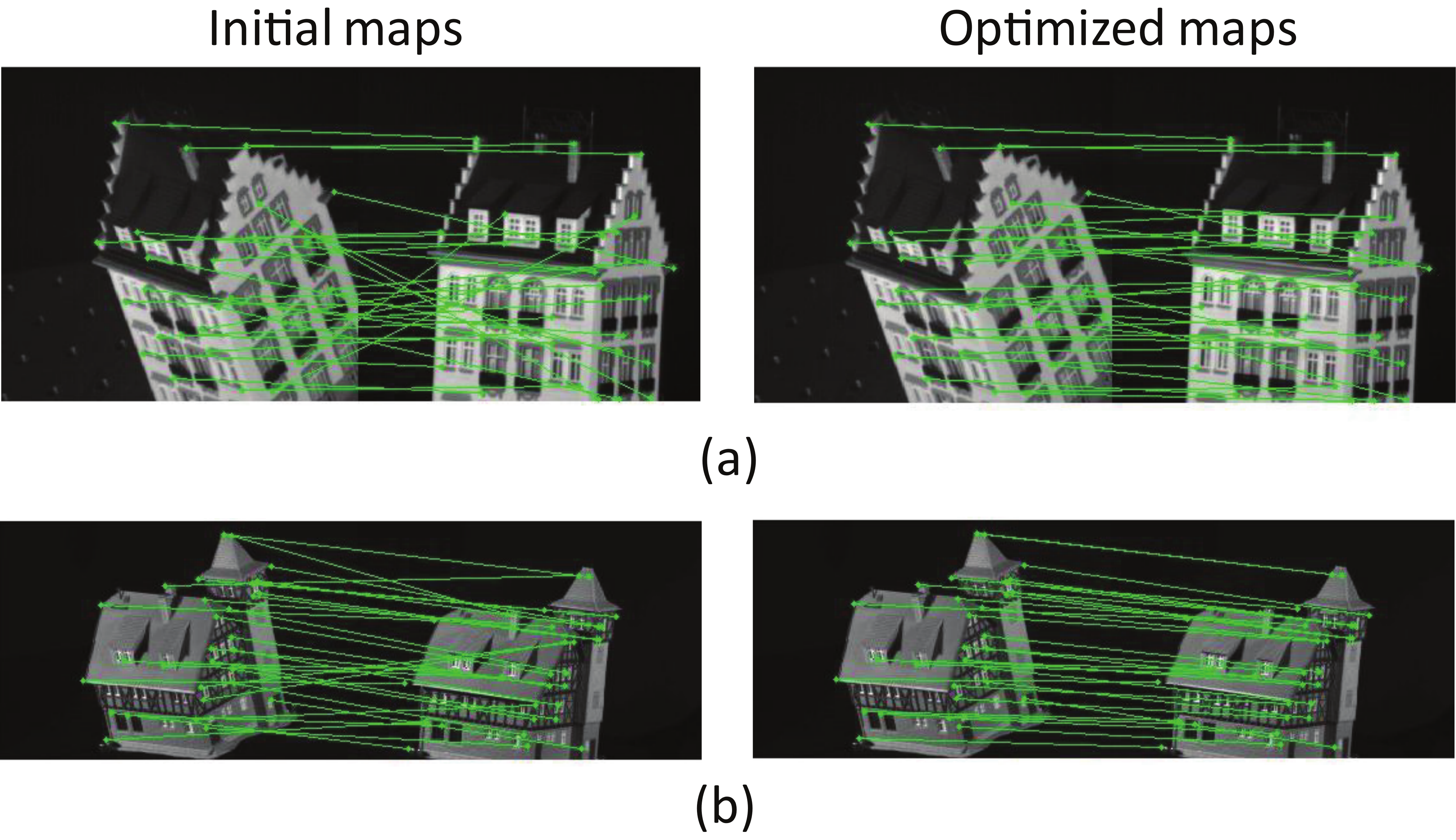}

\includegraphics[width=0.6\textwidth]{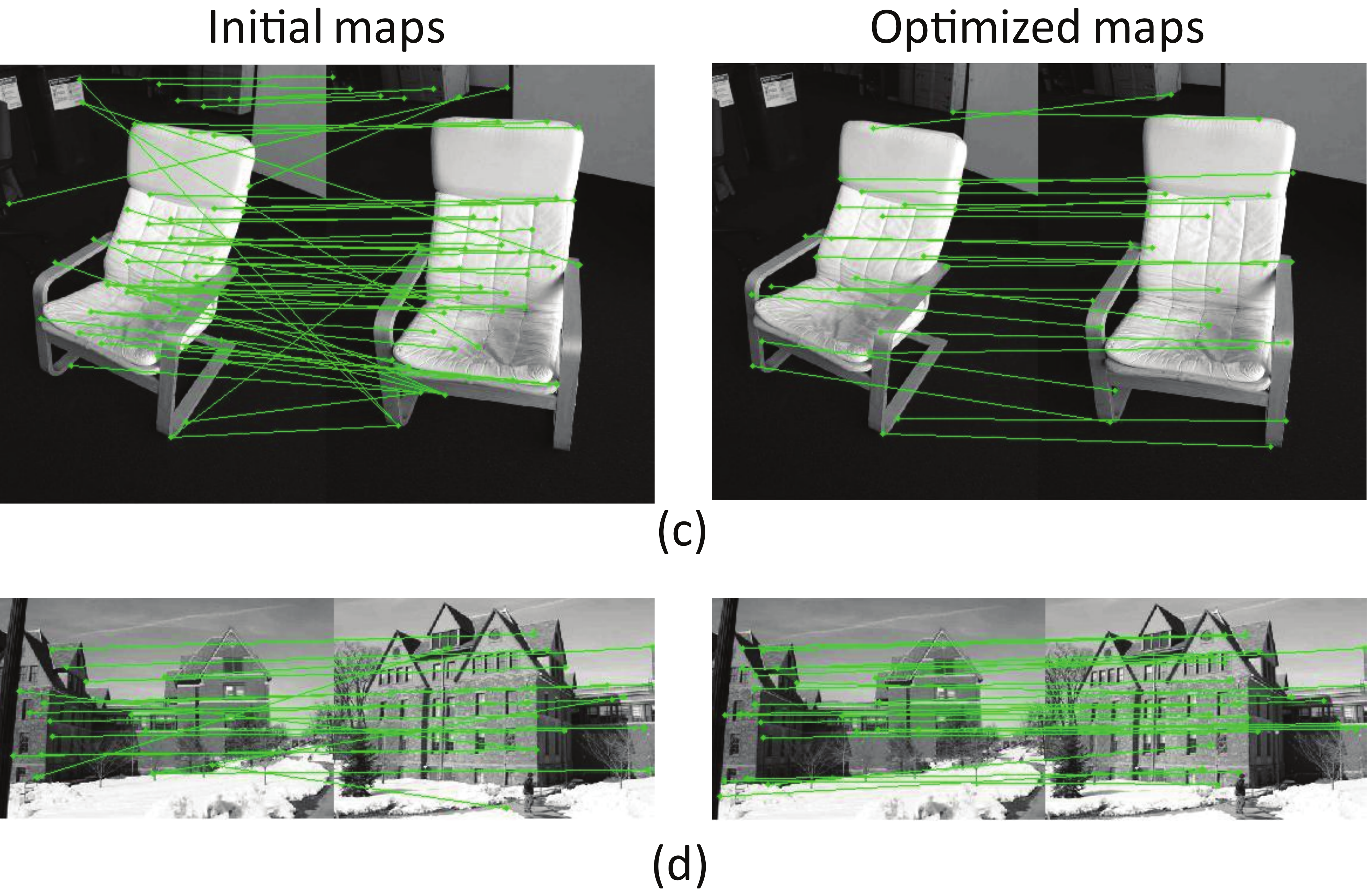}

\includegraphics[width=0.6\textwidth]{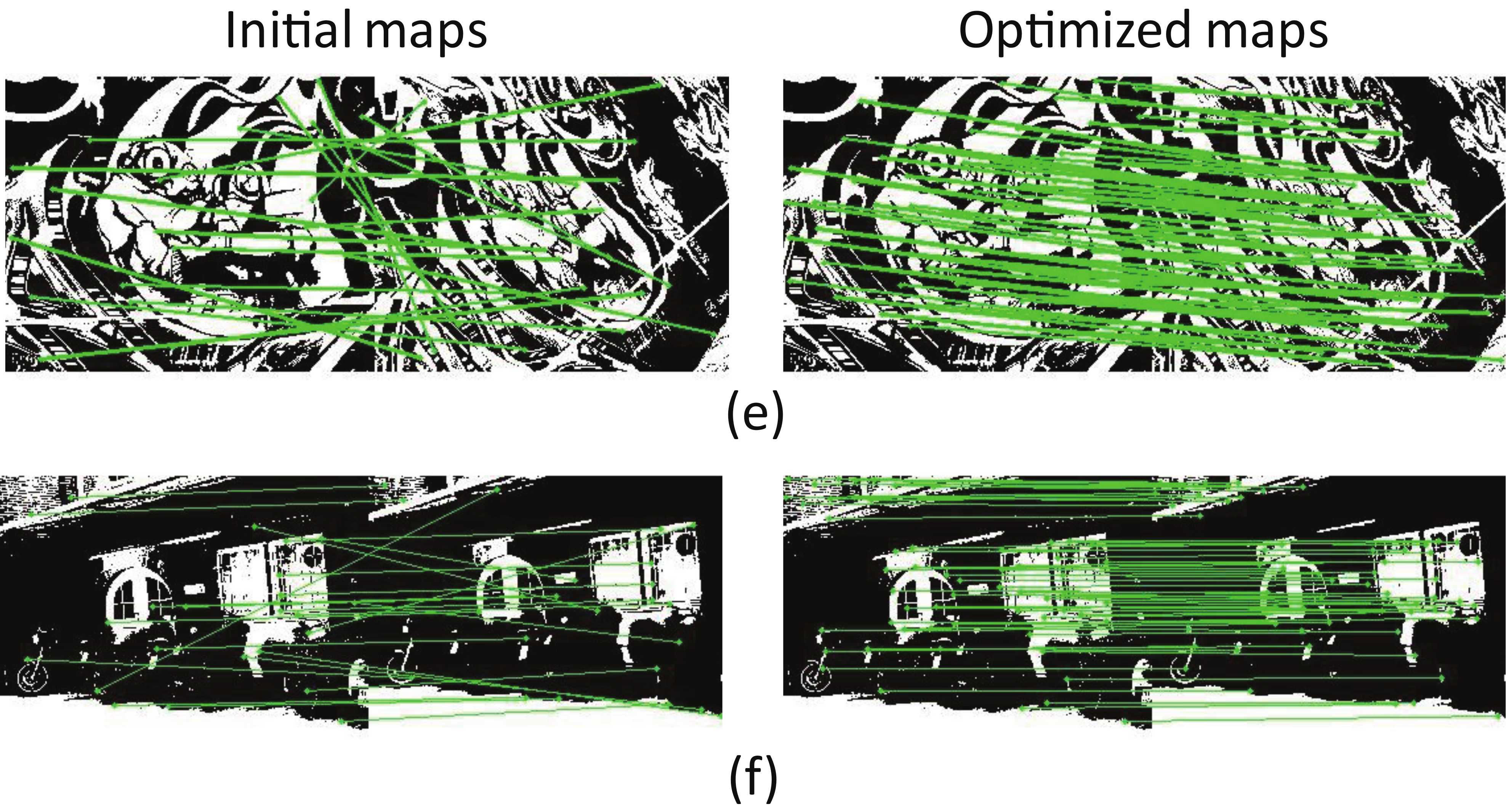}\caption{Comparisons between the input maps and the output of MatchLift on
six benchmark datasets: (a) CMU Hotel, (b) CMU House, (c) Chair, (d)
Building, (e) Graf, and (f) Bikes. The optimized maps not correct
incorrect correspondences as well as fill in missing correspondences
(generated by paths through intermediate shapes). One representative
pair from each dataset is shown here.}

\label{Fig:Comparison-CMU} 
\end{figure*}

To handle raw images in Chair, Building, Graf and Bikes, we apply
a different strategy to build feature points and initial maps. We
first detect dense SIFT feature points~\cite{Lowe:2004:DIF} on each
image. We then apply RANSAC~ \cite{Fischler:1981:RSC} to obtain
correspondences between each pair of images. As SIFT feature points
are over-complete, many of them do not appear in the resulting feature
correspondences between pairs of views. Thus, we remove all feature
points that have less than $2$ appearances in all pair-wise maps.
We further apply farthest point sampling on the feature points until
the sampling density is above $0.05w$, where $w$ is the width of
the input images. The remaining feature points turn out to be much
more distinct and thus are suitable for joint matching (See Figure~\ref{Fig:Downsampling}).
For the experiments we have tested, we obtain about $60-100$ features
points per image.

\item \textbf{Evaluation protocol.} On CMU-House and CMU-Hotel, we count
the percentage of correct feature correspondences produced by each
algorithm. On Chair, Building, Graf and Bikes, we apply the metric
described in~\cite{HaCohen:2011:NDC}, which evaluates the deviations
of manual feature correspondences. As the feature points computed
on each image do not necessarily align with the manual features, we
apply~\cite{conf/cvpr/AhmedTRTS08} to interpolate feature level
correspondences into pixel-wise correspondences for evaluation.
\begin{figure*}[htp]
\centering %
\begin{tabular}{cc}
\includegraphics[width=0.3\columnwidth]{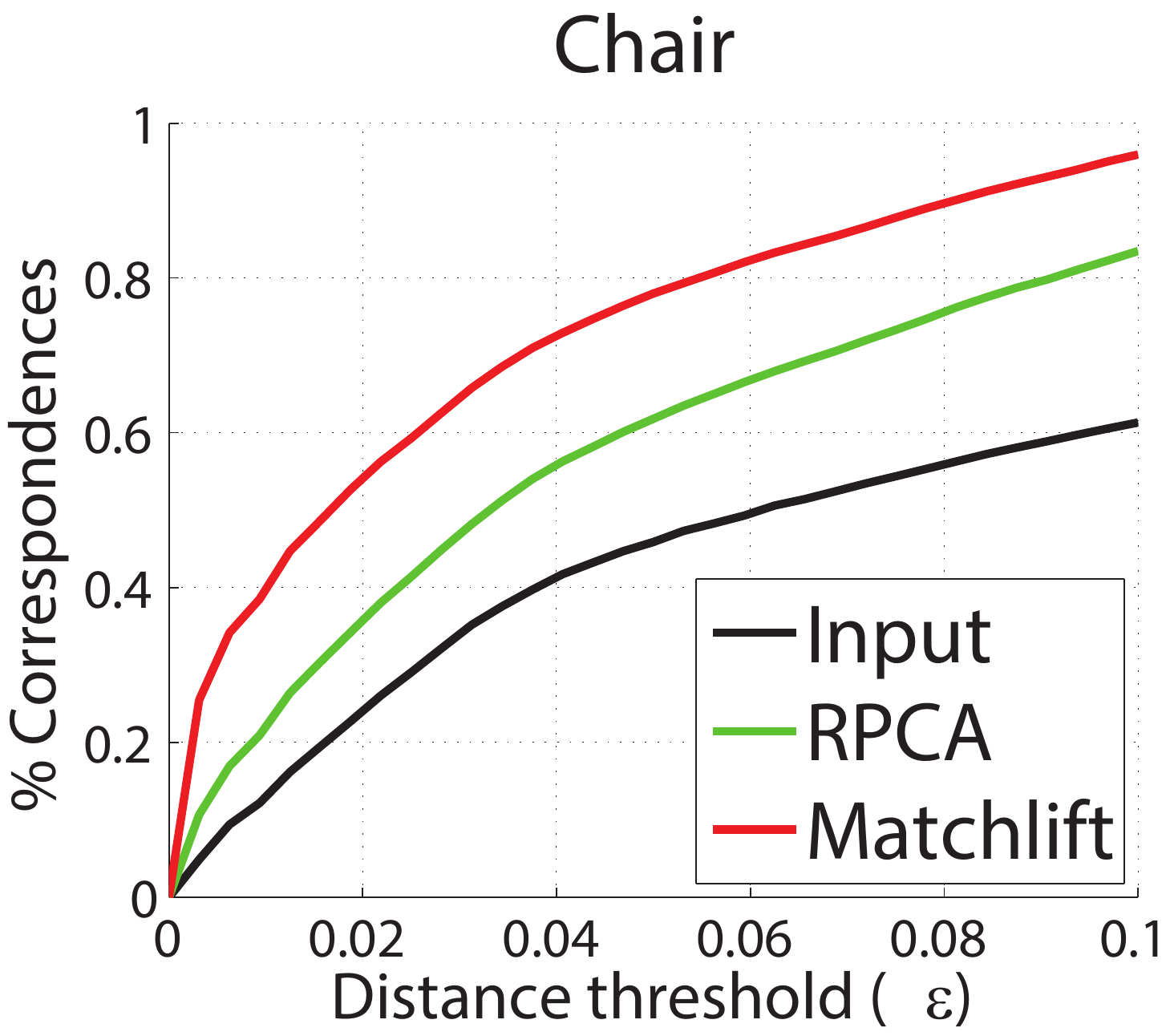} & \includegraphics[width=0.3\columnwidth]{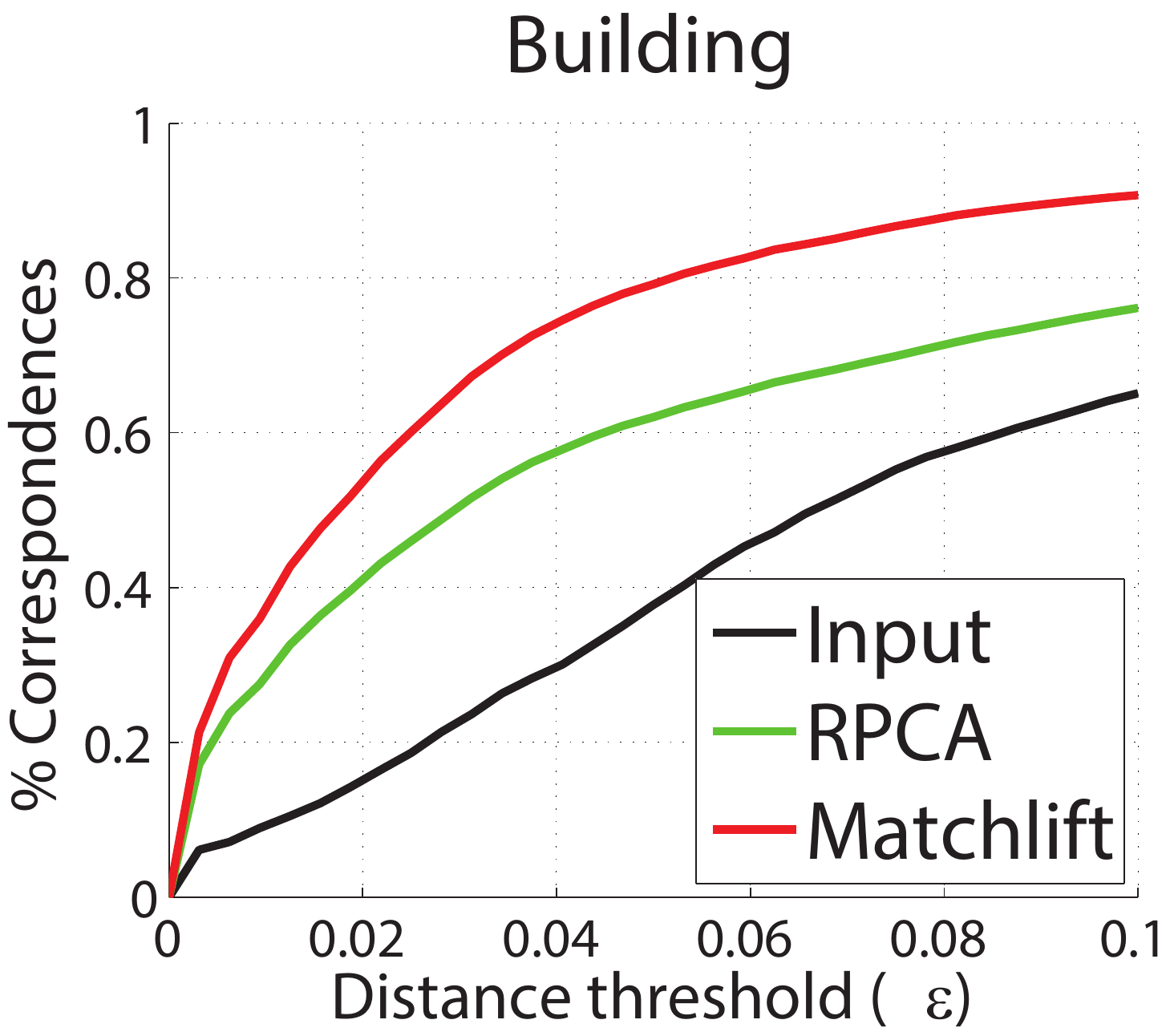}\tabularnewline
(a)  & (b) \tabularnewline
\includegraphics[width=0.3\columnwidth]{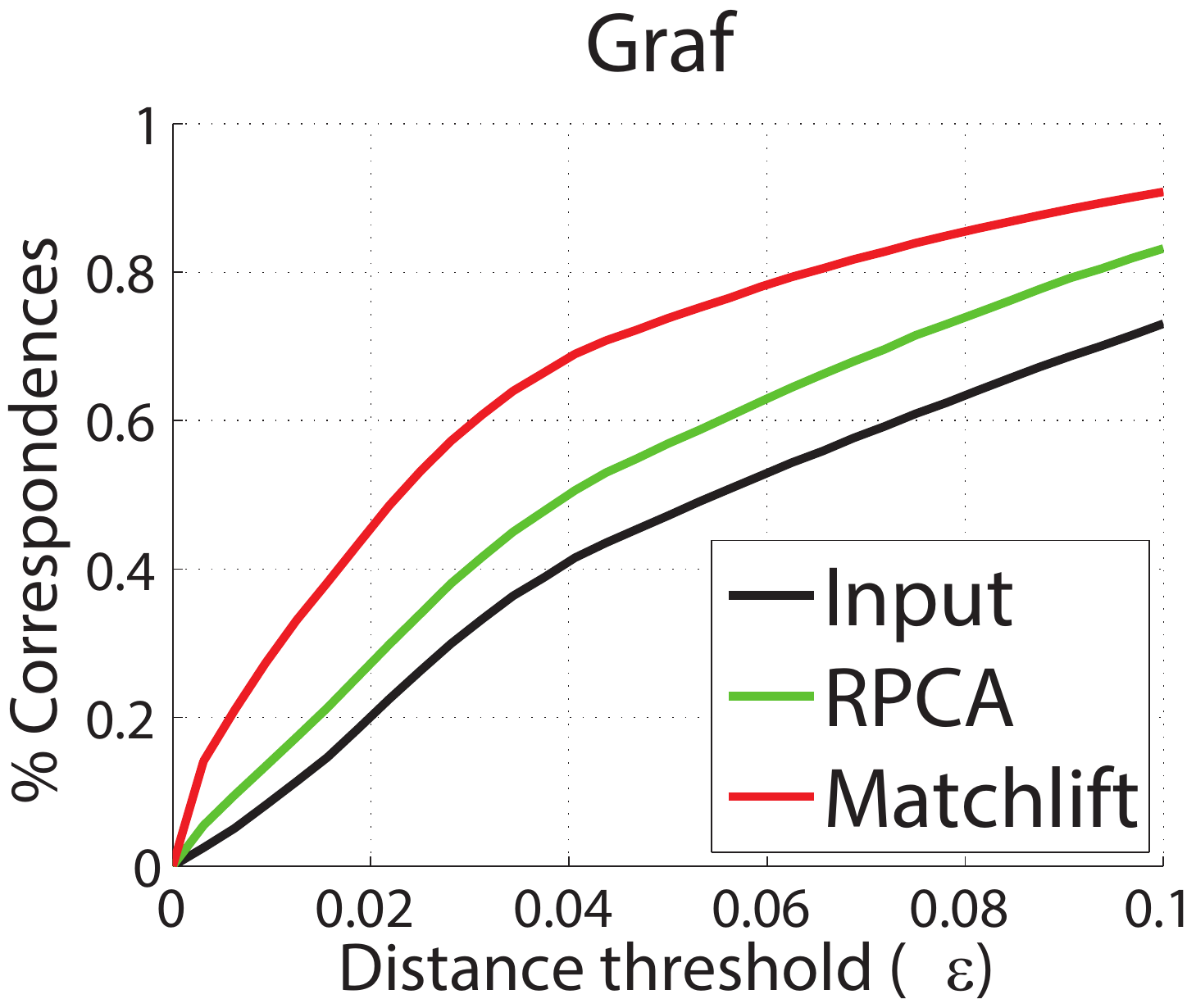} & \includegraphics[width=0.3\columnwidth]{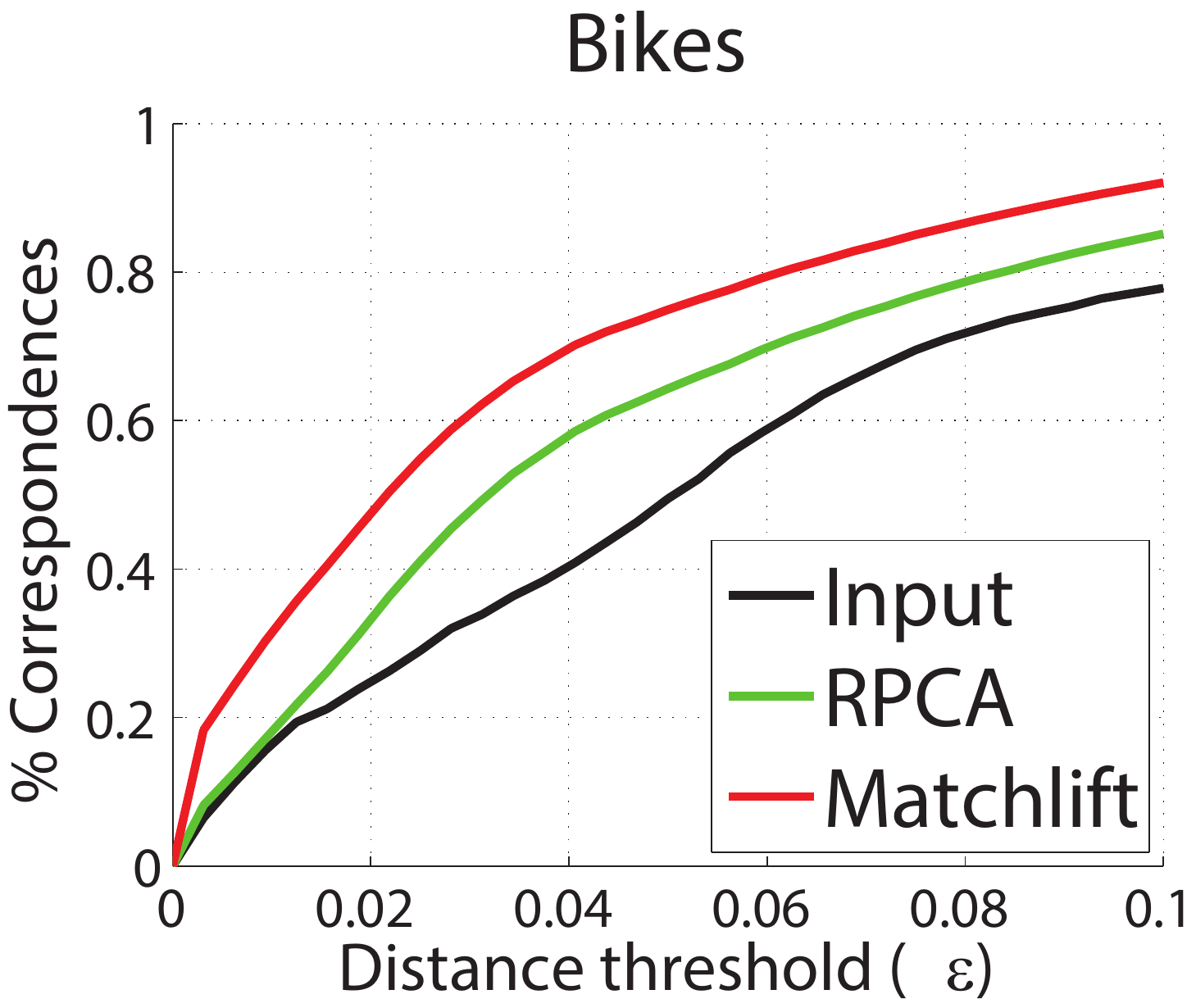}\tabularnewline
(c) & (d)\tabularnewline
\end{tabular}\caption{Percentages of ground-truth correspondences, whose distances to a
map collection are below a varying threshold $\epsilon$.}

\label{Fig:Benchmark1} 
\end{figure*}
\begin{table}
\centering%
\begin{tabular}{c|ccccc}
 & Input & MatchLift & RPCA & LearnI & LearnII\tabularnewline
\hline 
House & 68.2\% & 100\% & 92.2\% & 99.8\% & 96\%\tabularnewline
Hotel & 64.1\% & 100\% & 90.1\% & 99.8\% & 90\%\tabularnewline
\end{tabular}

\caption{Matching performance on the hotel and house datasets. We compare the
proposed MatchLift algorithm with Robust PCA (RPCA) and two learning
based graph matching methods: LearnI~\cite{leordeanu2012unsupervised}
and LearnII~\cite{caetano2009learning}.}

\label{Table:CMU} 
\end{table}

\item \textbf{Results.} Table~\ref{Table:CMU} shows the results of various
algorithms on CMU-House and CMU-Hotel. We can see that even with moderate
initial maps, MatchLift recovers all ground-truth correspondences.
In contrast, the method of~\cite{jalali2011clustering} can only
recover $92.2\%$ and $90.1\%$ ground-truth correspondences on CMU-House
and CMU-Hotel, respectively. Note that, MatchLift also outperforms
state-of-the-art learning based graph matching algorithms~\cite{caetano2009learning,leordeanu2012unsupervised}.
This shows the the advantage of joint object matching.

Figure~\ref{Fig:Comparison-CMU} and Figure~\ref{Fig:Benchmark1}
illustrate the results of MatchLift on Chair, Building, Graf and Bikes.
As these images contain noisy background information, the quality
of the input maps is lower than those on House and Hotel. Encouragingly,
MatchLift still recovers almost all manual correspondences. Moreover,
MatchLift significantly outperforms ~\cite{jalali2011clustering},
as the fault-tolerance rate of \cite{jalali2011clustering} is limited
by a small constant barrier.

Another interesting observation is that the improvements on Graf and
Bikes (each has 6 images) are lower than those on Chair and Building
(each has 16 images). This is consistent with the common knowledge
of data-driven effect, where large object collections possess stronger
self-correction power than small object collections. 

\end{itemize}

\section{Conclusions\label{sec:Conclusions}}

This paper delivers some encouraging news: given a few noisy object
matches computed in isolation, a collection of partially similar objects
can be accurately matched via semidefinite relaxation -- an approach
which provably works under dense errors. The proposed algorithm is
essentially parameter-free, and can be solved by ADMM achieving remarkable
efficiency and accuracy, with the assistance of a greedy rounding
strategy.

The proposed algorithm achieves near-optimal error-correction ability,
as it is guaranteed to work even when a dominant fraction of inputs
are corrupted. This in turn underscore the importance of joint object
matching: however low the quality of input sources is, perfect matching
is achievable as long as we obtain sufficiently many instances. Also,
while partial matching may incur much more severe input errors than
those occurring in full-similarity matching, in many situations, the
recovery ability of our algorithm is nearly the same as that in the
full-similarity case (up to some constant factor). In a broader sense,
our findings suggest that a large class of combinatorial / integer
programming problems might be solved perfectly by semidefinite relaxation.

\appendix

\section{Alternating Direction Method of Multipliers (ADMM)\label{sec:ADMM-Appendix}}

This section presents the procedure for the ADMM algorithm. For notational
simplicity, we represent the convex program as follows: 
\begin{align*}
\underset{\boldsymbol{X}}{\text{minimize}} & \quad\langle\boldsymbol{W},\overline{\boldsymbol{X}}\rangle & \textup{dual variable}\\
\subjectto & \quad\set{A}(\overline{\boldsymbol{X}})=\boldsymbol{b}, & \quad\boldsymbol{y}_{\set{A}}\\
 & \quad\overline{\boldsymbol{X}}\geq{\bf {\bf 0}}, & \quad\boldsymbol{Z}\geq{\bf {\bf 0}}\\
 & \quad\overline{\boldsymbol{X}}\succeq{\bf {\bf 0}}, & \quad\boldsymbol{S}\succeq{\bf {\bf 0}}
\end{align*}
where we denote $\overline{\boldsymbol{X}}:=\left[\begin{array}{cc}
m & {\bf 1}^{\top}\\
{\bf 1} & \boldsymbol{X}
\end{array}\right]$. The matrices and operators are defined as follows

(i) $\bs{W}$ encapsulate all block coefficient matrices $\bs{W}_{ij}$
for all $(i,j)\in\set{G}$;

(ii) $\set{A}(\overline{\boldsymbol{X}})=\boldsymbol{b}$ represents
the constraint that $\bs{X}_{ii}=\bs{I}_{m_{i}}$ ($1\leq i\leq n$)
and the constraint $\overline{\boldsymbol{X}}=\left[\begin{array}{cc}
m & {\bf 1}^{\top}\\
{\bf 1} & \boldsymbol{X}
\end{array}\right]$;

(iii) The variables on the right hand, i.e., $\bs{y}_{\set{A}},\bs{Z}$
and $\bs{S}$, represent dual variables associated with respective
constraints.

The Lagrangian associated with the convex program can be given as
follows
\begin{align*}
\mathcal{L} & =\left\langle \boldsymbol{W},\overline{\boldsymbol{X}}\right\rangle +\left\langle \boldsymbol{y}_{\mathcal{A}},\mathcal{A}(\overline{\boldsymbol{X}})-\boldsymbol{b}\right\rangle -\langle\boldsymbol{Z},\overline{\boldsymbol{X}}\rangle-\left\langle \boldsymbol{S},\overline{\boldsymbol{X}}\right\rangle \\
 & =\left\langle \boldsymbol{W}+\mathcal{A}^{*}(\boldsymbol{y}_{\mathcal{A}})-\boldsymbol{Z}-\boldsymbol{S},\overline{\boldsymbol{X}}\right\rangle -\left\langle \boldsymbol{b},\boldsymbol{y}_{\mathcal{A}}\right\rangle .
\end{align*}
where $\mathcal{A}^{*}$ denotes the conjugate operator w.r.t. an
operator $\mathcal{A}$. The augmented Lagrangian for the convex program
can now be written as 
\begin{align*}
\mathcal{L}_{1/\mu}= & \left\langle \boldsymbol{b},\boldsymbol{y}_{\mathcal{A}}\right\rangle +\left\langle \boldsymbol{Z}+\boldsymbol{S}-\boldsymbol{W}-\mathcal{A}^{*}(\boldsymbol{y}_{\mathcal{A}}),\overline{\boldsymbol{X}}\right\rangle \\
 & \quad\quad+\frac{1}{2\mu}\|\boldsymbol{Z}+\boldsymbol{S}-\boldsymbol{W}-\mathcal{A}^{*}(\boldsymbol{y}_{\mathcal{A}})\|_{\mathrm{F}}^{2}.
\end{align*}
Here, the linear terms above represent the negative standard Lagrangian,
whereas the quadratic parts represent the augmenting terms. $\mu$
is the penalty parameter that balances the standard Lagrangian and
the augmenting terms. The ADMM then proceeds by alternately optimizing
each primal and dual variable with others fixed, which results in
closed-form solution for each subproblem. Denote by superscript $k$
the iteration number, then we can present the ADMM iterative update
procedures as follows 
\begin{align}
\bs{y}_{\mathcal{A}}^{(k+1)}= & \left(\set{A}\set{A}^{*}\right)^{-1}\left\{ \set{A}\left(-\bs{W}+\bs{S}^{(k)}+\mu\overline{\boldsymbol{X}}^{(k)}+\bs{Z}^{(k)}\right)-\mu\bs{b}\right\} ,\nonumber \\
\bs{Z}^{(k+1)}= & \left(\bs{W}+\set{A}^{*}\left(\bs{y}_{\mathcal{A}}^{(k+1)}\right)-\bs{S}^{(k)}-\mu\overline{\boldsymbol{X}}^{(k)}\right)_{+},\nonumber \\
\bs{S}^{(k+1)}= & \text{ }\mathcal{P}_{\text{psd}}\left(\bs{W}+\set{A}^{*}\left(\vec{y}_{\mathcal{A}}^{(k+1)}\right)-\bs{Z}^{(k+1)}-\mu\overline{\boldsymbol{X}}^{(k)}\right),\\
\overline{\boldsymbol{X}}^{\left(k+1\right)}= & \text{ }\overline{\boldsymbol{X}}^{k}+\frac{1}{\mu}\left(\boldsymbol{Z}^{(k+1)}+\boldsymbol{S}^{\left(k+1\right)}-\boldsymbol{W}-\set{A}^{*}\left(\vec{y}_{\mathcal{A}}^{(k+1)}\right)\right)\\
= & -\frac{1}{\mu}\mathcal{P}_{\text{nsd}}\left(\bs{W}+\set{A}^{*}\left(\vec{y}_{\mathcal{A}}^{(k+1)}\right)-\bs{Z}^{(k+1)}-\mu\overline{\boldsymbol{X}}^{(k)}\right).
\end{align}
Here, the operator $\mathcal{P}_{\text{psd}}$ (resp. $\mathcal{P}_{\mathrm{nsd}}$)
denotes the projection onto the positive (resp. negative) semidefinite
cone, and $\left(\cdot\right)_{+}$ operator projects all entries
of a vector / matrix to non-negative values. Within a reasonable amount
of time, ADMM typically returns moderately acceptable results.

\section{Proof of Theorem \ref{thm:SpectralMethod}\label{sec:Proof_thm:SpectralMethod-m}}

The key step to the proof of Theorem \ref{thm:SpectralMethod} is
to show that the set of outliers, even when they account for a dominant
portion of the input matrix, behave only as a small perturbation to
the spectrum of the non-corrupted components. Under the randomized
model described in Section \ref{sub:Randomized-Model}, it can be
easily seen that the trimming procedure is not invoked with high probability.
Consequently, Theorem \ref{thm:SpectralMethod} can be established
through the following lemma.

\begin{lem}\label{lemma:SpectralMethod-m}Given any set of $n$ permutation
matrices $\boldsymbol{P}_{i}\in\mathbb{R}^{m\times m}$ ($1\leq i\leq n$),
generate a random matrix $\boldsymbol{M}$ via the following procedure.
\begin{enumerate}
\item Generate a symmetric block matrix $\boldsymbol{A}=\left[\boldsymbol{A}_{ij}\right]_{1\leq i,j\leq n}$
such that 
\[
\boldsymbol{A}_{ii}=\boldsymbol{I},\quad1\leq i\leq n
\]
and for all $i<j$,
\begin{equation}
\boldsymbol{A}_{ij}=\begin{cases}
{\bf 0}, & \text{if }\mu_{ij}=0,\\
\boldsymbol{P}_{i}\boldsymbol{P}_{j}^{\top},\quad & \text{if }\nu_{ij}=1\text{ and }\mu_{ij}=1,\\
\boldsymbol{U}_{ij}, & \text{else},
\end{cases}
\end{equation}
where $\nu_{ij}\sim\text{Bernoulli}\left(p\right)$ and $\mu_{ij}\sim\text{Bernoulli}\left(\tau\right)$
are independent binary variables, and $\boldsymbol{U}_{ij}\in\mathbb{R}^{m\times m}$
are independent random permutation matrices obeying $\mathbb{E}\boldsymbol{U}_{ij}=\frac{1}{m}{\bf 1}_{m}\cdot{\bf 1}_{m}^{\top}$. 
\item $\boldsymbol{M}$ is a principal minor of $\boldsymbol{A}$ from rows
/ columns at indices from a set $I\subseteq\left\{ 1,2,\cdots,mn\right\} $,
where each $1\leq i\leq mn$ is contained in $I$ independently with
probability $q$.
\end{enumerate}
Then there exist absolute constants $c_{1},c_{2}>0$ such that if
$p\geq c_{1}\frac{\log^{2}\left(mn\right)}{q\sqrt{\tau n}}$, one
has
\begin{equation}
\begin{cases}
\lambda_{i}\left(\boldsymbol{M}\right) & \geq\left(1-\frac{1}{\log\left(mn\right)}\right)\tau pqn,\quad\quad\quad\quad\text{if }1\leq i\leq m\\
\lambda_{i}\left(\boldsymbol{M}\right) & \leq c_{2}\sqrt{\tau n}\log\left(mn\right)<\frac{\tau pqn}{\log\left(mn\right)},\quad\text{if }i>m
\end{cases}\label{eq:EigenvalueGapSpectralMethod}
\end{equation}
with probability exceeding $1-\frac{1}{m^{5}n^{5}}$. Here, $\lambda_{i}(\boldsymbol{M})$
represents the $i$th largest eigenvalue of $\boldsymbol{M}$.\end{lem}

\begin{proof}[Proof of Lemma \ref{lemma:SpectralMethod-m}]Without
loss of generality, we assume that $\boldsymbol{P}_{i}=\boldsymbol{I}_{m}$
for all $1\leq i\leq n$, since rearranging rows / columns of $\boldsymbol{A}$
does not change its eigenvalues. For convenience of presentation,
we write $\boldsymbol{A}=\boldsymbol{Y}+\boldsymbol{Z}$ such that
\[
\boldsymbol{Y}_{ii}=\tau\left(\frac{\left(1-p\right)}{m}{\bf 1}_{m}\cdot{\bf 1}_{m}^{\top}+p\boldsymbol{I}_{m}\right),\quad1\leq i\leq n
\]
and for all $1\leq i\leq j\leq n$: 
\begin{equation}
\boldsymbol{Y}_{ij}=\begin{cases}
{\bf 0}, & \text{if }\mu_{ij}=0,\\
\boldsymbol{I}_{m},\quad & \text{if }\nu_{ij}=1\text{ and }\mu_{ij}=1,\\
\boldsymbol{U}_{ij}, & \text{else}.
\end{cases}
\end{equation}
This means that
\begin{equation}
\boldsymbol{Z}_{ij}=\begin{cases}
\boldsymbol{I}_{m}-\boldsymbol{Y}_{ii},\quad & \text{if }i=j,\\
{\bf 0}, & \text{else}.
\end{cases}
\end{equation}
Apparently, $\boldsymbol{Z}$ is a block diagonal matrix satisfying
\begin{equation}
\left\Vert \boldsymbol{Z}\right\Vert \leq2,
\end{equation}
which is only a mild perturbation of $\boldsymbol{Y}$. This way we
have reduced to the case where all blocks (including diagonal blocks)
are i.i.d., which is slightly more convenient to analyze. 

Decompose $\boldsymbol{Y}$ into 2 components $\boldsymbol{Y}=\boldsymbol{Y}^{\mathrm{mean}}+\boldsymbol{Y}^{\mathrm{var}}$
such that 
\begin{equation}
\forall1\leq i\leq j\leq n:\quad\boldsymbol{Y}_{ij}^{\mathrm{mean}}=\tau\left(\frac{\left(1-p\right)}{m}{\bf 1}_{m}\cdot{\bf 1}_{m}^{\top}+p\boldsymbol{I}_{m}\right),
\end{equation}
\begin{equation}
\forall1\leq i\leq n:\quad\boldsymbol{Y}_{ii}^{\mathrm{var}}={\bf 0},
\end{equation}
and
\begin{equation}
\forall1\leq i<j\leq n:\quad\boldsymbol{Y}_{ij}^{\mathrm{var}}=\begin{cases}
-\tau\left(\frac{\left(1-p\right)}{m}{\bf 1}_{m}\cdot{\bf 1}_{m}^{\top}+p\boldsymbol{I}_{m}\right),\quad & \text{if }\mu_{ij}=0,\\
\left(1-\tau p\right)\boldsymbol{I}_{m}-\frac{\left(1-p\right)}{m}{\bf 1}_{m}\cdot{\bf 1}_{m}^{\top},\quad & \text{if }\nu_{ij}=1\text{ and }\mu_{ij}=1,\\
\boldsymbol{U}_{ij}-\tau\left(\frac{\left(1-p\right)}{m}{\bf 1}_{m}\cdot{\bf 1}_{m}^{\top}+p\boldsymbol{I}_{m}\right), & \text{else}.
\end{cases}
\end{equation}
In other words, $\boldsymbol{Y}^{\mathrm{mean}}$ represents the mean
component of $\boldsymbol{Y}$, while $\boldsymbol{Y}^{\mathrm{var}}$
comprises all variations around the mean component. It is straightforward
to check that 
\[
\boldsymbol{Y}^{\mathrm{mean}}\succeq{\bf 0},\quad\mathrm{rank}\left(\boldsymbol{Y}^{\mathrm{mean}}\right)\leq m+1.
\]
If we denote by $\boldsymbol{Y}_{I}^{\mathrm{mean}}$ the principal
minor coming from the rows and columns of $\boldsymbol{Y}$ at indices
from $I$, then from Weyl's inequality one can easily see that
\begin{equation}
\lambda_{i}\left(\boldsymbol{M}\right)\geq\lambda_{i}\left(\boldsymbol{Y}_{I}^{\mathrm{mean}}\right)-\left\Vert \boldsymbol{Y}^{\mathrm{var}}\right\Vert -\left\Vert \boldsymbol{Z}\right\Vert \geq\lambda_{i}\left(\boldsymbol{Y}_{I}^{\mathrm{mean}}\right)-\left\Vert \boldsymbol{Y}^{\mathrm{var}}\right\Vert -2,\quad1\leq i\leq m\label{eq:Evalue-M-LB}
\end{equation}
and
\begin{equation}
\lambda_{i}\left(\boldsymbol{M}\right)\leq\lambda_{i}\left(\boldsymbol{Y}_{I}^{\mathrm{mean}}\right)+\left\Vert \boldsymbol{Y}^{\mathrm{var}}\right\Vert +\left\Vert \boldsymbol{Z}\right\Vert \leq\lambda_{i}\left(\boldsymbol{Y}_{I}^{\mathrm{mean}}\right)+\left\Vert \boldsymbol{Y}^{\mathrm{var}}\right\Vert +2,\quad\quad i>m.\label{eq:Evalue-M-UB}
\end{equation}
In light of this, it suffices to evaluate $\left\Vert \boldsymbol{Y}^{\mathrm{var}}\right\Vert $
as well as the eigenvalues of $\boldsymbol{Y}_{I}^{\mathrm{mean}}$.

We are now in position to quantify the eigenvalues of $\boldsymbol{Y}_{I}^{\mathrm{mean}}$.
Without affecting its eigenvalue distribution, one can rearrange the
rows / columns of $\boldsymbol{Y}_{I}^{\mathrm{mean}}$ so that
\begin{equation}
\boldsymbol{Y}_{I}^{\mathrm{mean}}\overset{\text{\footnotesize(permutation)}}{=}\tau p\left[\begin{array}{ccc}
{\bf 1}_{n_{1}}\cdot{\bf 1}_{n_{1}}^{\top}\\
 & \ddots\\
 &  & {\bf 1}_{n_{m}}\cdot{\bf 1}_{n_{m}}^{\top}
\end{array}\right]+\frac{\tau\left(1-p\right)}{m}{\bf 1}_{N}\cdot{\bf 1}_{N}^{\top}.\label{eq:Ymean_rearrange}
\end{equation}
Here, $n_{i}$ ($1\leq i\leq m$) denotes the cardinality of a set
$I_{i}$ generated by independently sampling $n$ elements each with
probability $q$, and we set $N:=n_{1}+\cdots+n_{m}$ for simplicity.
From Bernstein inequality, there exist universal constants $c_{5},c_{6}>0$
such that if $q>\frac{c_{5}\log\left(mn\right)}{n}$, then
\begin{equation}
\left|n_{i}-nq\right|\leq c_{6}\sqrt{nq\log\left(mn\right)},\quad1\leq i\leq m\label{eq:Concentration-ni}
\end{equation}
holds with probability exceeding $1-\left(mn\right)^{-10}$. 

Since $\boldsymbol{Y}_{I}^{\mathrm{mean}}$ is positive semidefinite,
from (\ref{eq:Ymean_rearrange}) one can easily check that all non-zero
eigenvalues of $\boldsymbol{Y}_{I}^{\mathrm{mean}}$ are also eigenvalues
of the following $(m+1)\times(m+1)$ matrix
\begin{align}
\footnotesize\overline{\boldsymbol{Y}}_{I}^{\mathrm{mean}}: & =\footnotesize\tau\left[\begin{array}{c}
\begin{array}{cccc}
\sqrt{p}{\bf 1}_{n_{1}}^{\top}\\
 & \sqrt{p}{\bf 1}_{n_{2}}^{\top}\\
 &  & \ddots\\
 &  &  & \sqrt{p}{\bf 1}_{n_{m}}^{\top}
\end{array}\\
\sqrt{\frac{1-p}{m}}{\bf 1}_{N}^{\top}
\end{array}\right]\left[\begin{array}{cc}
\begin{array}{c}
\begin{array}{cccc}
\sqrt{p}{\bf 1}_{n_{1}}\\
 & \sqrt{p}{\bf 1}_{n_{2}}\\
 &  & \ddots\\
 &  &  & \sqrt{p}{\bf 1}_{n_{m}}
\end{array}\end{array} & \sqrt{\frac{1-p}{m}}{\bf 1}_{N}\end{array}\right]\nonumber \\
 & =\footnotesize\tau\left[\begin{array}{ccccc}
pn_{1} &  &  &  & \sqrt{\frac{p\left(1-p\right)}{m}}n_{1}\\
 & pn_{2} &  &  & \sqrt{\frac{p\left(1-p\right)}{m}}n_{2}\\
 &  & \ddots &  & \vdots\\
 &  &  & pn_{m} & \sqrt{\frac{p\left(1-p\right)}{m}}n_{m}\\
\sqrt{\frac{p\left(1-p\right)}{m}}n_{1} & \sqrt{\frac{p\left(1-p\right)}{m}}n_{2} & \cdots & \sqrt{\frac{p\left(1-p\right)}{m}}n_{m} & \frac{1-p}{m}N
\end{array}\right]\\
 & =\footnotesize\underset{\overline{\boldsymbol{Y}}_{I,0}}{\underbrace{\tau qn\left[\begin{array}{cccc}
p &  &  & \sqrt{\frac{p\left(1-p\right)}{m}}\\
 & \ddots &  & \vdots\\
 &  & p & \sqrt{\frac{p\left(1-p\right)}{m}}\\
\sqrt{\frac{p\left(1-p\right)}{m}} & \cdots & \sqrt{\frac{p\left(1-p\right)}{m}} & 1-p
\end{array}\right]}}+\tau\underset{\overline{\boldsymbol{Y}}_{I,\Delta}}{\underbrace{\left[\begin{array}{cccc}
p\Delta_{1} &  &  & \sqrt{\frac{p\left(1-p\right)}{m}}\Delta_{1}\\
 & \ddots &  & \vdots\\
 &  & p\Delta_{m} & \sqrt{\frac{p\left(1-p\right)}{m}}\Delta_{m}\\
\sqrt{\frac{p\left(1-p\right)}{m}}\Delta_{1} & \cdots & \sqrt{\frac{p\left(1-p\right)}{m}}\Delta_{m} & \frac{1-p}{m}\Delta_{N}
\end{array}\right]}},\label{eq:Yi_mean_decomposition}
\end{align}
where 
\[
\Delta_{i}=n_{i}-nq\quad1\leq i\leq m,
\]
and 
\[
\Delta_{N}=N-qnm,
\]
which satisfies $\left|\Delta_{N}\right|\leq m\max_{1\leq i\leq m}\left|\Delta_{i}\right|$. 

By Schur complement condition for positive definite matrices \cite{boyd2004convex},
if $\left[\begin{array}{cc}
\boldsymbol{C} & \boldsymbol{B}\\
\boldsymbol{B}^{\top} & \boldsymbol{D}
\end{array}\right]\succ{\bf 0}$, then $\boldsymbol{C}\succ{\bf 0}$ and $\boldsymbol{D}-\boldsymbol{B}^{\top}\boldsymbol{C}^{-1}\boldsymbol{B}\succ{\bf 0}$.
Applying this condition to $\overline{\boldsymbol{Y}}_{I,0}$ suggests
that $\overline{\boldsymbol{Y}}_{I,0}\succ{\bf 0}$ can only hold
when
\[
\left(1-p\right)-\frac{p\left(1-p\right)}{m}\frac{1}{p}{\bf 1}_{m}^{\top}\cdot{\bf 1}_{m}>0,
\]
which however cannot be satisfied since $\left(1-p\right)-\frac{p\left(1-p\right)}{m}\frac{1}{p}{\bf 1}_{m}^{\top}\cdot{\bf 1}_{m}=0$.
Thus, $\overline{\boldsymbol{Y}}_{I,0}$ is rank deficient. 

In fact, all non-zero eigenvalues of $\overline{\boldsymbol{Y}}_{I,0}$
can be quantified as well. Specifically, for any vector
\[
\boldsymbol{z}_{i}:=\boldsymbol{e}_{i}-\frac{1}{m}\left[\begin{array}{c}
{\bf 1}_{m}\\
0
\end{array}\right],\quad1\leq i\leq m-1,
\]
one can compute
\begin{equation}
\overline{\boldsymbol{Y}}_{I,0}\cdot\boldsymbol{z}_{i}=\left(\tau qpn\right)\boldsymbol{z}_{i},\quad1\leq i\leq m-1.
\end{equation}
That said, $\tau qpn$ is an eigenvalue of $\overline{\boldsymbol{Y}}_{I,0}$
with multiplicity $m-1$. On the other hand, we have 
\begin{equation}
\begin{cases}
\overline{\boldsymbol{Y}}_{I,0}\cdot\left[\begin{array}{c}
{\bf 1}_{m}\\
\sqrt{\frac{\left(1-p\right)m}{p}}
\end{array}\right] & =\tau qn\left[\begin{array}{c}
{\bf 1}_{m}\\
\sqrt{\frac{\left(1-p\right)m}{p}}
\end{array}\right],\\
\boldsymbol{z}_{i}^{\top}\cdot\left[\begin{array}{c}
{\bf 1}_{m}\\
\sqrt{\frac{\left(1-p\right)m}{p}}
\end{array}\right] & =0,
\end{cases}
\end{equation}
indicating that $\tau qn$ is another eigenvalue of $\overline{\boldsymbol{Y}}_{I,0}$.
Putting these together yields
\begin{equation}
\lambda_{i}\left(\overline{\boldsymbol{Y}}_{I,0}\right)=\begin{cases}
\tau qn, & i=1\\
\tau pqn,\quad & 2\leq i\leq m,\\
0, & i>m.
\end{cases}\label{eq:evalue_Yi0}
\end{equation}

Furthermore, the residual component $\overline{\boldsymbol{Y}}_{I,\Delta}$
can be bounded as follows 
\begin{align*}
\left\Vert \overline{\boldsymbol{Y}}_{I,\Delta}\right\Vert  & \leq\footnotesize\tau\left\Vert \left[\begin{array}{cccc}
p\Delta_{1}\\
 & \ddots\\
 &  & p\Delta_{m}\\
 &  &  & \frac{1-p}{m}\Delta_{N}
\end{array}\right]\right\Vert +\tau\left\Vert \left[\begin{array}{cccc}
0 &  &  & \sqrt{\frac{p\left(1-p\right)}{m}}\Delta_{1}\\
 & \ddots &  & \vdots\\
 &  & 0 & \sqrt{\frac{p\left(1-p\right)}{m}}\Delta_{m}\\
\sqrt{\frac{p\left(1-p\right)}{m}}\Delta_{1} & \cdots & \sqrt{\frac{p\left(1-p\right)}{m}}\Delta_{m} & 0
\end{array}\right]\right\Vert _{\mathrm{F}}\\
 & \leq\tau\max\left\{ p\max_{1\leq i\leq m}\left|\Delta_{i}\right|,\frac{1-p}{m}\left|\Delta_{N}\right|\right\} +\tau\sqrt{2p\left(1-p\right)}\max_{1\leq i\leq m}\left|\Delta_{i}\right|\\
 & \leq2\tau\max_{1\leq i\leq m}\left|\Delta_{i}\right|\leq2c_{6}\tau\sqrt{nq\log\left(mn\right)},
\end{align*}
where the last inequality follows from (\ref{eq:Concentration-ni}).
This taken collectively with (\ref{eq:Yi_mean_decomposition}) and
(\ref{eq:evalue_Yi0}) yields that: when $p>\frac{2c_{6}\log^{2}\left(mn\right)}{\sqrt{nq}}$
or, equivalently, when $2c_{6}\sqrt{nq\log\left(mn\right)}<\frac{1}{\log^{1.5}\left(mn\right)}npq$,
one has
\begin{equation}
\begin{cases}
\lambda_{i}\left(\boldsymbol{Y}_{I}^{\mathrm{mean}}\right) & \geq\left(1-\frac{1}{\log^{\frac{3}{2}}\left(mn\right)}\right)\tau pqn,\quad\quad\quad\quad1\leq i\leq m,\\
\lambda_{i}\left(\boldsymbol{Y}_{I}^{\mathrm{mean}}\right) & \leq2c_{6}\tau\sqrt{nq\log\left(mn\right)}\leq\frac{1}{\log^{\frac{3}{2}}\left(mn\right)}\tau pqn,\quad\quad i>m.
\end{cases}\label{eq:EvalueYiMeanBound}
\end{equation}

Furthermore, observe that $\mathbb{E}\boldsymbol{Y}_{ij}^{\mathrm{var}}=0$,
$\mathbb{E}\left\Vert \frac{1}{2\sqrt{\tau}}\boldsymbol{Y}_{ij}^{\mathrm{var}}\right\Vert ^{2}\leq1$,
and $\frac{1}{2\sqrt{\tau}}\left\Vert \boldsymbol{Y}_{ij}^{\mathrm{var}}\right\Vert \leq\frac{1}{\sqrt{\tau}}$.
When $\tau>\frac{1}{n}$, Lemma \ref{lemma:MomentMethod} yields that
\begin{equation}
\left\Vert \boldsymbol{Y}^{\mathrm{var}}\right\Vert \leq2c_{0}\sqrt{\tau n}\log\left(mn\right)\label{eq:Yvar_UB}
\end{equation}
with probability at least $1-(mn)^{-5}$. Hence, $\left\Vert \boldsymbol{Y}^{\mathrm{var}}\right\Vert =o\left(\tau pqn\right)$
if $p>\frac{c_{10}\log^{2}n}{q\sqrt{\tau n}}$ for some constant $c_{10}>0$.

Finally, the claim follows by substituting (\ref{eq:EvalueYiMeanBound})
and (\ref{eq:Yvar_UB}) into (\ref{eq:Evalue-M-LB}) and (\ref{eq:Evalue-M-UB}).\end{proof}

\section{Proof of Theorem \ref{thm:RandomGraph-1} \label{sec:Proof-of-Theorem-RandomGraph}}

To prove Theorem \ref{thm:RandomGraph-1}, we first analyze the Karush\textendash{}Kuhn\textendash{}Tucker
(KKT) condition for exact recovery, which provides a sufficient and
almost necessary condition for uniqueness and optimality. Valid dual
certificates are then constructed to guarantee exact recovery.

\subsection{Preliminaries and Notations\label{sec:Prelim}}

Without loss of generality, we can treat $\boldsymbol{X}^{\mathrm{gt}}$
as a sub-matrix of an augmented square matrix $\boldsymbol{X}_{\mathrm{sup}}^{\mathrm{gt}}$
such that 
\begin{equation}
\boldsymbol{X}_{\mathrm{sup}}^{\mathrm{gt}}:=\boldsymbol{1}\cdot\boldsymbol{1}^{\top}\otimes\boldsymbol{I}_{n},\label{eq:ground_truth}
\end{equation}
and

\begin{equation}
\boldsymbol{X}^{\mathrm{gt}}:=\left[\begin{array}{cccc}
\boldsymbol{\Pi}_{1}\\
 & \boldsymbol{\Pi}_{2}\\
 &  & \ddots\\
 &  &  & \boldsymbol{\Pi}_{n}
\end{array}\right]\boldsymbol{X}_{\mathrm{sup}}^{\mathrm{gt}}\left[\begin{array}{cccc}
\boldsymbol{\Pi}_{1}^{\top}\\
 & \boldsymbol{\Pi}_{2}^{\top}\\
 &  & \ddots\\
 &  &  & \boldsymbol{\Pi}_{n}^{\top}
\end{array}\right],\label{eq:Xsup_defn}
\end{equation}
where the matrices $\boldsymbol{\Pi}_{i}\in\mathbb{R}^{|\mathcal{S}_{i}|\times m}$
are defined such that $\boldsymbol{\Pi}_{i}$ denotes the submatrix
of $\boldsymbol{I}_{m}$ coming from its rows at indices from $\mathcal{S}_{i}$.
For instance, if $\mathcal{S}_{i}=\{2,3\}$, then one has
\[
\boldsymbol{\Pi}_{i}=\left[\begin{array}{ccccc}
0 & 1 & 0 & \cdots & 0\\
0 & 0 & 1 & \cdots & 0
\end{array}\right].
\]
With this notation, $\boldsymbol{\Pi}_{i}\boldsymbol{M}\boldsymbol{\Pi}_{j}^{\top}$
represents a submatrix of $\boldsymbol{M}\in\mathbb{R}^{m\times m}$
coming from the rows at indices from $\mathcal{S}_{i}$ and columns
at indices from $\mathcal{S}_{j}$. Conversely, for any matrix $\tilde{\boldsymbol{M}}\in\mathbb{R}^{|\mathcal{S}_{i}|\times|\mathcal{S}_{j}|}$,
the matrix $\boldsymbol{\Pi}_{i}^{\top}\tilde{\boldsymbol{M}}\boldsymbol{\Pi}_{j}$
converts $\tilde{\boldsymbol{M}}$ to an $m\times m$ matrix space
via zero padding.

With this notation, we can represent $\boldsymbol{X}^{\mathrm{in}}$
as a submatrix of $\boldsymbol{X}_{\mathrm{sup}}^{\mathrm{in}}$,
which is a corrupted version of $\boldsymbol{X}_{\mathrm{sup}}^{\mathrm{gt}}$
and obeys
\begin{equation}
\boldsymbol{X}_{ij}^{\mathrm{in}}:=\boldsymbol{\Pi}_{i}\left(\boldsymbol{X}_{\mathrm{sup}}^{\mathrm{in}}\right)_{ij}\boldsymbol{\Pi}_{j}^{\top}.\label{eq:Xin_def}
\end{equation}
For notational simplicity, we set
\begin{equation}
\boldsymbol{W}_{ij}:=\begin{cases}
-\boldsymbol{X}_{ij}^{\mathrm{in}}+\lambda{\bf 1}\cdot{\bf 1}^{\top},\quad & \text{if }(i,j)\in\mathcal{G},\\
\lambda{\bf 1}\cdot{\bf 1}^{\top}, & \text{else}.
\end{cases}\label{eq:DefnW}
\end{equation}

Before continuing to the proof, it is convenient to introduce some
notations that will be used throughout. Denote by $\Omega_{\mathrm{gt}}$
and $\Omega_{\mathrm{gt}}^{\perp}$ the support of $\boldsymbol{X}^{\mathrm{gt}}$
and its complement support, respectively, and let $\mathcal{P}_{\Omega_{\mathrm{gt}}}$
and $\mathcal{P}_{\Omega_{\mathrm{gt}}^{\perp}}$ represent the orthogonal
projection onto the linear space of matrices supported on $\Omega_{\mathrm{gt}}$
and its complement support $\Omega_{\mathrm{gt}}^{\perp}$, respectively.
Define $T_{\mathrm{gt}}$ to be the tangent space at $\boldsymbol{X}^{\mathrm{gt}}$
w.r.t. all symmetric matrices of rank at most $m$, i.e. the space
of symmetric matrices of the form
\begin{equation}
T_{\mathrm{gt}}:=\left\{ \left[\begin{array}{c}
\boldsymbol{\Pi}_{1}\\
\boldsymbol{\Pi}_{2}\\
\vdots\\
\boldsymbol{\Pi}_{n}
\end{array}\right]\boldsymbol{M}+\boldsymbol{M}^{\top}\left[\begin{array}{cccc}
\boldsymbol{\Pi}_{1}^{\top} & \boldsymbol{\Pi}_{2}^{\top} & \cdots & \boldsymbol{\Pi}_{n}^{\top}\end{array}\right]:\text{ }\boldsymbol{M}\in\mathbb{R}^{m\times N}\right\} ,\label{eq:TangentSpace}
\end{equation}
and denote by $T_{\mathrm{gt}}^{\perp}$ its orthogonal complement.
We then denote by $\mathcal{P}_{T_{\mathrm{gt}}}$ (resp. $\mathcal{P}_{T_{\mathrm{gt}}^{\perp}}$)
the orthogonal projection onto $T_{\mathrm{gt}}$ (resp. $T_{\mathrm{gt}}^{\perp}$).
In passing, if we define 
\begin{equation}
\boldsymbol{\Sigma}:=\mathrm{Diag}\left\{ \left[\frac{n}{n_{1}},\cdots,\frac{n}{n_{m}}\right]\right\} ,\label{eq:DefnSigma}
\end{equation}
then the columns of
\begin{equation}
\boldsymbol{U}:=\frac{1}{\sqrt{n}}\left[\begin{array}{c}
\boldsymbol{\Pi}_{1}\\
\boldsymbol{\Pi}_{2}\\
\vdots\\
\boldsymbol{\Pi}_{n}
\end{array}\right]\boldsymbol{\Sigma}^{\frac{1}{2}}\label{eq:EigenSpace}
\end{equation}
form the set of eigenvectors of $\boldsymbol{X}^{\mathrm{gt}}$, and
for any symmetric matrix $\boldsymbol{M}$,
\begin{equation}
\mathcal{P}_{T_{\mathrm{gt}}^{\perp}}\left(\boldsymbol{M}\right)=\left(\boldsymbol{I}-\boldsymbol{U}\boldsymbol{U}^{\top}\right)\boldsymbol{M}\left(\boldsymbol{I}-\boldsymbol{U}\boldsymbol{U}^{\top}\right).\label{eq:ProjectionTangentPerp}
\end{equation}

Furthermore, we define a vector $\boldsymbol{d}$ to be
\begin{equation}
\boldsymbol{d}:=\left[\begin{array}{c}
\boldsymbol{\Pi}_{1}\\
\boldsymbol{\Pi}_{2}\\
\vdots\\
\boldsymbol{\Pi}_{n}
\end{array}\right]\boldsymbol{\Sigma}{\bf 1}_{m}.\label{eq:DefnD-evector}
\end{equation}
Put another way, if any row index $j$ of $\boldsymbol{X}^{\mathrm{gt}}$
is associated with the element $s\in[m]$, then $\boldsymbol{d}_{j}=\frac{n}{n_{s}}$.
One can then easily verify that
\begin{align}
\left\langle \boldsymbol{d}\cdot\boldsymbol{d}^{\top},\boldsymbol{X}^{\mathrm{gt}}-\frac{1}{m}\boldsymbol{1}\cdot\boldsymbol{1}^{\top}\right\rangle  & =\left\langle \boldsymbol{d}\cdot\boldsymbol{d}^{\top},\boldsymbol{X}^{\mathrm{gt}}\right\rangle -\frac{1}{m}\left({\bf 1}^{\top}\cdot\boldsymbol{d}\right)^{2}=0.
\end{align}
In fact, when $n_{i}$'s are sufficiently close to each other, $\boldsymbol{d}\cdot\boldsymbol{d}^{\top}$
is a good approximation of ${\bf 1}\cdot{\bf 1}^{\top}$, as claimed
in the following lemma.

\begin{lem}\label{lemma:MeanApprox}Consider a set of Bernoulli random
variables $\nu_{i}\sim\text{Bernoulli}\left(p\right)$ ($1\leq i\leq n$),
and set $s:=\sum_{i=1}^{n}\nu_{i}$. Let $n_{i}$ ($1\leq i\leq m$)
be independent copies of $s$, and denote $N=n_{1}+\cdots+n_{m}$.
If $p>\frac{c_{7}\log^{2}\left(mn\right)}{n}$, then the matrix
\begin{equation}
\boldsymbol{A}:=\left(np\right)^{2}\left[\begin{array}{c}
\frac{1}{n_{1}}{\bf 1}_{n_{1}}\\
\frac{1}{n_{2}}{\bf 1}_{n_{2}}\\
\vdots\\
\frac{1}{n_{m}}{\bf 1}_{n_{m}}
\end{array}\right]\left[\begin{array}{cccc}
\frac{1}{n_{1}}{\bf 1}_{n_{1}}^{\top} & \frac{1}{n_{2}}{\bf 1}_{n_{2}}^{\top} & \cdots & \frac{1}{n_{m}}{\bf 1}_{n_{m}}^{\top}\end{array}\right]\label{eq:DefnA}
\end{equation}
satisfies 
\begin{equation}
\left\Vert \frac{1}{m}\boldsymbol{A}-\frac{1}{m}{\bf 1}_{N}\cdot{\bf 1}_{N}^{\top}\right\Vert \leq c_{8}\sqrt{np\log(mn)}\label{eq:DeviationMeanAllOne}
\end{equation}
and
\begin{equation}
\left\Vert \boldsymbol{A}-{\bf 1}_{N}\cdot{\bf 1}_{N}^{\top}\right\Vert _{\infty}\leq c_{9}\sqrt{\frac{\log(mn)}{np}}\label{eq:DeviationMeanAllOne-infty}
\end{equation}
with probability exceeding $1-\frac{1}{m^{5}n^{5}}$, where $c_{7},c_{8},c_{9}$
are some universal constants.\end{lem}

\begin{proof}See Appendix \ref{sec:Proof_lemma:SpectralGapTight-1}.\end{proof}

Since $p^{2}\boldsymbol{d}\cdot\boldsymbol{d}^{\top}$ is equivalent
to $\boldsymbol{A}$ defined in (\ref{eq:DefnA}) up to row / column
permutation, Lemma \ref{lemma:MeanApprox} reveals that
\[
\left\Vert \frac{p^{2}}{m}\boldsymbol{d}\cdot\boldsymbol{d}^{\top}-\frac{1}{m}{\bf 1}_{N}\cdot{\bf 1}_{N}^{\top}\right\Vert \leq c_{8}\sqrt{np\log(mn)}
\]
with high probability.

The following bound on the operator norm of a random block matrix
is useful for deriving our main results.

\begin{lem}\label{lemma:MomentMethod}Let $\boldsymbol{M}=\left[\boldsymbol{M}_{ij}\right]_{1\leq i,j\leq n}$
be a symmetric block matrix, where $\boldsymbol{M}_{ij}$'s are jointly
independent $m_{i}\times m_{j}$ matrices satisfying
\begin{equation}
\mathbb{E}\boldsymbol{M}_{ij}={\bf 0},\quad\mathbb{E}\left\Vert \boldsymbol{M}_{ij}\right\Vert ^{2}\leq1,\quad\text{and}\quad\left\Vert \boldsymbol{M}_{ij}\right\Vert \leq\sqrt{n},\quad(1\leq i,j\leq n).\label{eq:M_block_assumption}
\end{equation}
Besides, $m_{i}\leq m$ holds for all $1\leq i\leq n$. Then there
exists an absolute constant $c_{0}>0$ such that 
\[
\left\Vert \boldsymbol{M}\right\Vert \leq c_{0}\sqrt{n}\log\left(mn\right)
\]
holds with probability exceeding $1-\frac{1}{m^{5}n^{5}}$.\end{lem}

\begin{proof}See Appendix \ref{sec:Proof_lemma:MomentMethod}.\end{proof}

Additionally, the second smallest eigenvalue of the Laplacian matrix
of a random Erd\H{o}s\textendash{}Rényi graph can be bounded below
by the following lemma.

\begin{lem}\label{lemma:SpectralGapTight}Consider an Erd\H{o}s\textendash{}Rényi
graph $\mathcal{G}\sim\mathcal{G}(n,p)$ and any positive integer
$m$, and let $\boldsymbol{L}\in\mathbb{R}^{n\times n}$ represent
its (unnormalized) Laplacian matrix. There exist absolute constants
$c_{3},c_{4}>0$ such that if $p>c_{3}\log^{2}\left(mn\right)/n$,
then the algebraic connectivity $a\left(\mathcal{G}\right)$ of $\mathcal{G}$
(i.e. the second smallest eigenvalue of $\boldsymbol{L}$) satisfies
\begin{equation}
a\left(\mathcal{G}\right)\geq np-c_{4}\sqrt{np}\log\left(mn\right)\label{eq:AlgebraicConnectivityGnp}
\end{equation}
with probability exceeding $1-\frac{2}{(mn)^{5}}$.\end{lem}

\begin{proof}See Appendix \ref{sec:Proof_lemma:SpectralGapTight}.\end{proof}

Finally, if we denote by $n_{s}$ (resp. $n_{s,t}$) the number of
sets $\mathcal{S}_{i}$ ($1\leq i\leq n$) containing the element
$s$ (resp. containing $s$ and $t$ simultaneously), then these quantities
sharply concentrate around their mean values, as stated in the following
lemma.

\begin{lem}\label{lemma:Concentration}There are some universal constants
$c_{8},c_{9}>0$ such that if $p_{\mathrm{set}}^{2}>\frac{\log\left(mn\right)}{n}$,
then
\begin{align*}
\left|n_{s}-np_{\mathrm{set}}\right| & \leq\sqrt{c_{8}np_{\mathrm{set}}\log\left(mn\right)},\quad\forall1\leq s\leq m,\\
\left|n_{s,t}-np_{\mathrm{set}}^{2}\right| & \leq\sqrt{c_{8}np_{\mathrm{set}}^{2}\log\left(mn\right)},\quad\forall1\leq s<t\leq m,
\end{align*}
hold with probability exceeding $1-\frac{1}{(mn)^{10}}$.\end{lem}

\begin{proof}In passing, the claim follows immediately from the Bernstein
inequality that
\[
\mathbb{P}\left(\left|\sum_{i=1}^{n}\nu_{i}-np\right|>t\right)\leq2\exp\left(-\frac{\frac{1}{2}t^{2}}{np(1-p)+\frac{1}{3}t}\right)
\]
where $\nu_{i}\sim\text{Bernoulli}(p)$ are i.i.d. random variables.
Interested readers are referred to \cite{alon2008probabilistic} for
a tutorial.\end{proof}

\subsection{Optimality and Uniqueness Condition\label{sec:Duality}}

Recall that $n_{i}:=\left|\mathcal{I}_{i}\right|$ denotes the number
of sets $\mathcal{S}_{j}$ containing the element $i$. The convex
relaxation is exact if one can construct valid dual certificates,
as summarized in the following lemma.

\begin{lem}\label{lemma:KKT}Suppose that there exist dual certificates
$\alpha>0$, $\boldsymbol{Z}=\left[\boldsymbol{Z}_{ij}\right]_{1\leq i,j\leq n}\in\mathbb{S}^{N\times N}$
and $\boldsymbol{Y}=\left[\boldsymbol{Y}_{ij}\right]_{1\leq i,j\leq n}\in\mathbb{S}^{N\times N}$
obeying
\begin{align}
\boldsymbol{Y}-\alpha\boldsymbol{d}\boldsymbol{d}^{\top} & \succeq{\bf 0},\label{eq:Z_pos_condition}\\
\mathcal{P}_{\Omega_{\mathrm{gt}}}\left(\boldsymbol{Z}\right) & ={\bf 0},\quad\mathcal{P}_{\Omega_{\mathrm{gt}}^{\perp}}\left(\boldsymbol{Z}\right)\geq{\bf 0},\label{eq:Z_pos}\\
\boldsymbol{Y}_{ij} & =\boldsymbol{W}_{ij}-\boldsymbol{Z}_{ij},\quad1\leq i<j\leq n,\label{eq:S_construction}\\
\boldsymbol{Y}-\alpha\boldsymbol{d}\boldsymbol{d}^{\top} & \in T_{\mathrm{gt}}^{\perp}.\label{eq:Y-tangent-space}
\end{align}

Then $\boldsymbol{X}^{\mathrm{gt}}$ is the unique solution to MatchLift
if either of the following two conditions is satisfied:

i) All entries of $\boldsymbol{Z}_{ij}$ ($\forall i\neq j$) within
the support $\Omega_{\mathrm{gt}}^{\perp}$ are strictly positive;

ii) For all $\boldsymbol{M}$ satisfying $\mathcal{P}_{T_{\mathrm{gt}}^{\perp}}\left(\boldsymbol{M}\right)\succ{\bf 0}$,
\begin{equation}
\left\langle \boldsymbol{Y}-\alpha\boldsymbol{d}\boldsymbol{d}^{\top},\mathcal{P}_{T_{\mathrm{gt}}^{\perp}}\left(\boldsymbol{M}\right)\right\rangle >0,\label{eq:S_psd_Null}
\end{equation}
and, additionally,
\begin{equation}
\frac{n}{n_{i}}+\frac{n}{n_{j}}\neq\frac{n^{2}}{n_{i}n_{j}},\quad1\leq i,j\leq m.\label{eq:IiIj_constraint-KKTlemma}
\end{equation}
\end{lem}

\begin{proof}See Appendix \ref{sec:Proof_lemma:KKT}.\end{proof}

That said, to prove Theorem \ref{thm:RandomGraph-1}, it is sufficient
(under the hypotheses of Theorem \ref{thm:RandomGraph-1}) to generate,
with high probability, valid dual certificates $\boldsymbol{Y}$,
$\boldsymbol{Z}$ and $\alpha>0$ obeying the optimality conditions
of Lemma \ref{lemma:KKT}. This is the objective of the next subsection.

\subsection{Construction of Dual Certificates\label{sec:DualConstruction}}

Decompose the input $\boldsymbol{X}^{\mathrm{in}}$ into two components
$\boldsymbol{X}^{\mathrm{in}}=\boldsymbol{X}^{\mathrm{false}}+\boldsymbol{X}^{\mathrm{true}}$,
where 
\begin{equation}
\boldsymbol{X}^{\mathrm{true}}=\mathcal{P}_{\Omega_{\mathrm{gt}}}\left(\boldsymbol{X}^{\mathrm{in}}\right),\quad\text{and}\quad\boldsymbol{X}^{\mathrm{false}}=\mathcal{P}_{\Omega_{\mathrm{gt}}^{\perp}}\left(\boldsymbol{X}^{\mathrm{in}}\right).\label{eq:DefnXtrue_false}
\end{equation}
That said, $\boldsymbol{X}^{\mathrm{true}}$ (resp. $\boldsymbol{X}^{\mathrm{false}}$)
consists of all correct (resp. incorrect) correspondences (i.e. non-zero
entries) encoded in $\boldsymbol{X}^{\mathrm{in}}$. This allows us
to write
\begin{equation}
\boldsymbol{W}_{ij}=\begin{cases}
-\boldsymbol{X}_{ij}^{\mathrm{false}}+\lambda\boldsymbol{E}_{ij}-\boldsymbol{X}_{ij}^{\mathrm{true}}+\lambda\boldsymbol{E}_{ij}^{\perp},\quad & \text{if }(i,j)\in\mathcal{G},\\
\lambda\boldsymbol{E}_{ij}+\lambda\boldsymbol{E}_{ij}^{\perp}, & \text{else},
\end{cases}\label{eq:DefnW_separate}
\end{equation}
where $\boldsymbol{E}$ and $\boldsymbol{E}^{\perp}$ are defined
to be
\begin{equation}
\boldsymbol{E}:=\mathcal{P}_{\Omega_{\mathrm{gt}}}\left(\boldsymbol{1}\cdot\boldsymbol{1}^{\top}\right),\quad\text{and}\quad\boldsymbol{E}^{\perp}:=\boldsymbol{1}\cdot\boldsymbol{1}^{\top}-\boldsymbol{E}.\label{eq:DefnE}
\end{equation}

We propose constructing the dual certificate $\boldsymbol{Y}$ by
producing three symmetric matrix components $\boldsymbol{Y}^{\mathrm{true},1}$,
$\boldsymbol{Y}^{\mathrm{true},2}$, and $\boldsymbol{Y}^{\mathrm{L}}$
separately, as follows. 
\begin{enumerate}
\item \emph{Construction of $\boldsymbol{Z}^{\mathrm{m}}$ and }$\boldsymbol{R}^{\mathrm{m}}$.
For any $\beta\geq0$, define $\alpha_{\beta}$ to be
\begin{equation}
\alpha_{\beta}:=\mathrm{arg}\min_{\alpha:\beta{\bf 1}\cdot{\bf 1}^{\top}-\alpha\boldsymbol{d}\cdot\boldsymbol{d}^{\top}\geq{\bf 0}}\left\Vert \beta{\bf 1}\cdot{\bf 1}^{\top}-\alpha\boldsymbol{d}\cdot\boldsymbol{d}^{\top}\right\Vert _{\infty}.\label{eq:DefnMm}
\end{equation}
By setting $\beta_{0}:=\lambda-\frac{\left(1-p_{\mathrm{true}}\right)p_{\mathrm{obs}}}{m}-\sqrt{\frac{c_{10}p_{\mathrm{obs}}\log\left(mn\right)}{np_{\mathrm{set}}^{3}}}$,
we produce \emph{$\boldsymbol{Z}^{\mathrm{m}}$} and $\boldsymbol{R}^{\mathrm{m}}$
as follows 
\begin{equation}
\boldsymbol{Z}^{\mathrm{m}}=\mathcal{P}_{\Omega_{\mathrm{gt}}^{\perp}}\left(\left(\lambda-\frac{\left(1-p_{\mathrm{true}}\right)p_{\mathrm{obs}}}{m}\right){\bf 1}\cdot{\bf 1}^{\top}-\alpha_{\beta_{0}}\boldsymbol{d}\cdot\boldsymbol{d}^{\top}\right)\label{eq:DefnZm}
\end{equation}
and
\begin{equation}
\boldsymbol{R}^{\mathrm{m}}=\mathcal{P}_{\Omega_{\mathrm{gt}}}\left(\left(\lambda-\frac{\left(1-p_{\mathrm{true}}\right)p_{\mathrm{obs}}}{m}\right){\bf 1}\cdot{\bf 1}^{\top}-\alpha_{\beta_{0}}\boldsymbol{d}\cdot\boldsymbol{d}^{\top}\right)\label{eq:DefnZm-diagonal}
\end{equation}
for some sufficiently large constant $c_{10}>0$. 
\item \emph{Construction of $\boldsymbol{Y}^{\mathrm{true},1}$ and $\boldsymbol{Y}^{\mathrm{true},2}$}.
We set 
\[
\boldsymbol{Y}_{ij}^{\mathrm{true},1}=\begin{cases}
-\boldsymbol{X}_{ij}^{\mathrm{true}}+\frac{\left(1-p_{\mathrm{true}}\right)p_{\mathrm{obs}}}{m}\boldsymbol{E}_{ij},\quad & \text{if }i<j,\\
\sum_{j=1}^{n}\left(\boldsymbol{X}_{ij}^{\mathrm{true}}-\frac{\left(1-p_{\mathrm{true}}\right)p_{\mathrm{obs}}}{m}\boldsymbol{E}_{ij}\right)\boldsymbol{\Pi}_{j}\boldsymbol{\Pi}_{i}^{\top},\quad & \text{if }i=j,
\end{cases}
\]
and
\[
\boldsymbol{Y}_{ij}^{\mathrm{true},2}=\begin{cases}
\boldsymbol{R}_{ij}^{\mathrm{m}},\quad & \text{if }i<j,\\
-\sum_{j=1}^{n}\boldsymbol{R}_{ij}^{\mathrm{m}}\boldsymbol{\Pi}_{j}\boldsymbol{\Pi}_{i}^{\top},\quad & \text{if }i=j.
\end{cases}
\]

\item \emph{Construction of $\boldsymbol{Y}^{\mathrm{L}}$ and $\boldsymbol{Z}^{\mathrm{L}}$
via an iterative procedure.} Next, we generate $\boldsymbol{Y}^{\mathrm{L}}$
via the following iterative procedure. Here, for any matrix $\boldsymbol{M}$,
we let $\boldsymbol{M}_{ij}(s,s')$ represent the entry in the $(i,j)^{\mathrm{th}}$
block $\boldsymbol{M}_{ij}$ that encodes the correspondence from
$s$ to $s'$.\\\quad\\%
\begin{tabular}{>{\raggedright}p{6in}}
\hline 
\textbf{Construction of a dual certificate $\boldsymbol{Y}^{\mathrm{L}}$.}\tabularnewline
\hline 
\noalign{\vskip\doublerulesep} $\quad\quad$1. \textbf{initialize}:
Set the symmetric matrix $\boldsymbol{Y}^{\mathrm{L},0}$ such that
\[
\boldsymbol{Y}_{ij}^{\mathrm{L},0}=\begin{cases}
-\boldsymbol{X}_{ij}^{\mathrm{false}}+\frac{\left(1-p_{\mathrm{true}}\right)p_{\mathrm{obs}}}{m}\boldsymbol{E}_{ij}^{\perp},\quad & \text{if }i<j,\\
{\bf 0}, & \text{if }i=j,
\end{cases}
\]
and start with $\boldsymbol{Z}^{\mathrm{L}}=\boldsymbol{0}$. \tabularnewline
$\quad\quad$2. \textbf{for} each\emph{ non-zero} entry $\boldsymbol{Y}_{ij}^{\mathrm{L},0}(s,s')$:\tabularnewline
$\quad\quad$3. $\quad\quad$Set $a=\boldsymbol{Y}_{ij}^{\mathrm{L},0}(s,s')$,
$B_{i,j,s,s'}=\left\{ l\notin\left\{ i,j\right\} \mid(s,s')\in\mathcal{S}_{l}\right\} $
and $n_{i,j}^{s,s'}=\left|B_{i,j,s,s'}\right|$.\tabularnewline
$\quad\quad$4. $\quad\quad$\textbf{for} each set\emph{ $l\in B_{i,j,s,s'}$}:
perform 
\[
\quad\quad\begin{cases}
\boldsymbol{Z}_{il}^{\mathrm{L}}\left(s,s'\right)\leftarrow\boldsymbol{Z}_{il}^{\mathrm{L}}\left(s,s'\right)-\frac{a}{n_{i,j}^{s,s'}},\quad & \boldsymbol{Z}_{li}^{\mathrm{L}}\left(s',s\right)\leftarrow\boldsymbol{Z}_{li}^{\mathrm{L}}\left(s',s\right)-\frac{a}{n_{i,j}^{s,s'}},\\
\boldsymbol{Z}_{lj}^{\mathrm{L}}\left(s,s'\right)\leftarrow\boldsymbol{Z}_{lj}^{\mathrm{L}}\left(s,s'\right)-\frac{a}{n_{i,j}^{s,s'}},\quad & \boldsymbol{Z}_{jl}^{\mathrm{L}}\left(s',s\right)\leftarrow\boldsymbol{Z}_{jl}^{\mathrm{L}}\left(s',s\right)-\frac{a}{n_{i,j}^{s,s'}},\\
\boldsymbol{Z}_{ll}^{\mathrm{L}}\left(s,s'\right)\leftarrow\boldsymbol{Z}_{ll}^{\mathrm{L}}\left(s,s'\right)+\frac{a}{n_{i,j}^{s,s'}},\quad & \boldsymbol{Z}_{ll}^{\mathrm{L}}\left(s',s\right)\leftarrow\boldsymbol{Z}_{ll}^{\mathrm{L}}\left(s',s\right)+\frac{a}{n_{i,j}^{s,s'}}.
\end{cases}
\]
\tabularnewline
$\quad\quad$5. \textbf{output}: $\boldsymbol{Y}^{\mathrm{L}}=\boldsymbol{Y}^{\mathrm{L},0}+\boldsymbol{Z}^{\mathrm{L}}$.\tabularnewline
\hline 
\end{tabular}\vspace{10pt}

\item \emph{Construction of $\boldsymbol{Y}$ and $\boldsymbol{Z}$: }define
$\boldsymbol{Y}$ and $\boldsymbol{Z}$ such that
\begin{align}
\boldsymbol{Y} & =\boldsymbol{Y}^{\mathrm{true},1}+\boldsymbol{Y}^{\mathrm{true},2}+\boldsymbol{Y}^{\mathrm{L}}+\alpha_{\beta_{0}}\boldsymbol{d}\cdot\boldsymbol{d}^{\top},\label{eq:DualY}\\
\boldsymbol{Z}_{ij} & =\begin{cases}
\boldsymbol{Z}_{ij}^{\mathrm{m}}-\boldsymbol{Z}_{ij}^{\mathrm{L}},\quad & \text{if }i\neq j,\\
{\bf 0}, & \text{if }i=j.
\end{cases}\label{eq:DualZ}
\end{align}

\end{enumerate}
\begin{remark}Below is a toy example to illustrate the proposed procedure
for constructing $\boldsymbol{Z}^{\mathrm{L}}$. Consider three sets
$\mathcal{S}_{1}=\left\{ 1,2\right\} $, $\mathcal{S}_{2}=\left\{ 1,3\right\} $,
$\mathcal{S}_{3}=\left\{ 2,3,4\right\} $, and $\mathcal{S}_{4}=\left\{ 1,3\right\} $.
Suppose that $\boldsymbol{Y}^{\mathrm{L},0}$ only contains two non-zero
entries that incorrectly maps elements $1$ to $3$ in $\boldsymbol{Y}_{12}^{\mathrm{L},0}$,
as illustrated in Fig. \ref{fig:BadPointExample}(a). The resulting
$\boldsymbol{Z}^{\mathrm{L}}$ is shown in Fig. \ref{fig:BadPointExample}(b).
Clearly, $\boldsymbol{Y}^{\mathrm{L},0}+\boldsymbol{Z}^{\mathrm{L}}$
obeys $\boldsymbol{Y}^{\mathrm{L},0}+\boldsymbol{Z}^{\mathrm{L}}\in T_{\mathrm{gt}}^{\perp}$.
\end{remark}

\begin{center}
\begin{figure}
\centering

\begin{tabular}{cc}
\includegraphics[scale=0.43]{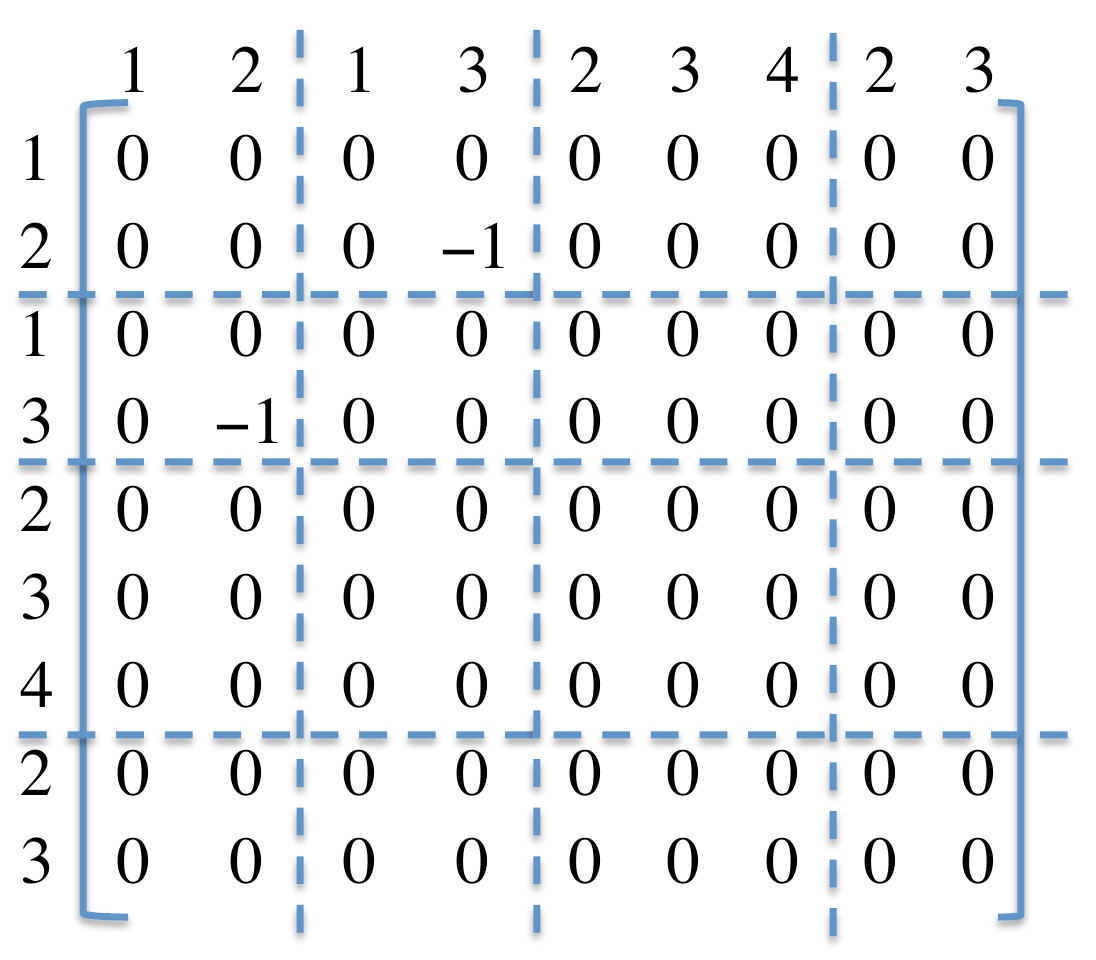} & \includegraphics[scale=0.43]{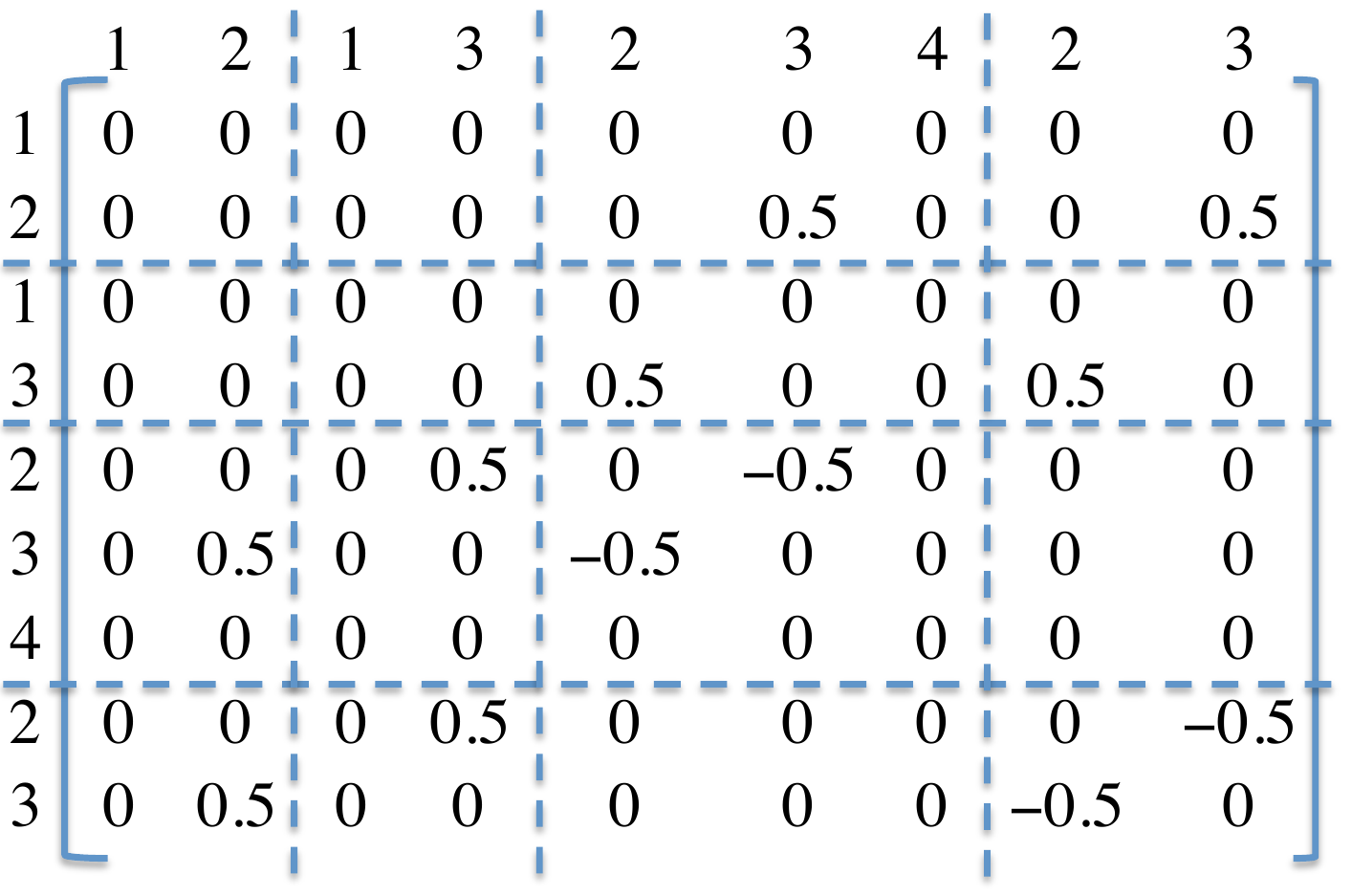}\tabularnewline
(a) Input $\boldsymbol{Y}^{0}$ & (b) $\boldsymbol{Z}^{\mathrm{L}}$\tabularnewline
\end{tabular}\caption{\label{fig:BadPointExample}A toy example for constructing $\boldsymbol{Z}^{\mathrm{L}}$,
where 4 shapes $\mathcal{S}_{1}=\left\{ 1,2\right\} $, $\mathcal{S}_{2}=\left\{ 1,3\right\} $,
$\mathcal{S}_{3}=\left\{ 2,3,4\right\} $, and $\mathcal{S}_{4}=\left\{ 2,3\right\} $
are considered. The input incorrectly maps point 1 to 3 between $\mathcal{S}_{1}$
and $\mathcal{S}_{2}$, and both points are contained in $\mathcal{S}_{3}$
and $\mathcal{S}_{4}$. One can check that $\boldsymbol{Y}^{0}+\boldsymbol{Z}^{\mathrm{L}}\in T_{\mathrm{gt}}^{\perp}$
in this example.}
\end{figure}

\par\end{center}

With the above construction procedure, one can easily verify that:

(1) $\boldsymbol{Y}^{\mathrm{true},1}$, $\boldsymbol{Y}^{\mathrm{true},2}$
and $\boldsymbol{Y}^{\mathrm{L}}$ are all contained in the space
$T_{\mathrm{gt}}^{\perp}$; 

(2) $\mathcal{P}_{\Omega^{\mathrm{gt}}}\left(\boldsymbol{Z}\right)={\bf 0}$; 

(3) If we set $\boldsymbol{M}^{\mathrm{m}}:=\alpha_{\beta_{0}}\boldsymbol{d}\cdot\boldsymbol{d}^{\top}$,
then for any $i\neq j$,
\begin{align}
\boldsymbol{Y}_{ij}= & \boldsymbol{Y}_{ij}^{\mathrm{true},1}+\boldsymbol{Y}_{ij}^{\mathrm{true},2}+\boldsymbol{Y}_{ij}^{\mathrm{L}}+\boldsymbol{M}_{ij}^{\mathrm{m}}\nonumber \\
= & -\boldsymbol{X}_{ij}^{\mathrm{true}}+\frac{\left(1-p_{\mathrm{true}}\right)p_{\mathrm{obs}}}{m}\boldsymbol{E}_{ij}+\boldsymbol{R}_{ij}^{\mathrm{m}}-\boldsymbol{X}_{ij}^{\mathrm{false}}+\frac{\left(1-p_{\mathrm{true}}\right)p_{\mathrm{obs}}}{m}\boldsymbol{E}_{ij}^{\perp}+\boldsymbol{Z}_{ij}^{\mathrm{L}}+\boldsymbol{M}_{ij}^{\mathrm{m}}\nonumber \\
= & -\boldsymbol{X}_{ij}^{\mathrm{true}}-\boldsymbol{X}_{ij}^{\mathrm{false}}+\lambda{\bf 1}\cdot{\bf 1}^{\top}-\left(\left(\lambda-\frac{\left(1-p_{\mathrm{true}}\right)p_{\mathrm{obs}}}{m}\right){\bf 1}\cdot{\bf 1}^{\top}-\boldsymbol{R}_{ij}^{\mathrm{m}}\right)+\boldsymbol{Z}_{ij}^{\mathrm{L}}+\boldsymbol{M}_{ij}^{\mathrm{m}}\nonumber \\
= & \boldsymbol{W}_{ij}-\left(\boldsymbol{Z}_{ij}^{\mathrm{m}}-\boldsymbol{Z}_{ij}^{\mathrm{L}}\right).
\end{align}

Furthermore, from Lemma \ref{lemma:MeanApprox} one can obtain
\[
\left\Vert \boldsymbol{d}\cdot\boldsymbol{d}^{\top}-\boldsymbol{1}\cdot\boldsymbol{1}^{\top}\right\Vert _{\infty}=O\left(\sqrt{\frac{\log\left(mn\right)}{np_{\mathrm{set}}}}\right).
\]
This taken collectively with (\ref{eq:DefnMm}) and the assumption
(\ref{eq:Lambda_Range}) ensures that 
\begin{equation}
\alpha_{\beta_{0}}=\lambda-\frac{\left(1-p_{\mathrm{true}}\right)p_{\mathrm{obs}}}{m}-O\left(\sqrt{\frac{c_{10}p_{\mathrm{obs}}\log\left(mn\right)}{np_{\mathrm{set}}^{3}}}\right)>0\label{eq:alpha_beta0_pos}
\end{equation}
as long as $p_{\mathrm{set}}^{3}>\frac{c_{15}\log\left(mn\right)}{n}$
for some constant $c_{15}>0$.

Consequently, we will establish that $\boldsymbol{Y}$ and $\boldsymbol{Z}$
are valid dual certificates if they satisfy
\begin{equation}
\begin{cases}
\text{all entries of }\boldsymbol{Z}_{ij}^{\mathrm{m}}-\boldsymbol{Z}_{ij}^{\mathrm{L}}\text{ }(\forall i\neq j)\text{ within }\Omega_{\mathrm{gt}}^{\perp}\text{ are strictly positive};\\
\boldsymbol{Y}^{\mathrm{true},1}+\boldsymbol{Y}^{\mathrm{true},2}+\boldsymbol{Y}^{\mathrm{L}}\succeq{\bf 0}.
\end{cases}\label{eq:RemainingCondition}
\end{equation}
Such conditions will be established through the following lemmas.

\begin{lem}\label{lemma:BoundYl_Zl}There are some universal constants
$c_{0},c_{1}>0$ such that
\[
\left\Vert \boldsymbol{Y}^{\mathrm{L}}\right\Vert \leq c_{0}\sqrt{\frac{np_{\mathrm{obs}}\log\left(mn\right)}{p_{\mathrm{set}}^{2}}}
\]
and
\begin{align*}
\left\Vert \boldsymbol{Z}_{ij}^{\mathrm{L}}\right\Vert _{\infty} & \leq\sqrt{\frac{c_{1}p_{\mathrm{obs}}\log\left(mn\right)}{np_{\mathrm{set}}^{3}}},\quad1\leq i<j\leq n
\end{align*}
with probability exceeding $1-\frac{1}{(mn)^{4}}$.\end{lem}

\begin{proof}See Appendix \ref{sec:Proof_lemma:BoundYl_Zl}.\end{proof}

\begin{lem}\label{lemma:BoundYtrue}There are some universal constants
$c_{5},c_{6},c_{7}>0$ such that if $p_{\mathrm{true}}p_{\mathrm{obs}}p_{\mathrm{set}}>\frac{c_{7}\log^{2}\left(mn\right)}{n}$
and $\lambda<\frac{\sqrt{p_{\mathrm{obs}}\log\left(mn\right)}}{p_{\mathrm{set}}}$,
then with probability exceeding $1-\frac{1}{(mn)^{10}}$, one has
\[
\left\Vert \boldsymbol{Y}^{\mathrm{true},2}\right\Vert \leq c_{5}\sqrt{\frac{np_{\mathrm{obs}}}{p_{\mathrm{set}}}}\log\left(mn\right),
\]
and
\begin{align*}
\left\langle \boldsymbol{v}\boldsymbol{v}^{\top},\boldsymbol{Y}^{\mathrm{true},1}\right\rangle  & \geq\frac{1}{2}np_{\mathrm{set}}p_{\mathrm{true}}p_{\mathrm{obs}}-c_{6}\sqrt{np_{\mathrm{set}}p_{\mathrm{obs}}}\log\left(mn\right)
\end{align*}
for all unit vector $\boldsymbol{v}$ satisfying $\boldsymbol{v}\boldsymbol{v}^{\top}\in T_{\mathrm{gt}}^{\perp}$.\end{lem}

\begin{proof}See Appendix \ref{sec:Proof_lemma:BoundYtrue}.\end{proof}

Combining Lemmas \ref{lemma:BoundYl_Zl} and \ref{lemma:BoundYtrue}
yields that there exists an absolute constant $c_{0}>0$ such that
if
\[
p_{\mathrm{true}}>c_{0}\frac{\log^{2}\left(mn\right)}{\sqrt{np_{\mathrm{obs}}p_{\mathrm{set}}^{4}}},
\]
then 
\[
\boldsymbol{Y}=\boldsymbol{Y}^{\mathrm{true},1}+\boldsymbol{Y}^{\mathrm{true},2}+\boldsymbol{Y}^{\mathrm{L}}\succeq{\bf 0}.
\]

On the other hand, observe that all entries of the non-negative matrix
$\boldsymbol{Z}^{\mathrm{m}}$ lying in the index set $\Omega_{\mathrm{gt}}^{\perp}$
are \emph{bounded below} in magnitude by $\sqrt{\frac{c_{10}p_{\mathrm{obs}}\log\left(mn\right)}{np_{\mathrm{set}}^{3}}}$.
For sufficiently large $c_{10}$, one can conclude that all entries
of $\boldsymbol{Z}_{il}^{\mathrm{m}}-\boldsymbol{Z}_{il}^{\mathrm{L}}$
outside $\Omega_{\mathrm{gt}}$ are strictly positive. 

So far we have justified that $\boldsymbol{Y}$ and $\boldsymbol{Z}$
satisfy (\ref{eq:RemainingCondition}), thereby certifying that the
proposed algorithm correctly recovers the ground-truth matching.

\section{Proofs of Auxiliary Lemmas\label{sec:ProofAuxiliaryLemmas}}

\subsection{Proof of Lemma \ref{lemma:MeanApprox}\label{sec:Proof_lemma:SpectralGapTight-1}}

Denote by $\overline{\boldsymbol{A}}:={\bf 1}_{N}\cdot{\bf 1}_{N}^{T}$.
From Bernstein inequality, $n_{i}$ sharply concentrates around $np$
such that if $p>\frac{c_{6}\log^{2}\left(mn\right)}{n}$
\begin{equation}
\left|n_{i}-np\right|\leq c_{5}\sqrt{np\log(mn)},\quad\quad\forall1\leq i\leq m\label{eq:ni_np_gap}
\end{equation}
with probability exceeding $1-(mn)^{-10}$, where $c_{5},c_{6}>0$
are some absolute constants. 

The bound (\ref{eq:ni_np_gap}) also implies that
\begin{align*}
\left\Vert \boldsymbol{I}-\left[\begin{array}{cccc}
\frac{np}{n_{1}}\\
 & \frac{np}{n_{2}}\\
 &  & \ddots\\
 &  &  & \frac{np}{n_{m}}
\end{array}\right]\right\Vert  & \leq\max_{1\leq i\leq m}\frac{\left|n_{i}-np\right|}{n_{i}}\leq\frac{c_{5}\sqrt{np\log(mn)}}{np-c_{5}\sqrt{np\log(mn)}}\\
 & \leq2c_{5}\sqrt{\frac{\log(mn)}{np}}.
\end{align*}
Similarly, one has 
\[
\left|N-nmp\right|\leq c_{5}\sqrt{pmn\log(mn)}
\]
with probability exceeding $1-(mn)^{-10}$, which implies that
\[
\left\Vert \overline{\boldsymbol{A}}\right\Vert =N\leq nmp+c_{5}\sqrt{pmn\log(mn)}<2nmp.
\]

Rewrite $\boldsymbol{A}$ as
\[
\boldsymbol{A}:=\left[\begin{array}{ccc}
\frac{np}{n_{1}}\mathrm{Diag}\left({\bf 1}_{n_{1}}\right)\\
 & \ddots\\
 &  & \frac{np}{n_{m}}\mathrm{Diag}\left({\bf 1}_{n_{m}}\right)
\end{array}\right]\cdot\overline{\boldsymbol{A}}\cdot\left[\begin{array}{ccc}
\frac{np}{n_{1}}\mathrm{Diag}\left({\bf 1}_{n_{1}}\right)\\
 & \ddots\\
 &  & \frac{np}{n_{m}}\mathrm{Diag}\left({\bf 1}_{n_{m}}\right)
\end{array}\right].
\]
This allows us to bound the deviation of $\boldsymbol{A}$ from $\overline{\boldsymbol{A}}$
as follows
\begin{align*}
\left\Vert \boldsymbol{A}-\overline{\boldsymbol{A}}\right\Vert  & \leq\left\Vert \boldsymbol{A}-\left[\begin{array}{ccc}
\frac{np}{n_{1}}\mathrm{Diag}\left({\bf 1}_{n_{1}}\right)\\
 & \ddots\\
 &  & \frac{np}{n_{m}}\mathrm{Diag}\left({\bf 1}_{n_{m}}\right)
\end{array}\right]\overline{\boldsymbol{A}}\right\Vert +\left\Vert \left[\begin{array}{ccc}
\frac{np}{n_{1}}\mathrm{Diag}\left({\bf 1}_{n_{1}}\right)\\
 & \ddots\\
 &  & \frac{np}{n_{m}}\mathrm{Diag}\left({\bf 1}_{n_{m}}\right)
\end{array}\right]\overline{\boldsymbol{A}}-\overline{\boldsymbol{A}}\right\Vert \\
 & \leq\left(\left\Vert \left[\begin{array}{ccc}
\frac{np}{n_{1}}\mathrm{Diag}\left({\bf 1}_{n_{1}}\right)\\
 & \ddots\\
 &  & \frac{np}{n_{m}}\mathrm{Diag}\left({\bf 1}_{n_{m}}\right)
\end{array}\right]\right\Vert +1\right)\left\Vert \overline{\boldsymbol{A}}\right\Vert \left\Vert \boldsymbol{I}-\left[\begin{array}{ccc}
\frac{np}{n_{1}}\mathrm{Diag}\left({\bf 1}_{n_{1}}\right)\\
 & \ddots\\
 &  & \frac{np}{n_{m}}\mathrm{Diag}\left({\bf 1}_{n_{m}}\right)
\end{array}\right]\right\Vert \\
 & \leq\left(1+c_{5}\sqrt{\frac{\log(mn)}{np}}+1\right)2nmp\cdot2c_{5}\sqrt{\frac{\log(mn)}{np}}\\
 & \leq c_{6}m\sqrt{np\log(mn)}
\end{align*}
for some universal constant $c_{6}>0$.

On the other hand, it follows immediately from (\ref{eq:ni_np_gap})
that
\begin{align*}
\left\Vert \boldsymbol{A}-{\bf 1}\cdot{\bf 1}^{\top}\right\Vert _{\infty} & =\max_{1\leq i,j\leq m}\left|\frac{\left(np\right)^{2}}{n_{i}n_{j}}-1\right|=\max_{1\leq i,j\leq m}\left|\frac{pn\left(pn-n_{j}\right)+\left(pn-n_{i}\right)n_{j}}{n_{i}n_{j}}\right|\\
 & \leq\max_{1\leq i,j\leq m}\frac{\left|pn+c_{5}\sqrt{np\log(mn)}\right|}{\left(pn-c_{5}\sqrt{np\log(mn)}\right)^{2}}c_{5}\sqrt{np\log(mn)}\\
 & \leq c_{9}\sqrt{\frac{\log(mn)}{np}}
\end{align*}
for some absolute constant $c_{9}>0$.

\subsection{Proof of Lemma \ref{lemma:MomentMethod}\label{sec:Proof_lemma:MomentMethod}}

The norm of $\boldsymbol{M}$ can be bounded via the moment method,
which attempts to control $\mathrm{tr}(\boldsymbol{M}^{k})$ for some
even integer $k$. See \cite[Section 2.3.4]{tao2012topics} for a
nice introduction. 

Specifically, observe that $\mathbb{E}\mathrm{tr}(\boldsymbol{M}^{k})$
can be expanded as follows
\[
\mathbb{E}\mathrm{tr}\left(\boldsymbol{M}^{k}\right)=\sum_{1\leq i_{1},\cdots,i_{k}\leq n}\mathbb{E}\mathrm{tr}\left(\boldsymbol{M}_{i_{1}i_{2}}\boldsymbol{M}_{i_{2}i_{3}}\cdots\boldsymbol{M}_{i_{k}i_{1}}\right),
\]
a trace sum over all $k$-cycles in the vertex set $\left\{ 1,\cdots,n\right\} $.
Note that $\left(i,i\right)$ are also treated as valid edges. For
each term $\mathbb{E}\mathrm{tr}(\boldsymbol{M}_{i_{1}i_{2}}\boldsymbol{M}_{i_{2}i_{3}}\cdots\boldsymbol{M}_{i_{k}i_{1}})$,
if there exists an edge occurring exactly once, then the term vanishes
due to the independence assumption. Thus, it suffices to examine the
terms in which each edge is repeated at least twice. Consequently,
there are at most $k/2$ relevant edges, which span at most $k/2+1$
distinct vertices. We also need to assign vertices to $k/2$ edges,
which adds up to no more than $\left(k/2\right)^{k}$ different choices. 

By following the same procedure and notation as adopted in \cite[Page 119]{tao2012topics},
we divide all non-vanishing $k$-cycles into $\left(k/2\right)^{k}$
classes based on the above labeling order; each class is associated
with $j$ ($1\leq j\leq k/2$) edges $e_{1},\cdots,e_{j}$ with multiplicities
$a_{1},\cdots,a_{j}$, where $(e_{1},\cdots,a_{1},\cdots,a_{j})$
determines the class of cycles and $a_{1}+\cdots+a_{j}=k$. Since
there are at most $n^{j+1}$ distinct vertices, one can see that no
more than $n^{j+1}$ cycles falling within this particular class.
For notational simplicity, set $K=\sqrt{n}$, and hence $\|\boldsymbol{M}_{ij}\|\leq K$.
By assumption (\ref{eq:M_block_assumption}), one has 
\begin{align*}
\mathbb{E}\mathrm{tr}\left(\boldsymbol{M}_{i_{1}i_{2}}\boldsymbol{M}_{i_{2}i_{3}}\cdots\boldsymbol{M}_{i_{k}i_{1}}\right) & \leq m\mathbb{E}\left(\left\Vert \boldsymbol{M}_{e_{1}}\right\Vert ^{a_{1}}\cdots\left\Vert \boldsymbol{M}_{e_{j}}\right\Vert ^{a_{j}}\right)\\
 & \leq m\mathbb{E}\left\Vert \boldsymbol{M}_{e_{1}}\right\Vert ^{2}\cdots\mathbb{E}\left\Vert \boldsymbol{M}_{e_{j}}\right\Vert ^{2}K^{a_{1}-2}\cdots K^{a_{j}-2}\\
 & \leq mK^{k-2j}.
\end{align*}
Thus, the total contribution of this class does not exceed
\[
mn^{j+1}K^{k-2j}=mn^{\frac{k}{2}+1}.
\]

By summing over all classes one obtains the crude bound
\[
\mathbb{E}\mathrm{tr}\left(\boldsymbol{M}^{k}\right)\leq m\left(\frac{k}{2}\right)^{k}n^{\frac{k}{2}+1},
\]
which follows that
\[
\mathbb{E}\left\Vert \boldsymbol{M}\right\Vert ^{k}\leq\mathbb{E}\mathrm{tr}\left(\boldsymbol{M}^{k}\right)\leq m\left(\frac{k}{2}\right)^{k}n^{\frac{k}{2}+1}.
\]
If we set $k=\log\left(mn\right)$, then from Markov's inequality
we have
\[
\mathbb{P}\left(\left\Vert \boldsymbol{M}\right\Vert \geq\frac{k}{2}n^{\frac{1}{2}+\frac{1}{k}}\left(mn\right)^{\frac{5}{k}}m^{\frac{1}{k}}\right)\leq\frac{\mathbb{E}\left\Vert \boldsymbol{M}\right\Vert ^{k}}{\left(\frac{k}{2}n^{\frac{1}{2}+\frac{1}{k}}\left(mn\right)^{\frac{5}{k}}m^{\frac{1}{k}}\right)^{k}}\leq\frac{m\left(\frac{k}{2}\right)^{k}n^{\frac{k}{2}+1}}{m\left(\frac{k}{2}\right)^{k}n^{\frac{k}{2}+1}\left(mn\right)^{5}}\leq\frac{1}{\left(mn\right)^{5}}.
\]
Since $n^{\frac{1}{\log n}}=O\left(1\right)$, there exists a constant
$c_{0}>0$ such that 
\[
\mathbb{P}\left(\left\Vert \boldsymbol{M}\right\Vert \geq c_{0}n^{\frac{1}{2}}\log\left(mn\right)\right)\leq\frac{1}{m^{5}n^{5}},
\]
which completes the proof.

\subsection{Proof of Lemma \ref{lemma:SpectralGapTight}\label{sec:Proof_lemma:SpectralGapTight}}

When $\mathcal{G}\sim\mathcal{G}(n,p)$, the adjacency matrix $\boldsymbol{A}$
consists of independent Bernoulli components (except for diagonal
entries), each with mean $p$ and variance $p(1-p)$. Lemma \ref{lemma:MomentMethod}
immediately implies that if $p>\frac{2\log\left(mn\right)}{n}$, then
\begin{equation}
\frac{1}{\sqrt{p(1-p)}}\left\Vert \boldsymbol{A}-p{\bf 1}_{n}\cdot{\bf 1}_{n}^{\top}\right\Vert \leq c_{0}\sqrt{n}\log\left(mn\right)+1
\end{equation}
with probability at least $1-(mn)^{-5}$. That said, there exists
an absolute constant $c_{1}>0$ such that 
\begin{equation}
\left\Vert \boldsymbol{A}-p{\bf 1}_{n}\cdot{\bf 1}_{n}^{\top}\right\Vert \leq c_{1}\sqrt{pn}\log\left(mn\right)\label{eq:AdjMatrixGnp}
\end{equation}
with probability exceeding $1-(mn)^{-5}$. 

On the other hand, from Bernstein inequality, the degree of each vertex
exceeds
\begin{equation}
d_{\min}:=pn-c_{2}\sqrt{pn\log\left(mn\right)}\label{eq:DegreeMatrixGnp}
\end{equation}
with probability at least $1-\left(mn\right)^{-10}$, where $c_{2}$
is some constant. When $p>\frac{2\log\left(mn\right)}{n}$, $\mathcal{G}$
is connected, and hence the least eigenvalue of $\boldsymbol{L}$
is zero with the eigenvector ${\bf 1}_{n}$. This taken collectively
with (\ref{eq:AdjMatrixGnp}) and (\ref{eq:DegreeMatrixGnp}) suggests
that when $p>\frac{c_{3}^{2}\log^{2}\left(mn\right)}{n}$, one has
\[
a\left(\mathcal{G}\right)\geq d_{\min}-\left\Vert \boldsymbol{A}-p{\bf 1}_{n}\cdot{\bf 1}_{n}^{\top}\right\Vert \geq pn-c_{3}\sqrt{pn}\log\left(mn\right)
\]
with high probability.

\subsection{Proof of Lemma \ref{lemma:KKT}\label{sec:Proof_lemma:KKT}}

Suppose that $\boldsymbol{X}^{\mathrm{gt}}+\boldsymbol{H}$ is the
solution to MatchLift for some perturbation $\boldsymbol{H}\neq{\bf 0}$.
By Schur complement condition for positive definiteness, the feasibility
constraint $\left[\begin{array}{cc}
m & {\bf 1}^{\top}\\
{\bf 1} & \boldsymbol{X}^{\mathrm{gt}}+\boldsymbol{H}
\end{array}\right]\succeq{\bf 0}$ is equivalent to
\[
\begin{cases}
\boldsymbol{X}^{\mathrm{gt}}+\boldsymbol{H} & \succeq{\bf 0},\\
\boldsymbol{X}^{\mathrm{gt}}+\boldsymbol{H}-\frac{1}{m}{\bf 1}\cdot{\bf 1}^{\top} & \succeq{\bf 0},
\end{cases}
\]
which immediately yields
\begin{equation}
\mathcal{P}_{T_{\mathrm{gt}}^{\perp}}\left(\boldsymbol{H}\right)=\left(\boldsymbol{I}-\boldsymbol{U}\boldsymbol{U}^{\top}\right)\left(\boldsymbol{X}^{\mathrm{gt}}+\boldsymbol{H}\right)\left(\boldsymbol{I}-\boldsymbol{U}\boldsymbol{U}^{\top}\right)\succeq{\bf 0},\label{eq:PSD_Pnm_H}
\end{equation}
and
\begin{equation}
\left\langle \boldsymbol{d}\cdot\boldsymbol{d}^{\top},\boldsymbol{H}\right\rangle =\left\langle \boldsymbol{d}\cdot\boldsymbol{d}^{\top},\boldsymbol{X}^{\mathrm{gt}}-\frac{1}{m}{\bf 1}\cdot{\bf 1}^{\top}+\boldsymbol{H}\right\rangle \geq0.\label{eq:PSD_dd_H}
\end{equation}
The above inequalities follow from the facts $\mathcal{P}_{T_{\mathrm{gt}}^{\perp}}\left(\boldsymbol{X}^{\mathrm{gt}}\right)={\bf 0}$
and $\left\langle \boldsymbol{d}\cdot\boldsymbol{d}^{\top},\boldsymbol{X}^{\mathrm{gt}}-\frac{1}{m}{\bf 1}\cdot{\bf 1}^{\top}\right\rangle =0$.

From Assumption (\ref{eq:Y-tangent-space}), one can derive
\begin{align}
\left\langle \boldsymbol{Y}-\alpha\boldsymbol{d}\cdot\boldsymbol{d}^{\top},\mathcal{P}_{T_{\mathrm{gt}}^{\perp}}\left(\boldsymbol{H}\right)\right\rangle +\left\langle \alpha\boldsymbol{d}\cdot\boldsymbol{d}^{\top},\boldsymbol{H}\right\rangle  & =\left\langle \boldsymbol{Y}-\alpha\boldsymbol{d}\cdot\boldsymbol{d}^{\top},\boldsymbol{H}\right\rangle +\left\langle \alpha\boldsymbol{d}\cdot\boldsymbol{d}^{\top},\boldsymbol{H}\right\rangle \nonumber \\
 & =\left\langle \boldsymbol{Y},\boldsymbol{H}\right\rangle =\sum\limits _{i\neq j}\left\langle \boldsymbol{Y}_{ij},\boldsymbol{H}_{ij}\right\rangle .\label{eq:trace_T_H}
\end{align}
This allows us to bound 
\begin{align}
 & \left\langle \boldsymbol{Y}-\alpha\boldsymbol{d}\cdot\boldsymbol{d}^{\top},\mathcal{P}_{T_{\mathrm{gt}}^{\perp}}\left(\boldsymbol{H}\right)\right\rangle +\sum\limits _{i\neq j}\left\langle \boldsymbol{Z}_{ij},\boldsymbol{H}_{ij}\right\rangle \nonumber \\
\leq & \left\langle \boldsymbol{Y}-\alpha\boldsymbol{d}\cdot\boldsymbol{d}^{\top},\mathcal{P}_{T_{\mathrm{gt}}^{\perp}}\left(\boldsymbol{H}\right)\right\rangle +\left\langle \alpha\boldsymbol{d}\cdot\boldsymbol{d}^{\top},\boldsymbol{H}\right\rangle +\sum\limits _{i\neq j}\left\langle \boldsymbol{Z}_{ij},\boldsymbol{H}_{ij}\right\rangle \\
= & \text{ }\sum\limits _{i\neq j}\left\langle \boldsymbol{Y}_{ij},\boldsymbol{H}_{ij}\right\rangle +\sum\limits _{i\neq j}\left\langle \boldsymbol{Z}_{ij},\boldsymbol{H}_{ij}\right\rangle \\
= & \text{ }\sum\limits _{i\neq j}\left\langle \boldsymbol{W}_{ij},\boldsymbol{H}_{ij}\right\rangle ,\label{eq:InnerProductYH}
\end{align}
where the first inequality follows from (\ref{eq:PSD_dd_H}), and
the last equality follows from Assumption (\ref{eq:S_construction}).

In order to preclude the possibility that $\boldsymbol{X}^{\mathrm{gt}}+\boldsymbol{H}$
is the solution to MatchLift, we need to show that $\sum_{i\neq j}\left\langle \boldsymbol{W}_{ij},\boldsymbol{H}_{ij}\right\rangle >0$.
From (\ref{eq:InnerProductYH}) it suffices to establish that
\begin{equation}
\left\langle \boldsymbol{Y}-\alpha\boldsymbol{d}\cdot\boldsymbol{d}^{\top},\mathcal{P}_{T_{\mathrm{gt}}^{\perp}}\left(\boldsymbol{H}\right)\right\rangle +\sum\limits _{i\neq j}\left\langle \boldsymbol{Z}_{ij},\boldsymbol{H}_{ij}\right\rangle >0\label{eq:ZH_positive_Y}
\end{equation}
for any feasible $\boldsymbol{H}\neq{\bf 0}$. In fact, since $\boldsymbol{Y}-\alpha\boldsymbol{d}\cdot\boldsymbol{d}^{\top}$
and $\mathcal{P}_{T_{\mathrm{gt}}^{\perp}}\left(\boldsymbol{H}\right)$
are both positive semidefinite, one must have
\begin{equation}
\left\langle \boldsymbol{Y}-\alpha\boldsymbol{d}\cdot\boldsymbol{d}^{\top},\mathcal{P}_{T_{\mathrm{gt}}^{\perp}}\left(\boldsymbol{H}\right)\right\rangle \geq0.\label{eq:YdH_non-neg}
\end{equation}
On the other hand, the constraints
\[
\mathrm{supp}\left(\boldsymbol{Z}\right)\subseteq\Omega_{\mathrm{gt}}^{\perp},\quad\mathcal{P}_{\Omega_{\mathrm{gt}}^{\perp}}\left(\boldsymbol{Z}\right)\geq{\bf 0},\text{ and }\mathcal{P}_{\Omega_{\mathrm{gt}}^{\perp}}\left(\boldsymbol{H}\right)\geq{\bf 0}
\]
taken together imply that
\begin{equation}
\sum\limits _{i\neq j}\left\langle \boldsymbol{Z}_{ij},\boldsymbol{H}_{ij}\right\rangle \geq0.\label{eq:ZH_non-neg}
\end{equation}
Putting (\ref{eq:YdH_non-neg}) and (\ref{eq:ZH_non-neg}) together
gives 
\[
\left\langle \boldsymbol{Y}-\alpha\boldsymbol{d}\cdot\boldsymbol{d}^{\top},\mathcal{P}_{T_{\mathrm{gt}}^{\perp}}\left(\boldsymbol{H}\right)\right\rangle +\sum\limits _{i\neq j}\left\langle \boldsymbol{Z}_{ij},\boldsymbol{H}_{ij}\right\rangle \geq0.
\]
Comparing this with (\ref{eq:ZH_positive_Y}), we only need to establish
either $\left\langle \boldsymbol{Y}-\alpha\boldsymbol{d}\cdot\boldsymbol{d}^{\top},\mathcal{P}_{T_{\mathrm{gt}}^{\perp}}\left(\boldsymbol{H}\right)\right\rangle >0$
or $\sum_{i\neq j}\left\langle \boldsymbol{Z}_{ij},\boldsymbol{H}_{ij}\right\rangle >0$.

i) Suppose first that all entries of $\boldsymbol{Z}_{ij}$ ($\forall i\neq j$)
in the support $\Omega_{\mathrm{gt}}^{\perp}$ are strictly positive.
If the identity $\sum_{i\neq j}\left\langle \boldsymbol{Z}_{ij},\boldsymbol{H}_{ij}\right\rangle =0$
holds, then the strict positivity assumption of $\boldsymbol{Z}_{ij}$
on $ $$\Omega_{\mathrm{gt}}^{\perp}$ as well as the constraint $\mathcal{P}_{\Omega_{\mathrm{gt}}^{\perp}}\left(\boldsymbol{H}\right)\geq{\bf 0}$
immediately leads to
\[
\mathcal{P}_{\Omega_{\mathrm{gt}}^{\perp}}\left(\boldsymbol{H}\right)={\bf 0}.
\]
Besides, the feasibility constraint requires that $\mathcal{P}_{\Omega_{\mathrm{gt}}}\left(\boldsymbol{H}_{ij}\right)\leq{\bf 0}$.
If $\mathcal{P}_{\Omega_{\mathrm{gt}}}\left(\boldsymbol{H}_{ij}\right)\neq{\bf 0}$,
then all non-zero entries of $\boldsymbol{H}_{ij}$ are \emph{negative},
and hence
\[
\left\langle \boldsymbol{d}\cdot\boldsymbol{d}^{\top},\boldsymbol{H}\right\rangle =\left\langle \boldsymbol{d}\cdot\boldsymbol{d}^{\top},\mathcal{P}_{\Omega_{\mathrm{gt}}}\left(\boldsymbol{H}\right)\right\rangle <0,
\]
which follows since all entries of $\boldsymbol{d}$ are strictly
positive. This contradicts with (\ref{eq:PSD_dd_H}). Consequently,
we must either have $\boldsymbol{H}={\bf 0}$ or $\sum_{i\neq j}\left\langle \boldsymbol{Z}_{ij},\boldsymbol{H}_{ij}\right\rangle >0$.
This together with (\ref{eq:ZH_positive_Y}) establishes the claim.

ii) Next, we prove the claim under Assumptions (\ref{eq:S_psd_Null})
and (\ref{eq:IiIj_constraint-KKTlemma}). In fact, Assumption (\ref{eq:S_psd_Null})
together with (\ref{eq:PSD_Pnm_H}) asserts that $\left\langle \boldsymbol{Y},\mathcal{P}_{T_{\mathrm{gt}}^{\perp}}\left(\boldsymbol{H}\right)\right\rangle \leq0$
can only occur if $\mathcal{P}_{T_{\mathrm{gt}}^{\perp}}\left(\boldsymbol{H}\right)=\boldsymbol{0}$.
This necessarily leads to $\boldsymbol{H}={\bf 0}$, as claimed by
Lemma \ref{lemma:PgtH_H}. 

\begin{lem}\label{lemma:PgtH_H}Suppose that $\boldsymbol{X}^{\mathrm{gt}}+\boldsymbol{H}$
is feasible for MatchLift, and assume that
\begin{equation}
\frac{n}{n_{i}}+\frac{n}{n_{j}}\neq\frac{n^{2}}{n_{i}n_{j}},\quad\forall1\leq i,j\leq m.\label{eq:IiIj_constraint}
\end{equation}
If $\mathcal{P}_{T_{\mathrm{gt}}^{\perp}}\left(\boldsymbol{H}\right)={\bf 0}$,
then one has $\boldsymbol{H}=\boldsymbol{0}$.\end{lem}

\begin{proof} See Appendix \ref{sec:Proof_lemma:PgtH_H}.\end{proof}

In summary, we can conclude that $\boldsymbol{X}^{\mathrm{gt}}$ is
the unique optimizer in both cases.

\subsection{Proof of Lemma \ref{lemma:BoundYl_Zl}\label{sec:Proof_lemma:BoundYl_Zl}}

First, we would like to bound the operator norm of $\boldsymbol{Y}_{\mathrm{L}}$.
Since each random matrix $\boldsymbol{X}_{ij}^{\mathrm{in}}\mathbb{I}_{\left\{ \boldsymbol{X}_{ij}^{\mathrm{in}}\text{ is observed and corrupted}\right\} }$
is independently drawn with mean $\frac{\left(1-p_{\mathrm{true}}\right)p_{\mathrm{obs}}}{m}\boldsymbol{1}\cdot{\bf 1}^{\top}$,
it is straightforward to see that
\[
\mathbb{E}\boldsymbol{Y}^{\mathrm{L,0}}=\mathbb{E}\left(-\boldsymbol{X}^{\mathrm{false}}+\frac{\left(1-p_{\mathrm{true}}\right)p_{\mathrm{obs}}}{m}\boldsymbol{E}^{\perp}\right)={\bf 0}.
\]
By observing that $\boldsymbol{Z}^{\mathrm{L}}$ is constructed as
a linear transform of $\boldsymbol{Y}^{\mathrm{L,0}}$, one can also
obtain
\[
\mathbb{E}\boldsymbol{Z}^{\mathrm{L}}={\bf 0},\quad\Rightarrow\quad\mathbb{E}\boldsymbol{Y}^{\mathrm{L}}=\mathbb{E}\boldsymbol{Z}^{\mathrm{L}}+\mathbb{E}\boldsymbol{Y}^{\mathrm{L},0}={\bf 0}.
\]
Thus, it suffices to examine the deviation of $\left\Vert \boldsymbol{Y}^{\mathrm{L}}\right\Vert $
incurred by the uncertainty of $\boldsymbol{X}^{\mathrm{false}}$.

Denote by $\boldsymbol{A}^{i,j}\in\mathbb{R}^{N\times N}$ the component
of $\boldsymbol{Z}^{\mathrm{L}}$ generated due to the $(i,j)^{\mathrm{th}}$
block $-\boldsymbol{X}_{ij}^{\mathrm{false}}$, which clearly satisfies
\[
\boldsymbol{Z}^{\mathrm{L}}=\boldsymbol{A}^{i,j}-\mathbb{E}\boldsymbol{A}^{i,j}.
\]
For each non-zero entry of $\boldsymbol{X}_{ij}^{\mathrm{false}}$,
if it encodes an incorrect correspondence between elements $s$ and
$t$, then it will affect no more than $6n_{s,t}$ entries in $\boldsymbol{A}^{i,j}$,
where each of these entries are affected in magnitude by an amount
at most $\frac{1}{n_{\mathrm{s,t}}}$. Recall that $n_{s,t}$ represents
the number of sets $\mathcal{S}_{i}$ ($1\leq i\leq n$) containing
$s$ and $t$ simultaneously, which sharply concentrates within $\left[np_{\mathrm{set}}^{2}\pm O\left(\sqrt{np_{\mathrm{set}}^{2}\log\left(mn\right)}\right)\right]$
as asserted in Lemma \ref{lemma:Concentration}. As a result, the
\emph{sum of squares} of these affected entries is bounded by 
\begin{equation}
\frac{6n_{s,t}}{n_{s,t}^{2}}=O\left(\frac{1}{n_{s,t}}\right).\label{eq:FroNorm_Aij}
\end{equation}

Moreover, since each row / column of $\boldsymbol{X}_{ij}^{\mathrm{false}}$
can have at most one non-zero entry, we can rearrange $\boldsymbol{A}^{i,j}$
with row / column permutation such that $\boldsymbol{A}^{i,j}$ becomes
a block-diagonal matrix, where the components affected by different
entries of $\boldsymbol{X}_{ij}^{\mathrm{false}}$ are separated into
distinct diagonal blocks. This together with (\ref{eq:FroNorm_Aij})
leads to 
\[
\left\Vert \boldsymbol{A}^{i,j}\right\Vert \leq\left\Vert \boldsymbol{A}^{i,j}\right\Vert _{\mathrm{F}}\leq\max_{s\neq t}\sqrt{\frac{8}{n_{s,t}}},
\]
and hence
\[
\left\Vert \mathbb{E}\boldsymbol{A}^{i,j}\left(\boldsymbol{A}^{i,j}\right)^{\top}\right\Vert \leq p_{\mathrm{obs}}\left(\max_{s\neq t}\sqrt{\frac{8}{n_{s,t}}}\right)^{2}\leq\frac{c_{16}p_{\mathrm{obs}}}{np_{\mathrm{set}}^{2}}
\]
for some absolute constant $c_{16}>0$, where the last inequality
follows from Lemma \ref{lemma:Concentration}.

Observe that $\boldsymbol{A}^{i,j}-\mathbb{E}\boldsymbol{A}^{i,j}$
$(i\neq j)$ are independently generated with mean zero, whose operator
norm is bounded above by $2\max_{s\neq t}\sqrt{\frac{8}{n_{s,t}}}$.
Applying the matrix Bernstein inequality \cite[Theorem 1.4]{tropp2012user}
suggests that there exist universal constants $c_{5},c_{6}>0$ such
that for any $t=O\left(\sqrt{n}\mathrm{poly}\log\left(mn\right)\right)$,
\[
\mathbb{P}\left(\left\Vert \sum_{(i,j)\in\mathcal{G}}\boldsymbol{A}^{i,j}-\mathbb{E}\boldsymbol{A}^{i,j}\right\Vert >t\right)\leq n^{2}\exp\left(-\frac{\frac{1}{2}t^{2}}{n^{2}\left(\frac{c_{16}p_{\mathrm{obs}}}{np_{\mathrm{set}}^{2}}\right)+\frac{2\max_{s\neq t}\sqrt{\frac{8}{n_{s,t}}}}{3}}\right).
\]
Put in another way, there exists a universal constant $c_{6}>0$ such
that 
\begin{equation}
\left\Vert \boldsymbol{Z}^{\mathrm{L}}\right\Vert =\left\Vert \sum_{i\neq j}\boldsymbol{A}^{i,j}-\mathbb{E}\boldsymbol{A}^{i,j}\right\Vert <c_{6}\sqrt{\frac{np_{\mathrm{obs}}}{p_{\mathrm{set}}^{2}}\log\left(mn\right)}\label{eq:A_bound}
\end{equation}
holds with probability exceeding $1-\frac{1}{(mn)^{10}}$. This follows
from Lemma \ref{lemma:Concentration}.

Additionally, observe that $\mathbb{E}\boldsymbol{Y}_{ij}^{\mathrm{L},0}={\bf 0}$
and 
\[
\left\Vert \frac{1}{\sqrt{p_{\mathrm{obs}}}}\boldsymbol{Y}_{ij}^{\mathrm{L},0}\right\Vert \leq\sqrt{n}
\]
as long as $p_{\mathrm{obs}}>\frac{1}{n}$. Applying Lemma \ref{lemma:MomentMethod}
suggests that
\[
\left\Vert \boldsymbol{Y}^{\mathrm{L},0}\right\Vert <c_{0}\sqrt{np_{\mathrm{obs}}\log\left(mn\right)}
\]
with probability at least $1-\frac{1}{(mn)^{5}}$. This combined with
(\ref{eq:A_bound}) yields
\[
\left\Vert \boldsymbol{Y}^{\mathrm{L}}\right\Vert \leq\left\Vert \boldsymbol{Y}^{\mathrm{L},0}\right\Vert +\left\Vert \boldsymbol{Z}^{\mathrm{L}}\right\Vert <c_{11}\sqrt{\frac{np_{\mathrm{obs}}\log\left(mn\right)}{p_{\mathrm{set}}^{2}}}
\]
with probability at least $1-\frac{3}{(mn)^{5}}$, where $c_{11}$
is some universal constant.

On the other hand, for each $(s,t)$ entry of $\boldsymbol{Z}_{il}^{\mathrm{L}}$
($i\neq l$), it can only be affected by those \emph{observed} blocks
$\boldsymbol{X}_{ij}^{\mathrm{false}}$ (or $\boldsymbol{X}_{jl}^{\mathrm{false}}$)
satisfying $t\in\mathcal{S}_{j}$ (or $s\in\mathcal{S}_{j}$). Consequently,
each entry of $\boldsymbol{Z}_{il}^{\mathrm{L}}$ can be expressed
as a sum of $\Theta\left(np_{\mathrm{set}}p_{\mathrm{obs}}\right)$
zero-mean independent variables, each of them being bounded in magnitude
by $\frac{1}{\left(\min_{s\neq t}n_{s,t}\right)}$. From Hoeffding's
inequality one can derive
\[
\mathbb{P}\left(\left\Vert \boldsymbol{Z}_{il}^{\mathrm{L}}\right\Vert _{\infty}>t\right)\leq m^{2}\mathbb{P}\left(-\frac{t^{2}}{c_{7}np_{\mathrm{set}}p_{\mathrm{obs}}\frac{1}{\left(\min\limits _{s\neq t}n_{s,t}\right)^{2}}}\right)\leq m^{2}\mathbb{P}\left(-\frac{t^{2}}{\tilde{c}_{7}p_{\mathrm{obs}}\frac{1}{np_{\mathrm{set}}^{3}}}\right)
\]
for some constants $c_{7},\tilde{c}_{7}>0$, indicating that
\[
\left\Vert \boldsymbol{Z}_{il}^{\mathrm{L}}\right\Vert _{\infty}\leq\sqrt{\frac{c_{8}p_{\mathrm{obs}}\log\left(mn\right)}{np_{\mathrm{set}}^{3}}},\quad\forall i\neq l
\]
with probability exceeding $1-\frac{1}{\left(mn\right)^{10}}$.

\subsection{Proof of Lemma \ref{lemma:BoundYtrue}\label{sec:Proof_lemma:BoundYtrue}}

By construction of $\boldsymbol{Y}^{\mathrm{true},1}$, one can see
that all non-zero entries lie within the support $\Omega_{\mathrm{gt}}$.
One important feature of $\boldsymbol{X}_{ij}^{\mathrm{gt}}$ is that
it can be converted, via row / column permutation, into a block diagonal
matrix that consists of $m$ all-one blocks, where the $i^{\mathrm{th}}$
block is of size $n_{i}$ ($1\leq i\leq m$). From Lemma \ref{lemma:Concentration},
one has
\[
n_{i}\in\left[np_{\mathrm{set}}\pm\sqrt{c_{8}np_{\mathrm{set}}\log\left(mn\right)}\right],\quad1\leq i\leq m
\]
with high probability. Thus, $\boldsymbol{Y}^{\mathrm{true},1}$ can
also be rearranged such that its non-zero entries form $m$ disjoint
diagonal blocks. We will quantify the eigenvalues of $\boldsymbol{Y}^{\mathrm{true},1}$
by bounding the spectrum of each of these matrix blocks.

We first decompose the matrix $\boldsymbol{Y}^{\mathrm{true},1}$
into two parts $\overline{\boldsymbol{Y}}_{\mathrm{}}^{\mathrm{true},1}$
and $\tilde{\boldsymbol{Y}}_{\mathrm{}}^{\mathrm{true},1}$ such that
\[
\forall i\neq j,\quad\overline{\boldsymbol{Y}}_{ij}^{\mathrm{true},1}=\begin{cases}
-\boldsymbol{X}_{ij}^{\mathrm{in}},\quad & \text{if }\boldsymbol{X}_{ij}^{\mathrm{in}}\text{ is observed and not corrupted},\\
{\bf 0}, & \text{else};
\end{cases}
\]
and
\[
\forall i\neq j,\quad\overline{\boldsymbol{Y}}_{ij}^{\mathrm{true},1}=\begin{cases}
-\boldsymbol{X}_{ij}^{\mathrm{in}}+\frac{\left(1-p_{\mathrm{true}}\right)p_{\mathrm{obs}}}{m},\quad & \text{if }\boldsymbol{X}_{ij}^{\mathrm{in}}\text{ is observed and corrupted},\\
\frac{\left(1-p_{\mathrm{true}}\right)p_{\mathrm{obs}}}{m}, & \text{else}.
\end{cases}
\]
That said, $\overline{\boldsymbol{Y}}_{ij}^{\mathrm{true},1}$ consists
of all non-corrupted components, while $\tilde{\boldsymbol{Y}}_{\mathrm{}}^{\mathrm{true},1}$
consists of all ``debiased'' random outliers.

By Lemma \ref{lemma:SpectralGapTight}, one can verify that for all
unit vector $\boldsymbol{v}$ such that $\boldsymbol{v}\boldsymbol{v}^{\top}\in T_{\mathrm{gt}}^{\perp}$,
\begin{align}
\left\langle \boldsymbol{v}\boldsymbol{v}^{\top},\overline{\boldsymbol{Y}}_{\mathrm{}}^{\mathrm{true},1}\right\rangle  & \geq\min_{1\leq s\leq m}\left(n_{s}p_{\mathrm{true}}p_{\mathrm{obs}}-c_{4}\sqrt{n_{s}p_{\mathrm{obs}}}\log\left(mn\right)\right)\nonumber \\
 & \geq\frac{1}{2}np_{\mathrm{set}}p_{\mathrm{true}}p_{\mathrm{obs}}-c_{5}\sqrt{np_{\mathrm{set}}p_{\mathrm{obs}}}\log\left(mn\right)\label{eq:BoundYtrue}
\end{align}
for some absolute constant $c_{5}>0$, where the second inequality
follows from the concentration result stated in Lemma \ref{lemma:Concentration}. 

In addition, each entry of $\tilde{\boldsymbol{Y}}_{ij}^{\mathrm{true},1}$
($i\neq j$) lying in the support $\Omega_{\mathrm{gt}}$ has mean
zero and variance $\frac{\left(1-p_{\mathrm{true}}\right)p_{\mathrm{obs}}}{m}\left(1-\frac{\left(1-p_{\mathrm{true}}\right)p_{\mathrm{obs}}}{m}\right)$.
Lemma \ref{lemma:MomentMethod} then suggests that the norm of each
non-zero block of $\tilde{\boldsymbol{Y}}_{\mathrm{}}^{\mathrm{true},1}$
(the ones with size $n_{i}$) is bounded above by $O\left(\sqrt{p_{\mathrm{obs}}n_{i}}\log\left(nm\right)\right)$.
As a result, 
\[
\left\Vert \tilde{\boldsymbol{Y}}_{\mathrm{}}^{\mathrm{true},1}\right\Vert \leq c_{15}\max_{1\leq s\leq m}\sqrt{p_{\mathrm{obs}}n_{s}}\log\left(nm\right)<\tilde{c}_{15}\sqrt{np_{\mathrm{set}}p_{\mathrm{obs}}}\log\left(nm\right).
\]
This taken collectively with (\ref{eq:BoundYtrue}) yields that
\begin{align}
\left\langle \boldsymbol{v}\boldsymbol{v}^{\top},\boldsymbol{Y}_{\mathrm{}}^{\mathrm{true},1}\right\rangle  & \geq\frac{1}{2}np_{\mathrm{set}}p_{\mathrm{true}}p_{\mathrm{obs}}-\left(c_{5}+\tilde{c}_{15}\right)\sqrt{np_{\mathrm{set}}p_{\mathrm{obs}}}\log\left(mn\right).
\end{align}

On the other hand, we know from the construction procedure and Lemma
\ref{lemma:MeanApprox} that
\begin{align*}
\left\Vert \boldsymbol{R}^{\mathrm{m}}\right\Vert _{\infty} & \leq\sqrt{\frac{c_{10}p_{\mathrm{obs}}\log\left(mn\right)}{np_{\mathrm{set}}^{3}}}+\lambda\left\Vert \boldsymbol{d}\cdot\boldsymbol{d}^{\top}-{\bf 1}\cdot{\bf 1}^{\top}\right\Vert _{\infty}\\
 & \leq\tilde{c}_{10}\left(\sqrt{\frac{p_{\mathrm{obs}}\log\left(mn\right)}{np_{\mathrm{set}}^{3}}}+\frac{\sqrt{p_{\mathrm{obs}}\log\left(mn\right)}}{p_{\mathrm{set}}}\sqrt{\frac{\log\left(mn\right)}{np_{\mathrm{set}}}}\right)\\
 & \leq2\tilde{c}_{10}\sqrt{\frac{p_{\mathrm{obs}}}{np_{\mathrm{set}}^{3}}}\log\left(mn\right)
\end{align*}
for some constants $c_{10},\tilde{c}_{10}>0$. Since $\boldsymbol{R}^{\mathrm{m}}\in\Omega_{\mathrm{gt}}$,
we can also rearrange $\boldsymbol{R}^{\mathrm{m}}$ into $m$ diagonal
blocks each of size $n_{i}$ ($1\leq i\leq m$). Hence, a crude upper
bound yields
\begin{align*}
\left\Vert \boldsymbol{Y}^{\mathrm{true},2}\right\Vert  & \leq\left\Vert \boldsymbol{Y}^{\mathrm{true},2}\right\Vert _{1}\leq\left(\max_{1\leq i\leq m}n_{i}\right)\left(2\tilde{c}_{10}\sqrt{\frac{p_{\mathrm{obs}}}{np_{\mathrm{set}}^{3}}}\log\left(mn\right)\right)\leq c_{11}np_{\mathrm{set}}\sqrt{\frac{np_{\mathrm{obs}}}{p_{\mathrm{set}}^{3}}}\log\left(mn\right)\\
 & =c_{11}n\sqrt{\frac{np_{\mathrm{obs}}}{p_{\mathrm{set}}}}\log\left(mn\right)
\end{align*}
for some universal constant $c_{11}>0$.

\subsection{Proof of Lemma \ref{lemma:PgtH_H}\label{sec:Proof_lemma:PgtH_H}}

Define an augmented matrix $\boldsymbol{H}^{\mathrm{sup}}$ such that
\begin{equation}
\boldsymbol{H}_{ij}^{\mathrm{sup}}=\boldsymbol{\Pi}_{i}^{\top}\boldsymbol{H}_{ij}\boldsymbol{\Pi}_{j}.
\end{equation}
Recall that $n_{i}$ denotes the number of sets containing element
$i$, and that
\[
\boldsymbol{\Sigma}:=\left[\begin{array}{cccc}
\frac{n}{n_{1}}\\
 & \frac{n}{n_{2}}\\
 &  & \ddots\\
 &  &  & \frac{n}{n_{m}}
\end{array}\right].
\]

The assumption that $\mathcal{P}_{T_{\mathrm{gt}}^{\perp}}\left(\boldsymbol{H}\right)={\bf 0}$
can be translated into 
\[
\left(\boldsymbol{I}-\frac{1}{n}\left({\bf 1}_{n}\otimes\boldsymbol{I}_{m}\right)\boldsymbol{\Sigma}\left({\bf 1}_{n}\otimes\boldsymbol{I}_{m}\right)\right)\boldsymbol{H}^{\mathrm{sup}}\left(\boldsymbol{I}-\frac{1}{n}\left({\bf 1}_{n}\otimes\boldsymbol{I}_{m}\right)\boldsymbol{\Sigma}\left({\bf 1}_{n}\otimes\boldsymbol{I}_{m}\right)\right)={\bf 0}.
\]
We can easily compute that
\[
\boldsymbol{H}_{ii}^{\mathrm{sup}}-\boldsymbol{\Sigma}\overline{\boldsymbol{H}}_{\cdot i}^{\mathrm{sup}}-\overline{\boldsymbol{H}}_{i\cdot}^{\mathrm{sup}}\boldsymbol{\Sigma}+\boldsymbol{\Sigma}\overline{\boldsymbol{H}}_{\cdot\cdot}^{\mathrm{sup}}\boldsymbol{\Sigma}={\bf 0},\quad1\leq i\leq n,
\]
where
\[
\begin{cases}
\overline{\boldsymbol{H}}_{\cdot i}^{\mathrm{sup}} & :=\frac{1}{n}\sum_{j=1}^{n}\boldsymbol{H}_{ji}^{\mathrm{sup}},\\
\overline{\boldsymbol{H}}_{i\cdot}^{\mathrm{sup}} & :=\frac{1}{n}\sum_{j=1}^{n}\boldsymbol{H}_{ij}^{\mathrm{sup}},\\
\overline{\boldsymbol{H}}_{\cdot\cdot}^{\mathrm{sup}} & :=\frac{1}{n^{2}}\sum_{i=1}^{n}\sum_{j=1}^{n}\boldsymbol{H}_{ij}^{\mathrm{sup}}.
\end{cases}
\]
This combined with the identity $\boldsymbol{H}_{ii}={\bf 0}$ (and
hence $\boldsymbol{H}_{ii}^{\mathrm{sup}}={\bf 0}$) yields
\[
\boldsymbol{\Sigma}\overline{\boldsymbol{H}}_{\cdot\cdot}^{\mathrm{sup}}\boldsymbol{\Sigma}=\boldsymbol{\Sigma}\overline{\boldsymbol{H}}_{\cdot i}^{\mathrm{sup}}+\overline{\boldsymbol{H}}_{i\cdot}^{\mathrm{sup}}\boldsymbol{\Sigma},\quad1\leq i\leq n.
\]
Summing over all $i$ leads to
\[
\boldsymbol{\Sigma}\overline{\boldsymbol{H}}_{\cdot\cdot}^{\mathrm{sup}}\boldsymbol{\Sigma}=\boldsymbol{\Sigma}\left(\frac{1}{n}\sum_{i=1}^{n}\overline{\boldsymbol{H}}_{\cdot i}^{\mathrm{sup}}\right)+\left(\frac{1}{n}\sum_{i=1}^{n}\overline{\boldsymbol{H}}_{i\cdot}^{\mathrm{sup}}\right)\boldsymbol{\Sigma}=\boldsymbol{\Sigma}\overline{\boldsymbol{H}}_{\cdot\cdot}^{\mathrm{sup}}+\overline{\boldsymbol{H}}_{\cdot\cdot}^{\mathrm{sup}}\boldsymbol{\Sigma}.
\]
Expanding it yields
\[
\frac{n^{2}}{n_{i}n_{j}}\left(\overline{\boldsymbol{H}}_{\cdot\cdot}^{\mathrm{sup}}\right)_{i,j}=\left(\frac{n}{n_{i}}+\frac{n}{n_{j}}\right)\left(\overline{\boldsymbol{H}}_{\cdot\cdot}^{\mathrm{sup}}\right)_{i,j},\quad1\leq i,j\leq m.
\]
From our assumption that $\frac{n^{2}}{n_{i}n_{j}}\neq\frac{n}{n_{i}}+\frac{n}{n_{j}}$,
we can derive 
\begin{equation}
\overline{\boldsymbol{H}}_{\cdot\cdot}^{\mathrm{sup}}={\bf 0}.\label{eq:H_avg_zero}
\end{equation}

Due to the feasibility constraint, all diagonal entries of $\boldsymbol{H}_{ij}^{\mathrm{sup}}$
are non-positive, and all off-diagonal entries of $\boldsymbol{H}_{ij}^{\mathrm{sup}}$
are non-negative. These conditions together with (\ref{eq:H_avg_zero})
establish that $\boldsymbol{H}={\bf 0}$.

\bibliographystyle{IEEEtran} \bibliographystyle{IEEEtran} \bibliographystyle{IEEEtran}
\bibliography{matching}

\begin{thebibliography}{10}
\providecommand{\url}[1]{#1}
\csname url@samestyle\endcsname
\providecommand{\newblock}{\relax}
\providecommand{\bibinfo}[2]{#2}
\providecommand{\BIBentrySTDinterwordspacing}{\spaceskip=0pt\relax}
\providecommand{\BIBentryALTinterwordstretchfactor}{4}
\providecommand{\BIBentryALTinterwordspacing}{\spaceskip=\fontdimen2\font plus
\BIBentryALTinterwordstretchfactor\fontdimen3\font minus
  \fontdimen4\font\relax}
\providecommand{\BIBforeignlanguage}[2]{{%
\expandafter\ifx\csname l@#1\endcsname\relax
\typeout{** WARNING: IEEEtran.bst: No hyphenation pattern has been}%
\typeout{** loaded for the language `#1'. Using the pattern for}%
\typeout{** the default language instead.}%
\else
\language=\csname l@#1\endcsname
\fi
#2}}
\providecommand{\BIBdecl}{\relax}
\BIBdecl

\bibitem{Cho:2010:JPS}
T.~S. Cho, S.~Avidan, and W.~T. Freeman, ``A probabilistic image jigsaw puzzle
  solver,'' in \emph{IEEE Conference on Computer Vision and Pattern Recognition
  (CVPR)}, 2010, pp. 183--190.

\bibitem{Goldberg:2004:GAA}
D.~Goldberg, C.~Malon, and M.~Bern, ``A global approach to automatic solution
  of jigsaw puzzles,'' \emph{Comput. Geom. Theory Appl.}, vol.~28, pp.
  165--174, June 2004.

\bibitem{Zach:2010:DVR}
C.~Zach, M.~Klopschitz, and M.~Pollefeys, ``Disambiguating visual relations
  using loop constraints,'' in \emph{IEEE Conference on Computer Vision and
  Pattern Recognition (CVPR)}, 2010, pp. 1426--1433.

\bibitem{crandall2011discrete}
D.~Crandall, A.~Owens, N.~Snavely, and D.~Huttenlocher, ``Discrete-continuous
  optimization for large-scale structure from motion,'' in \emph{IEEE
  Conference on Computer Vision and Pattern Recognition (CVPR)}, 2011, pp.
  3001--3008.

\bibitem{Huang:2006:RFO}
Q.-X. Huang, S.~Fl{\"o}ry, N.~Gelfand, M.~Hofer, and H.~Pottmann,
  ``Reassembling fractured objects by geometric matching,'' in \emph{ACM
  Transactions on Graphics (TOG)}, vol.~25, no.~3.\hskip 1em plus 0.5em minus
  0.4em\relax ACM, 2006, pp. 569--578.

\bibitem{zhu2008globally}
L.~Zhu, Z.~Zhou, and D.~Hu, ``Globally consistent reconstruction of ripped-up
  documents,'' \emph{IEEE Transactions on Pattern Analysis and Machine
  Intelligence}, vol.~30, no.~1, pp. 1--13, 2008.

\bibitem{Marande:2007:DNA}
W.~Marande and G.~Burger, ``Mitochondrial {DNA} as a genomic jigsaw puzzle,''
  \emph{Science}, vol. 318, no. 5849, pp. 415--415, 2007.

\bibitem{Roberts:2011:SFM}
R.~Roberts, S.~N. Sinha, R.~Szeliski, and D.~Steedly, ``Structure from motion
  for scenes with large duplicate structures,'' in \emph{IEEE Conference on
  Computer Vision and Pattern Recognition (CVPR)}, 2011, pp. 3137--3144.

\bibitem{Nguyen:2011:CSM}
A.~Nguyen, M.~Ben-Chen, K.~Welnicka, Y.~Ye, and L.~Guibas, ``An optimization
  approach to improving collections of shape maps,'' in \emph{Computer Graphics
  Forum}, vol.~30, no.~5.\hskip 1em plus 0.5em minus 0.4em\relax Wiley Online
  Library, 2011, pp. 1481--1491.

\bibitem{huang2012optimization}
Q.~Huang, G.~Zhang, L.~Gao, S.~Hu, A.~Butscher, and L.~Guibas, ``An
  optimization approach for extracting and encoding consistent maps in a shape
  collection,'' \emph{ACM Transactions on Graphics}, vol.~31, no.~6, p. 167,
  2012.

\bibitem{Kim:2012:FC}
V.~Kim, W.~Li, N.~Mitra, S.~DiVerdi, and T.~Funkhouser, ``Exploring collections
  of 3d models using fuzzy correspondences,'' in \emph{ACM SIGGRAPH}, 2012.

\bibitem{huang2013consistent}
Q.~Huang and L.~Guibas, ``Consistent shape maps via semidefinite programming,''
  \emph{Computer Graphics Forum}, vol.~32, no.~5, pp. 177--186, 2013.

\bibitem{schellewald2005probabilistic}
C.~Schellewald and C.~Schn{\"o}rr, ``Probabilistic subgraph matching based on
  convex relaxation,'' in \emph{Energy minimization methods in computer vision
  and pattern recognition}.\hskip 1em plus 0.5em minus 0.4em\relax Springer,
  2005, pp. 171--186.

\bibitem{cour2007balanced}
T.~Cour, P.~Srinivasan, and J.~Shi, ``Balanced graph matching,'' \emph{Advances
  in Neural Information Processing Systems (NIPS)}, 2007.

\bibitem{caetano2009learning}
T.~S. Caetano, J.~J. McAuley, L.~Cheng, Q.~V. Le, and A.~J. Smola, ``Learning
  graph matching,'' \emph{IEEE Transactions on Pattern Analysis and Machine
  Intelligence}, vol.~31, no.~6, pp. 1048--1058, 2009.

\bibitem{PachauriKS13}
D.~Pachauri, R.~Kondor, and V.~Singh, ``Solving the multi-way matching problem
  by permutation synchronization.'' in \emph{Advanced in Neural Information
  Processing Systems (NIPS)}, 2013.

\bibitem{WangSingerRotation}
L.~Wang and A.~Singer, ``Exact and stable recovery of rotations for robust
  synchronization,'' \emph{arxiv:1211.2441}, 2013.

\bibitem{chaudhury2013global}
K.~Chaudhury, Y.~Khoo, A.~Singer, and D.~Cowburn, ``Global registration of
  multiple point clouds using semidefinite programming,''
  \emph{arXiv:1306.5226}, 2013.

\bibitem{ExactMC09}
E.~J. Candes and B.~Recht, ``Exact matrix completion via convex optimization,''
  \emph{Foundations of Computational Mathematics}, vol.~9, no.~6, pp. 717--772,
  April 2009.

\bibitem{keshavan2010few}
R.~H. Keshavan, A.~Montanari, and S.~Oh, ``Matrix completion from a few
  entries,'' \emph{IEEE Transactions on Information Theory}, vol.~56, no.~6,
  pp. 2980--2998, 2010.

\bibitem{CanLiMaWri09}
E.~J. Cand\`{e}s, X.~Li, Y.~Ma, and J.~Wright, ``Robust principal component
  analysis?'' \emph{Journal of ACM}, vol.~58, no.~3, pp. 11:1--11:37, Jun 2011.

\bibitem{chandrasekaran2011rank}
V.~Chandrasekaran, S.~Sanghavi, P.~Parrilo, and A.~S. Willsky, ``Rank-sparsity
  incoherence for matrix decomposition,'' \emph{SIAM Journal on Optimization},
  vol.~21, no.~2, 2011.

\bibitem{xu2010robust}
H.~Xu, C.~Caramanis, and S.~Sanghavi, ``Robust pca via outlier pursuit,''
  \emph{Advances on Neural Information Processing Systems (NIPS)}, 2010.

\bibitem{ganesh2010dense}
A.~Ganesh, J.~Wright, X.~Li, E.~J. Candes, and Y.~Ma, ``Dense error correction
  for low-rank matrices via principal component pursuit,'' in \emph{IEEE
  International Symposium on Information Theory Proceedings (ISIT)}, 2010, pp.
  1513--1517.

\bibitem{YudongRPCA2013}
Y.~Chen, A.~Jalali, S.~Sanghavi, and C.~Caramanis, ``Low-rank matrix recovery
  from errors and erasures,'' \emph{IEEE Transactions on Information Theory},
  vol.~59, no.~7, pp. 4324--4337, July 2013.

\bibitem{jiaming2014block}
J.~Xu, R.~Wu, K.~Zhu, B.~Hajek, R.~Srikant, and L.~Ying, ``Jointly clustering
  rows and columns of binary matrices: Algorithms and trade-offs,''
  \emph{arxiv:1310.0512}, 2013.

\bibitem{bansal2004correlation}
N.~Bansal, A.~Blum, and S.~Chawla, ``Correlation clustering,'' \emph{Machine
  Learning}, vol.~56, no. 1-3, pp. 89--113, 2004.

\bibitem{mathieu2010correlation}
C.~Mathieu and W.~Schudy, ``Correlation clustering with noisy input,'' in
  \emph{ACM-SIAM SODA}, 2010, pp. 712--728.

\bibitem{jalali2011clustering}
A.~Jalali, Y.~Chen, S.~Sanghavi, and H.~Xu, ``Clustering partially observed
  graphs via convex optimization,'' \emph{International Conf. on Machine
  Learning (ICML)}, 2011.

\bibitem{chen2012clustering}
Y.~Chen, S.~Sanghavi, and H.~Xu, ``Clustering sparse graphs,'' \emph{Advances
  in Neural Information Processing Systems (NIPS)}, 2012.

\bibitem{jalali2012max_norm}
A.~Jalali and N.~Srebro, ``Clustering using max-norm constrained
  optimization,'' \emph{International Conference on Machine Learning (ICML)},
  June 2012.

\bibitem{ailon2013breaking}
N.~Ailon, Y.~Chen, and X.~Huan, ``Breaking the small cluster barrier of graph
  clustering,'' \emph{International Conference on Machine Learning (2013)},
  2013.

\bibitem{keshavan2010matrix}
R.~H. Keshavan, A.~Montanari, and S.~Oh, ``Matrix completion from noisy
  entries,'' \emph{Journal of Machine Learning Research}, vol.~99, pp.
  2057--2078, 2010.

\bibitem{horner1984phase}
J.~L. Horner and P.~D. Gianino, ``Phase-only matched filtering,'' \emph{Applied
  optics}, vol.~23, no.~6, pp. 812--816, 1984.

\bibitem{singer2011angular}
A.~Singer, ``Angular synchronization by eigenvectors and semidefinite
  programming,'' \emph{Applied and computational harmonic analysis}, vol.~30,
  no.~1, pp. 20--36, 2011.

\bibitem{bandeira2014multireference}
A.~S. Bandeira, M.~Charikar, A.~Singer, and A.~Zhu, ``Multireference alignment
  using semidefinite programming,'' in \emph{Conference on Innovations in
  Theoretical Computer Science}, 2014, pp. 459--470.

\bibitem{wen2010alternating}
Z.~Wen, D.~Goldfarb, and W.~Yin, ``Alternating direction augmented lagrangian
  methods for semidefinite programming,'' \emph{Mathematical Programming
  Computation}, vol.~2, no. 3-4, pp. 203--230, 2010.

\bibitem{durrett2007random}
R.~Durrett, \emph{Random graph dynamics}.\hskip 1em plus 0.5em minus
  0.4em\relax Cambridge university press, 2007, vol.~20.

\bibitem{Mikolajczyk:2005:PEL}
K.~Mikolajczyk and C.~Schmid, ``A performance evaluation of local
  descriptors,'' \emph{IEEE Transactions on Pattern Analysis and Machine
  Intelligence}, vol.~27, no.~10, pp. 1615--1630, 2005.

\bibitem{Leordeanu:2005:SM}
M.~Leordeanu and M.~Hebert, ``A spectral technique for correspondence problems
  using pairwise constraints,'' in \emph{IEEE International Conference on
  Computer Vision (ICCV)}, vol.~2, 2005, pp. 1482--1489.

\bibitem{Lowe:2004:DIF}
D.~G. Lowe, ``Distinctive image features from scale-invariant keypoints,''
  \emph{International journal of computer vision}, vol.~60, no.~2, pp. 91--110,
  2004.

\bibitem{Fischler:1981:RSC}
M.~A. Fischler and R.~C. Bolles, ``Random sample consensus: a paradigm for
  model fitting with applications to image analysis and automated
  cartography,'' \emph{Commun. ACM}, vol.~24, no.~6, pp. 381--395, Jun. 1981.

\bibitem{HaCohen:2011:NDC}
Y.~HaCohen, E.~Shechtman, D.~Goldman, and D.~Lischinski, ``Non-rigid dense
  correspondence with applications for image enhancement,'' \emph{ACM Trans.
  Graph.}, vol.~30, no.~4, pp. 70:1--70:10, Jul. 2011.

\bibitem{conf/cvpr/AhmedTRTS08}
N.~Ahmed, C.~Theobalt, C.~Rossl, S.~Thrun, and H.~Seidel, ``Dense
  correspondence finding for parametrization-free animation reconstruction from
  video.'' in \emph{CVPR}, 2008.

\bibitem{leordeanu2012unsupervised}
M.~Leordeanu, R.~Sukthankar, and M.~Hebert, ``Unsupervised learning for graph
  matching,'' \emph{International journal of computer vision}, vol.~96, no.~1,
  pp. 28--45, 2012.

\bibitem{boyd2004convex}
S.~P. Boyd and L.~Vandenberghe, \emph{Convex optimization}.\hskip 1em plus
  0.5em minus 0.4em\relax Cambridge university press, 2004.

\bibitem{alon2008probabilistic}
N.~Alon and J.~H. Spencer, \emph{The probabilistic method (3rd Edition)}.\hskip
  1em plus 0.5em minus 0.4em\relax Wiley, 2008.

\bibitem{tao2012topics}
T.~Tao, \emph{Topics in random matrix theory}.\hskip 1em plus 0.5em minus
  0.4em\relax AMS Bookstore, 2012, vol. 132.

\bibitem{tropp2012user}
J.~A. Tropp, ``User-friendly tail bounds for sums of random matrices,''
  \emph{Foundations of Computational Mathematics}, vol.~12, no.~4, pp.
  389--434, 2012.

\end{thebibliography}

\bibliographystyle{IEEEtran} \bibliographystyle{IEEEtran} \bibliographystyle{IEEEtran}
\bibliographystyle{IEEEtran} 
\end{document}